\documentclass{article}

\usepackage{microtype}
\usepackage{graphicx}
\usepackage{subfigure}
\usepackage{booktabs} 
\usepackage[inkscapeformat=png]{svg}

\usepackage[breaklinks]{hyperref}
\usepackage{multirow}
\usepackage{graphicx}
\usepackage{subcaption}

\usepackage{enumitem}
\usepackage{tabulary}
\usepackage{url}

\usepackage{breakurl}

\newcommand{\hX}{\hat{X}}
\newcommand{\hx}{\hat{x}}

\newcommand{\prob}{\mathrm{Pr}}
\newcommand{\norm}{\mathrm{Norm}}
\newcommand{\accept}{\mathtt{acc}}
\newcommand{\reject}{\mathtt{rej}}
\newcommand{\btau}{\boldsymbol{\tau}}
\usepackage{algorithm}
\usepackage{algorithmic}

\newcommand{\CONTINUE}{\STATE \textbf{continue}}
\newcommand{\BREAK}{\STATE \textbf{break}}
\usepackage{array}
\usepackage{eqparbox}
\renewcommand\algorithmiccomment[1]{%
  \hfill \textit{//}\ \eqparbox{COMMENT}{\textit{#1}}%
}

\usepackage{etoolbox}  
\makeatletter
\patchcmd{\algorithmic}{\addtolength{\ALC@tlm}{\leftmargin} }{\addtolength{\ALC@tlm}{\leftmargin}}{}{}
\makeatother

\usepackage[accepted]{icml2024}

\usepackage{amsmath}
\usepackage{amssymb}
\usepackage{mathtools}
\usepackage{amsthm}

\usepackage[capitalize,noabbrev]{cleveref}

\theoremstyle{plain}

\newtheorem{manualtheoreminner}{Theorem}
\newenvironment{manualtheorem}[1]{%
  \renewcommand\themanualtheoreminner{#1}%
  \manualtheoreminner
}{\endmanualtheoreminner}

\newtheorem{theorem}{Theorem}[section]

\theoremstyle{definition}

\theoremstyle{remark}

\usepackage[textsize=tiny]{todonotes}

\newcommand{\draftlen}{L_{\mathrm{draft}}}
\newcommand{\branchingfactor}{\mathbf{b}}

\icmltitlerunning{Recursive Speculative Decoding: Accelerating LLM Inference via Sampling Without Replacement}

\begin{document}

\twocolumn[
\icmltitle{Recursive Speculative Decoding: \\Accelerating LLM Inference via Sampling Without Replacement}



\icmlsetsymbol{equal}{*}

\begin{icmlauthorlist}
\icmlauthor{Wonseok Jeon}{qualcomm}
\icmlauthor{Mukul Gagrani}{qualcomm}
\icmlauthor{Raghavv Goel}{qualcomm}
\icmlauthor{Junyoung Park}{qualcomm}
\icmlauthor{Mingu Lee}{equal,qualcomm}
\icmlauthor{Christopher Lott}{equal,qualcomm}
\end{icmlauthorlist}

\icmlaffiliation{qualcomm}{Qualcomm AI Research}

\icmlcorrespondingauthor{Wonseok Jeon}{wjeon@qti.qualcomm.com}
\icmlcorrespondingauthor{Mingu Lee}{mingul@qti.qualcomm.com}
\icmlcorrespondingauthor{Christopher Lott}{clott@qti.qualcomm.com}

\icmlkeywords{Machine Learning, ICML}

\vskip 0.3in
]



\printAffiliationsAndNotice{\textsuperscript{*}Equal advising}  

\begin{abstract}
Speculative decoding is an inference-acceleration method for large language models (LLMs) where a small language model generates a draft-token sequence which is further verified by the target LLM in parallel. Recent works have advanced this method by establishing a draft-token tree, achieving superior performance over a single-sequence speculative decoding. However, those works independently generate tokens at each level of the tree, not leveraging the tree's entire diversifiability. Besides, their empirical superiority has been shown for fixed length of sequences, implicitly granting more computational resource to LLM for the tree-based methods. None of the existing works has conducted empirical studies with fixed target computational budgets despite its importance to resource-bounded devices. We present Recursive Speculative Decoding (RSD), a novel tree-based method that samples draft tokens \emph{without} replacement and maximizes the diversity of the tree. During RSD's drafting, the tree is built by either \emph{Gumbel-Top-$k$ trick} that draws tokens without replacement in parallel or \emph{Stochastic Beam Search} that samples \emph{sequences} without replacement while early-truncating unlikely draft sequences and reducing the computational cost of LLM. We empirically evaluate RSD with Llama 2 and OPT models, showing that RSD outperforms the baseline methods, consistently for fixed draft sequence length and in most cases for fixed computational budgets at LLM. 
\end{abstract}

\section{Introduction}

Large language models (LLMs)~\cite{touvron2023llama,zhang2022opt,brown2020language,achiam2023gpt,jiang2023mistral} have gained popularity due to their outstanding achievements with high-quality text generation, 
which has drastically increased demands for faster text generation.
However, auto-regressive nature of LLMs limits text generation to produce a single token at a time and often suffers from memory-bandwidth bottleneck, which leads to slower inference~\cite{shazeer2019fast}. translation~\cite{xiao2023survey}.

Speculative decoding~\cite{chen2023accelerating,leviathan2023fast} has emerged as a solution for LLM inference acceleration by leveraging the innate parallelizability of the transformer network~\cite{vaswani2017attention}. 
This decoding method utilizes a draft model, i.e., a smaller language model, to auto-regressively generate a sequence of draft tokens with a significantly lower cost and latency, followed by the target LLM producing the token-wise probability distributions in parallel. 
Rejection sampling then verifies those draft tokens, recovering the 
sequence distribution by auto-regressive decoding with the target model.
As speculative decoding uses a single sequence of draft tokens, 
one needs to increase the draft-sequence length to better exploit LLM's parallelizability.
However, the longer draft sequence may slow down the overall inference in practice due to the computational overhead caused by additional auto-regressive decoding steps from the draft model, possibly decelerating the target model process due to the increased number of draft tokens.

Recent works on tree-based speculative decoding~\cite{sun2023spectr,miao2023specinfer} have achieved better diversity and higher acceptance rate via multiple draft-token sequences. Despite promising results, their decoding methods independently sample the draft tokens, often harming the diversity of the tree when samples overlap. 
Also, their experiments have been conducted for the fixed length of draft-token sequences across decoding methods, implicitly requiring more computational resource to the target model when using tree-based methods.
To the best of our knowledge, no prior work has thoroughly investigated the performance of single-sequence and tree-based speculative decoding methods with fixed target computational budget, which has practical importance for resource-bounded devices.

We propose \textbf{R}ecursive \textbf{S}peculative \textbf{D}ecoding (RSD), a novel tree-based speculative decoding algorithm that fully exploits the diversity of the draft-token tree by using sampling without replacement.
We summarize our contributions as below:

\textbf{Theoretical contribution.} We propose \emph{recursive rejection sampling} capable of recovering the target model's distribution with the sampling-without-replacement distribution defined by the draft model. 
\newline
\textbf{Algorithmic contribution.} We present RSD which builds draft-token tree composed of the tokens \emph{sampled without replacement}. 
Two tree construction methods, \textbf{RSD} with \textbf{C}onstant branching factors (RSD-C) and \textbf{RSD} with \textbf{S}tochastic Beam Search (RSD-S)~\cite{kool2019stochastic}, are proposed. 
\newline
\textbf{Empirical contribution.} Two perspectives are considered in our experiments: \textbf{(\textit{Exp1})} \emph{performance for fixed length of draft sequence}, which is also widely considered in previous works~\cite{sun2023spectr,miao2023specinfer}, and \textbf{(\textit{Exp2})} \emph{performance for fixed target computational budget}, where we compared methods with given size of the draft-token tree.
RSD is shown to outperform the baselines consistently in \textbf{(\textit{Exp1})}  and for the majority of experiments in \textbf{(\textit{Exp2})}.

\section{Background}

Let us consider a sequence generation problem with a set $\mathcal{X}$ of tokens. We also assume that there is a target model characterized by its conditional probability $q(x_{i+1}|x_{1:i}):=\prob\{{X_{i+1}}=x_{i+1}|X_{1:i}=x_{1:i}\}$, $i\in\mathbb{N}$ for $x_{1:i}:=(x_1,...,x_i)$, where $X_{1}, ..., X_{i+1}\in\mathcal{X}$ and $x_1, ..., x_{i+1}\in\mathcal{X}$ are random tokens and their realizations, respectively. 
Given an input sequence $X_{1:t}=x_{1:t}$, we can auto-regressively and randomly sample an output sequence $X_{t+1:t+i}$ for $i\in \mathbb{N}$, i.e., 
$
X_{t+i+1}\sim q(\cdot|X_{1:t+i}).$

\textbf{Speculative decoding.} Auto-regressive sampling with modern neural network accelerators (e.g., GPU/TPU) is known to suffer from the memory-bandwidth bottleneck~\cite{shazeer2019fast}, which prevents us from utilizing the entire computing power of those accelerators. 
Speculative decoding~\cite{leviathan2023fast,chen2023accelerating} addresses such issue by using the target model's parallelizability. It introduces a (small) draft model which outputs $p(\hX_{i+1}|\hX_{1:i}):=\prob\{\hX_{i+1}=\hx_{i+1}|\hX_{1:i}=\hx_{1:i}\}, i\in\mathbb{N}$.
Speculative decoding accelerates the inference speed by iteratively conducting the following steps:
\newline\textit{1) Draft token generation:} For an input sequence $X_{1:m}=x_{1:m}$ and the draft sequence length $L$, sample draft tokens $\hX_{n+1}\sim p(\cdot|X_{1:m},\hX_{m+1:n})$ auto-regressively for $n=m,...,m+L-1$ (where $\hX_{m+1:m}=\emptyset$).
\newline\textit{2) Evaluation with target model:} Use the target model to compute $q(\cdot|X_{1:m},\hX_{m+1:n}), n=m, ..., m+L$ in parallel.
\newline\textit{3) Verification via rejection sampling:} Starting from $n=m+1$ to $m+L$, sequentially accept the draft token $\hX_{n}$ (i.e., $X_n=\hX_n$) with the probability $\min\{1, \frac{q(\hX_{n}|X_{1:n-1})}{p(\hX_{n}|X_{1:n-1})}\}$. If one of the draft tokens $\hX_{n}$ is rejected, we sample $X_{n}\sim q_{\mathrm{res}}(\cdot|X_{1:n-1})$, where the residual distribution is defined by
\begin{align*}
 q_{\mathrm{res}}(\cdot|\tau)&:=\norm[[q(\cdot|\tau) - p(\cdot|\tau)]^+],
\end{align*}
for $[f]^+:=\max\{0, f(\cdot)\}$ and $\norm[f]:=\frac{f}{\sum_{x'\in\mathcal{X}}f(x')}$.
If all draft tokens are accepted ($X_n=\hX_n$ for $n=m+1,...,m+L$), sample an extra token $X_{m+L+1}\sim q(\cdot|X_{1:m+L})$.

\citet{chen2023accelerating} and \citet{leviathan2023fast} have shown that the target distribution can be recovered when rejection sampling is applied.

\textbf{Tree-based speculative decoding.} One can further improve the sequence generation speed by using multiple draft-token sequences, or equivalently, a tree of draft tokens.

\emph{SpecTr}~\cite{sun2023spectr} is a tree-based speculative decoding algorithm motivated by the Optimal Transport (OT)~\cite{villani2009optimal}. It generalizes speculative decoding with $K$ i.i.d. draft tokens $\hX^{(k)}\sim p, k=1, ..., K,$ while recovering the target distribution $q$. 
To this end, a $K$-sequential draft selection algorithm ($K$-SEQ) was proposed, where the algorithm decides whether to accept $K$ draft tokens $\hX^{(k)}, k=1, ..., K,$ or not with the probability
$
\min\{1, \frac{q(\hX^{(k)})}{\gamma\cdot p(\hX^{(k)})}\}, \gamma\in[1, K]$.
If all draft tokens are rejected, we use a token drawn from the residual distribution
\begin{align*}
\norm\left[q-\min\left\{p,\frac{q}{\gamma}\right\} \frac{1-(1-\beta_{p,q}(\gamma))^K}{\beta_{p,q}(\gamma)}\right]
\end{align*}
for $\beta_{p,q}(\gamma):=\sum_{x\in\mathcal{X}}\min\{p(x), q(x)/\gamma\}$.

\emph{SpecInfer} also used the draft-token tree to speed up the inference \emph{with multiple draft models} $p^{(k)}, k=1, ..., K$~\cite{miao2023specinfer}. 
During the inference of SpecInfer, all draft models generate their own draft tokwns independently and create a draft-token tree collectively through repetetion.
For draft verification, \emph{multi-round rejection sampling} is used to recover the target distribution, where we determine whether to accept one of the draft tokens or not with probability $\min\{1, \frac{q^{(k)}(\hX^{(k)})}{p^{(k)}(\hX^{(k)})}\}$ 
with distributions $q^{(1)}:=q$ and 
$q^{(k)}:=\norm\left[[q^{(k-1)}-p^{(k-1)}]^+\right], k=2, ..., K+1.$
If all draft tokens are rejected, we sample a token $Y\sim q^{(K+1)}$ from the last residual distribution.

\section{Recursive Speculative Decoding}

In this section, we present \textbf{R}ecursive \textbf{S}peculative \textbf{D}ecoding (RSD), a tree-based speculative decoding method that constructs draft-token trees via sampling without replacement. 
We first propose recursive rejection sampling, a generalization of multi-round speculative decoding~\cite{miao2023specinfer} that is applicable to draft distributions with dependencies, where sampling-without-replacement distribution is one instance of such distributions.
Then, we use recursive rejection sampling to validate each level of the draft-token tree which can be efficiently constructed via either Gumbel-Top-$k$ trick~\cite{vieira2014gumbel} and Stochastic Beam Search~\cite{kool2019stochastic},

\begin{algorithm}[t]
   \caption{Recursive Rejection Sampling}
   \label{alg:rrs_abs}
\begin{algorithmic}[1]
   \STATE {\bfseries Input:}
   Draft dist. $p^{(k)}, k=1, ..., K,$ target dist. $q$. 
   \STATE 
   Sample $\hX^{(k)}$ by \eqref{eq:draft_samples}.
   \STATE 
   Compute $q^{(k)}(\cdot|\hX^{(1:k-2)})$ and $\Theta^{(k)}$ by \eqref{eq:q_res} and \eqref{eq:threshold}.
   \FOR{$k$ \textbf{in} $\{1, ..., K\}$}
   \STATE Sample $A^{(k)}\in\{\accept,\reject\}$ with probability $\Theta^{(k)}$.
   \STATE 
   \textbf{if} $A^{(k)} = \accept$ \textbf{then}
   \textbf{return} $Z\leftarrow\hX^{(k)}$; \textbf{end if}
   \ENDFOR
   \STATE \textbf{return} $Z\sim q^{(K+1)}(\cdot|\hX^{(1:K-1)})$

\end{algorithmic}
\end{algorithm}

\subsection{Recursive Rejection Sampling: Generalized Multi-Round Rejection Sampling}\label{sec:recursive_rejection_sampling}

Suppose we have target distribution $q(x), x\in\mathcal{X}$. In recursive rejection sampling, we introduce random variables $\hX^{(1)}, ..., \hX^{(K)}\in\mathcal{X}$ that represent $K$ draft tokens; these tokens will locate at the same level of the draft-token tree in Section \ref{sec:tree-based-sd-with-rrs}. 
We aim to recover target distribution $q$, where 
\begin{align}
    \hX^{(1)}\sim p^{(1)}, \hX^{(k)}\sim p^{(k)}(\cdot|\hX^{(1:k-1)}), k=2, ..., K,
    \label{eq:draft_samples}
\end{align}
for some distributions $p^{(k)}, k=1,..., K$ and a sequence $\hX^{(1:k-1)}:=(\hX^{(1)}, ..., \hX^{(k-1)})$. Note that we assume distributions with dependencies unlike prior works such as SpecTr~\cite{sun2023spectr} consider independent distributions.
By using $p^{(1)}, ..., p^{(K)}$ and $q$, we define $q^{(1)}:=q$ and residual distributions
\begin{align}
    &q^{(k+1)}(\cdot|x^{(1:k-1)})\nonumber\\
    &:=\norm\left[[q^{(k)}(\cdot|x^{(1:k-2)})-p^{(k)}(\cdot|x^{(1:k-1)})]^+\right]
    \label{eq:q_res}
\end{align}
for $k=1, ..., K$ and $x^{(1)}, ..., x^{(K+1)}\in\mathcal{X}$, where $x^{(1:k')}=\emptyset$ (empty sequence, i.e., no conditioning) if $k'<1$, or $(x^{(1)}, ..., x^{(k')})$, otherwise.
Together with draft, target, and residual distributions, recursive rejection sampling introduces threshold random variables $\Theta^{(1)}, ..., \Theta^{(K)}\in[0, 1]$ which determines rejection criteria for each draft token $\hX^{(k)}, k=1, ..., K$:
\begin{align}
    \Theta^{(k)}
    &:=\min\left\{1, \frac{q^{(k)}(\hX^{(k)}|\hX^{(1:k-2)})}{p^{(k)}(\hX^{(k)}|\hX^{(1:k-1)})}\right\}.
    \label{eq:threshold}
\end{align}
Specifically, each $\Theta^{(k)}$ can be used to define random variables $A^{(k)}\in\{\accept, \reject\}$ (where $\accept$ and $\reject$ indicate acceptance and rejection of draft tokens, respectively) such that $\prob\left\{A^{(k)}=\accept|\Theta^{(k)}=\theta\right\}=\theta$ for $\theta\in[0, 1]$.

\begin{figure}[t]
\begin{center}
\centerline{\includegraphics[width=\columnwidth]{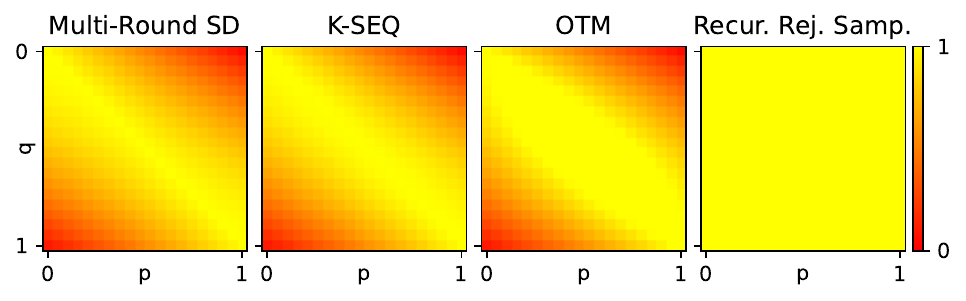}}
\vspace{-0.15in}
\caption{
Acceptance rates for multi-round speculative decoding, K-SEQ, OTM and recursive rejection sampling are given when $\mathrm{Ber}(p)$ and $\mathrm{Ber}(q)$ are draft and target distributions, respectively, and two tokens are proposed by the draft model ($K=2$). 
}
\label{fig:toy}
\end{center}
\vspace{-0.3in}
\end{figure}

\begin{figure*}[t]
\begin{center}
\centerline{\includegraphics[width=\textwidth]{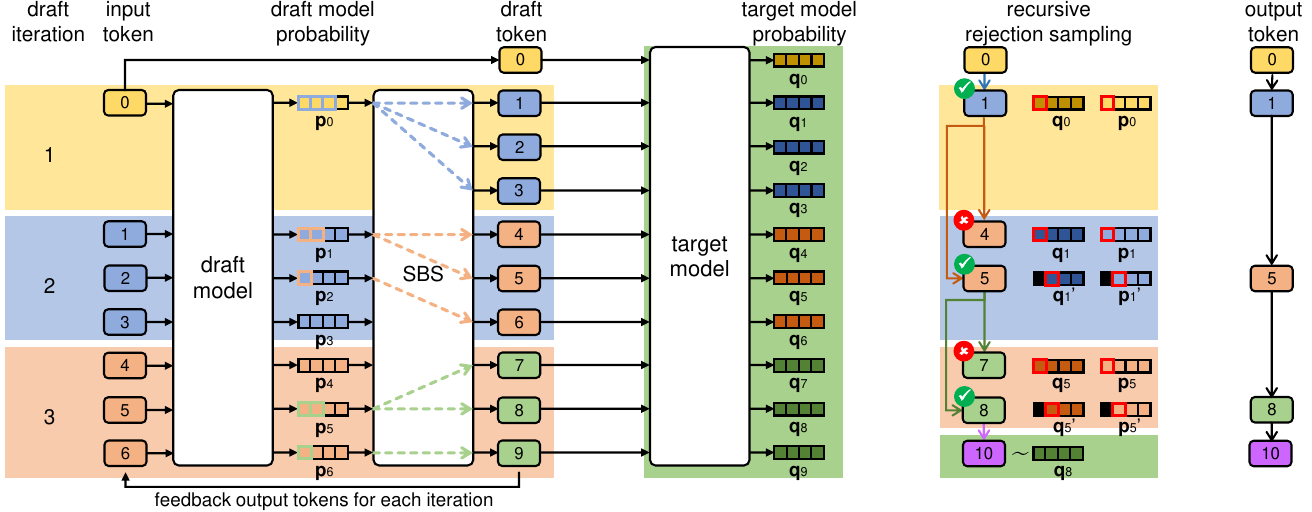}}
\caption{We describe the entire process of RSD with Stochastic Beam Search (RSD-S); the difference between RSD-S and RSD with Constant branching factors (RSD-C) lies at the method of constructing the draft-token tree. 
Draft tokens the tree are sampled in parallel at each level and auto-regressively across levels, while Stochastic Beam Search samples sequences without replacement at each tree level. The established draft-token tree is then processed by the target model in parallel, which lets us acquire the token-wise target model probabilities. Finally, recursive rejection sampling (for sampling-without-replacement distribution) is applied to each level of the tree, recovering the sequence generation distribution of the target model.}
\label{fig:rsd_big_picture}
\end{center}
\vspace{-0.25in}
\end{figure*}
Finally, recursive rejection sampling can be characterized by defining a random variable $Z\in\mathcal{X}$ such that
\begin{align}
Z:=
\begin{cases}
    \hX^{(k)},           &\text{if~}A^{(1:k-1)}=\reject^{k-1}, A^{(k)}=\accept,\\
                        &{\color{white}\text{if~}}k=1, ..., K,\\
    Y,  &\text{if~}A^{(1:K)}=\reject^{K},\\
                        &{\color{white}\text{if~}}Y\sim q^{(K+1)}(\cdot|\hX^{(1:K-1)}),
\end{cases}
\label{eq:recursive_rejection_sampling}
\end{align}
where $A^{(1:k-1)}:=(A^{(1)}, ..., A^{(k-1)})$ and $\reject^{k}$ is a length-$k$ sequence with all of its elements equal to $\reject$.
Intuitively, we select $\hX^{(1)}$ if it is accepted ($A^{(1)}=\accept$); we select $\hX^{(k)}$ when all previous draft tokens $\hX^{(1)}, ..., \hX^{(k-1)}$ are rejected \emph{and} $\hX^{(k)}$ is accepted ($A^{(1:k-1)}=\reject^{k-1}, A^{(k)}=\accept$) for each $k$; we sample $Y\sim q^{(K+1)}(\cdot|\hX^{(1:K-1)})$ and select $Y$ if all draft tokens are rejected ($A^{(1:K)}=\reject^{K}$). 
We summarize the entire process of recursive rejection sampling in \textbf{Algorithm~\ref{alg:rrs_abs}}.
Note that the original rejection sampling~\cite{leviathan2023fast,chen2023accelerating} is a special case of our recursive rejection sampling with $K=1$. 
Also, it can be shown that recursive rejection sampling \eqref{eq:recursive_rejection_sampling} always recovers the target distribution $q$:
\begin{theorem}[Recursive rejection sampling recovers target distribution]
\label{thm:1}
For the random variable $Z\in\mathcal{X}$ in \eqref{eq:recursive_rejection_sampling}, 
\begin{align*}
    \prob\{Z=z\}=q(z), z\in\mathcal{X}.
\end{align*}
\end{theorem}
\begin{proof}
    See Appendix~\ref{appendix:A:1}.
\end{proof}
\vspace{-0.1in}
Although the proposed recursive rejection sampling is applicable to arbitrary distributions with dependencies following \eqref{eq:draft_samples}, we assume a single draft model (as in SpecTr~\cite{sun2023spectr} and focus on the cases where the draft model samples predictive tokens without replacement, which is an instance of \eqref{eq:draft_samples}.

\textbf{Toy example.} We present a didactic example with Bernoulli distributions (given by \citet{sun2023spectr}) to showcase the benefit of  recursive rejection sampling.
Suppose that Bernoulli distributions are used for both draft and target models and only $K=2$ tokens are allowed for draft proposals. The acceptance rates for different methods are depicted in Figure~\ref{fig:toy}; multi-round speculative decoding (from SpecInfer~\cite{miao2023specinfer}), K-SEQ and Optimal Transport with Membership costs (OTM)~\cite{sun2023spectr}, use sampling \emph{with} replacement, whereas recursive rejection sampling uses sampling \emph{without} replacement; note that both K-SEQ and OTM were presented in SpecTr paper~\cite{sun2023spectr} where OTM shows theoretically optimal acceptance rate.
For all the baselines, acceptance rates decrease as the discrepancy between draft and target distribution increases, since tokens sampled from draft models become more unlikely from target models. 
On the other hand, recursive rejection sampling achieves 100\% acceptance rate even with high draft-target-model discrepancy; once the first draft token is rejected, the second draft token is always aligned with the residual distribution. This example shows that draft distributions with dependencies, e.g., sampling-without-replacement distribution, leads to higher acceptance rate and becomes crucial, especially for the cases with higher distributional discrepancy between draft and target.


\newcommand{\vocabsize}{N_{\mathrm{vocab}}}
\newcommand{\inputseq}{\mathbf{x}_{\mathrm{input}}}
\newcommand{\outputlen}{L_{\mathrm{output}}}
\newcommand{\inputlen}{L_{\mathrm{input}}}
\newcommand{\subs}{\leftarrow}
\newcommand{\numtreenodes}{N_{tree}}
\newcommand{\attentionmask}{\mathbf{M}}
\newcommand{\createemptytensor}{\mathtt{CreateEmptyTensor}}
\newcommand{\createdrafttreeconst}{\mathtt{CreateDraftTreeConst}}
\newcommand{\createdrafttreesbs}{\mathtt{CreateDraftTreeStochasticBeamSearch}}
\newcommand{\getlength}{\mathtt{GetLength}}
\newcommand{\kvcache}{\mathbf{C}}
\newcommand{\kvcachedraft}{\kvcache_{\mathrm{draft}}}
\newcommand{\kvcachetarget}{\kvcache_{\mathrm{target}}}

\newcommand{\listlogprobsdraft}{\mathcal{L}_{\mathrm{log\_probs\_draft}}}
\newcommand{\listlogprobstarget}{\mathcal{L}_{\mathrm{log\_probs\_target}}}
\newcommand{\listnumnodes}{\mathcal{L}_{\mathrm{num\_nodes}}}
\newcommand{\listflatnodeids}{\mathcal{L}_{\mathrm{flat\_node\_ids}}}
\newcommand{\listflatnodeidsaccepted}{\mathcal{L}_{\mathrm{accepted\_flat\_node\_ids}}}
\newcommand{\listparentids}{\mathcal{L}_{\mathrm{parent\_ids}}}
\newcommand{\listdrafttokens}{\mathcal{L}_{\mathrm{draft\_tokens}}}
\newcommand{\listdrafttokensaccepted}{\mathcal{L}_{\mathrm{accepted\_draft\_tokens}}}
\newcommand{\listpositionids}{\mathcal{L}_{\mathrm{position\_ids}}}
\newcommand{\drafttree}{\mathcal{T}}
\newcommand{\beam}{\mathcal{B}}
\newcommand{\positionids}{\mathbf{id}_{\mathrm{position}}}
\newcommand{\parentids}{\mathbf{id}_{\mathrm{parent}}}
\newcommand{\flatnodeids}{\mathbf{id}_{\mathrm{flat\_node}}}
\newcommand{\nodeid}{i_{node}}
\newcommand{\flatnodeidsaccepted}{\mathbf{id}_{\mathrm{accepted\_flat\_node}}}
\newcommand{\currdraftlen}{l_{\mathrm{draft}}}
\newcommand{\logprobs}{\boldsymbol{\Phi}}
\newcommand{\sumlogprob}{\boldsymbol{\Sigma}}
\newcommand{\truncatedgumbel}{\boldsymbol{\Gamma}}
\newcommand{\perturbedlogprobs}{\tilde{\boldsymbol{\Phi}}}
\newcommand{\maxperturbedlogprobs}{\tilde{\boldsymbol{\Phi}}_{\mathrm{max}}}
\newcommand{\logprobsdraft}{\logprobs_{\mathrm{draft}}}
\newcommand{\logprobstarget}{\logprobs_{\mathrm{target}}}
\newcommand{\logprob}{\boldsymbol{\phi}}
\newcommand{\logprobdraft}{\logprob_d}
\newcommand{\logprobtarget}{\logprob_t}
\newcommand{\draftmodelforwardpass}{\mathtt{DraftModelForwardPass}}
\newcommand{\draftinputseq}{\mathbf{x}_{in,d}}
\newcommand{\samplewithgumbeltopk}{\mathtt{SampleWithGumbelTopK}}
\newcommand{\samplewithstochasticbeam}{\mathtt{SampleWithStochasticBeam}}
\newcommand{\concatenate}{\mathtt{Concat}}
\newcommand{\stack}{\mathtt{Stack}}
\newcommand{\buildattentionmask}{\mathtt{BuildAttentionMask}}
\newcommand{\arange}{\mathtt{Arange}}
\newcommand{\astart}{\mathtt{start}}
\newcommand{\aend}{\mathtt{end}}
\newcommand{\dimension}{\mathtt{dim}}
\newcommand{\tokens}{\mathbf{x}}
\newcommand{\drafttokens}{\tokens_{\mathrm{draft}}}
\newcommand{\acceptedtokens}{\tokens_{\mathrm{accepted}}}
\newcommand{\lasttoken}{\tokens_{\mathrm{last}}}
\newcommand{\numnodecurr}{N_{\mathrm{tree\_curr}}}
\newcommand{\numnodeprev}{N_{\mathrm{tree\_prev}}}
\newcommand{\append}{\mathtt{.append}}
\newcommand{\flatten}{\mathtt{.flatten}}
\newcommand{\clamp}{\mathtt{.clamp}}
\newcommand{\filldiagonal}{\mathtt{.fill\_diagonal}}
\newcommand{\splittensor}{\mathtt{SplitTensor}}
\newcommand{\numnodes}{N_{\mathrm{nodes}}}
\newcommand{\tokenisaccepted}{\mathrm{accept}}
\newcommand{\istrue}{\mathrm{True}}
\newcommand{\isfalse}{\mathrm{False}}
\newcommand{\gumbel}{\mathbf{G}}

\newcommand{\getpositionidsfromtree}{\mathtt{GetPositionIDsFromTree}}
\newcommand{\targetmodelforwardpass}{\mathtt{TargetModelForwardPass}}
\newcommand{\recursiverejectionsampling}{\mathtt{RecursiveRejectionSampling}}
\newcommand{\filterkvcache}{\mathtt{FilterKVCache}}
\newcommand{\logsumexp}{\mathtt{LogSumExp}}
\newcommand{\createidentitymatrixdiagonalmatrix}{\mathtt{Create}}
\newcommand{\zeropadding}{\mathtt{ZeroPadding}}

\newcommand{\Comment}[1]{{\color{gray}\texttt{#1}}}
\newcommand{\Red}[1]{{\color{red}#1}}

\subsection{Tree-Based Speculative Decoding with Recursive Rejection Sampling}
\label{sec:tree-based-sd-with-rrs}

Recursive rejection sampling is applicable to tree-based speculative decoding algorithms if sampling without replacement is used to construct a \emph{draft-token tree}. Two Recursive Speculative Decoding (RSD) algorithms using recursive rejection sampling are presented in this section, while they share the same pipeline for parallel target evaluation and draft tree verification after building the draft-token tree (See \emph{Figure~\ref{fig:rsd_big_picture}.}).
We describe details about how RSD works in the following sections.

\subsubsection{Draft-token tree generation}

We consider two RSD algorithms: \textbf{RSD} with \textbf{C}onstant branching factors (RSD-C) and \textbf{RSD} with \textbf{S}tochastic Beam Search (RSD-S). RSD-C builds the draft-token tree having constant branching factors, which makes sequences from the tree to have the same length. RSD-S, on the other hand, builds the tree via Stochastic Beam Search~\cite{kool2019stochastic} that 
samples \emph{draft sequences} without replacement, while truncating sequences that are unlikely to be generated from the draft model and efficiently handling the computational cost. 

\textbf{RSD with Constant branching factors (RSD-C).} Let $L$ denote the fixed length for all draft sequences, which is equivalent to the depth of the draft-token tree, and $\btau_0^{(1)}$ denote the input sequence of tokens. 
Let us assume that the tree level increases from root ($l=0$) to leaf ($l=L$) nodes, where each node is characterized by the (partial) sequence.
We also define $\mathbf{b}:=(b_0, ..., b_{L-1})$ where $b_l$ is the branching factor at the level $l$ (See \textit{Figure~\ref{fig:rsd_example}(a)} for the example with $\mathbf{b}=(3,2,1)$.).

At each level $l\in\{0, ..., L-1\}$ of the draft tree, we begin with  
$N_l$ sequences 
$
\btau_l^{(k)}, k=1, ..., N_l   
$
generated from the previous level, where
$N_0:=1$ and $N_l:=\prod_{l'=0}^{l-1}b_{l'}$ for $l\ge 1$. Then, we evaluate log probabilities $\phi_{l}(\btau_l^{(k)}, \cdot)$ and perturbed log probabilities $\tilde{\phi}_{l}(\btau_l^{(k)}, \cdot)$ for each $k$, i.e., 
for i.i.d. Gumbel samples $G_l^{(k)}(x), x\in\mathcal{X},$
\begin{align}
    \phi_{l}(\btau_l^{(k)}, \cdot)&:=\log p(\cdot|\btau_l^{(k)}),
    \label{eq:log_probs}
    \\
    \tilde{\phi}_{l}(\btau_l^{(k)}, \cdot)&:=\phi_{l}(\btau_l^{(k)}, \cdot)+G_l^{(k)},
    \label{eq:per_log_probs}
\end{align}
where both log probabilities and Gumbel samples can be computed in parallel; proper positional encodings and attention masking~\cite{medusa, miao2023specinfer} are required for the parallel log-probability computation when transformer architecture is used~\cite{vaswani2017attention}.
By using \emph{Gumbel-Top-$k$ trick}~\cite{vieira2014gumbel,kool2019stochastic} with perturbed log probabilities \eqref{eq:per_log_probs}, one can sample top-$b_l$ tokens without replacement for each sequence $\tau_l^{(k)}$:
\begin{align}
    \hX_{l+1}^{((k-1)b_l+1)}, ..., \hX_{l+1}^{((k-1)b_l+b_l)}
    =\underset{x\in\mathcal{X}}{\mathrm{argtop}\text{-}b_l}\left(\tilde{\phi}_{l}(\btau_l^{(k)}, x)\right).
    \label{eq:rsd_c_level_l}
\end{align}
Note that the outputs $\hX_{l+1}^{((k-1)b_l+k')}, k'=1, ..., b_l,$ in \eqref{eq:rsd_c_level_l} are assumed to be in the decreasing order of values $\tilde{\phi}_{l}(\btau_l^{(k)}, \hX_{l+1}^{((k-1)b_l+k')})$, 
for each $k$. 
Finally, we define
\begin{align}
    O_{l+1}^{((k-1)b_l+k')}&:=(\hX_{l+1}^{((k-1)b_l+k')}, k),\label{eq:pairs_def}\\
    \btau_{l+1}^{((k-1)b_l+1)}&:=(\btau_l^{(k)}, \hX_{l+1}^{((k-1)b_l+1)})
\end{align}
for $k\in{1, ..., N_l}$ and $k'\in\{1, ..., b_l\}$, where $O_{l+1}^{((k-1)b_l+k')}$ is a pair of draft token and parent sequence index.
Those pairs in \eqref{eq:pairs_def} are stored for all levels $l=0, ..., L-1$ and used for draft tree verification, 
which exploits
the fact that \emph{the tokens
$\hX_{l+1}^{((k-1)b_l+1)}, ..., \hX_{l+1}^{((k-1)b_l+b_l)}$
follow sampling without replacement from $p(\cdot|\btau_l^{(k)})$ for any given parent sequence index $k$}.

\begin{figure}[t]
\begin{center}
\begin{tikzpicture}
\node at (0,0) {\includegraphics[width=\columnwidth]{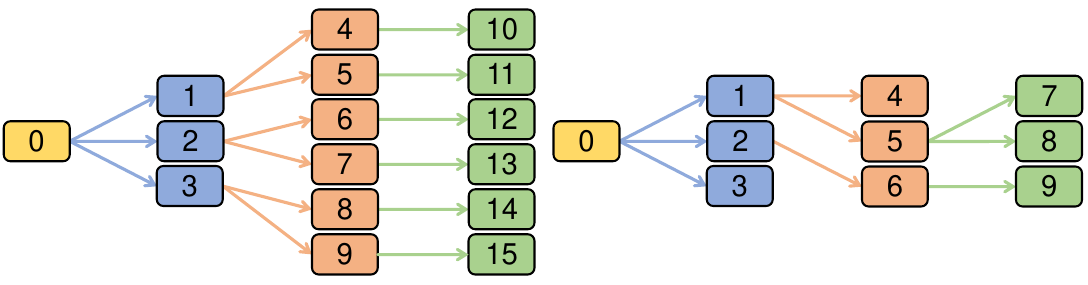}};
\node at (-2, -1.5){\small \textit{(a)} RSD-C};
\node at ( 2, -1.5){\small \textit{(b)} RSD-S};
\end{tikzpicture}
\vspace{-0.2in}
\caption{We describe examples of constructing draft-token trees with the (maximum) draft length equal to 3; \textit{(a)} The tree constructed by RSD-C with branching factors $\mathbf{b}=(3, 2, 1)$ is given; \textit{(b)} we depict the tree constructed by RSD-S with beamwidth $W=3$, where edges are determined via Stochastic Beam Search. 
}
\label{fig:rsd_example}
\end{center}
\vspace{-0.2in}
\end{figure}

\textbf{RSD with Stochastic Beam Search (RSD-S).} 
One caveat of RSD-C is that its constant branching factors $\mathbf{b}$ should be carefully determined to handle tree complexity, when the computation budget is limited; for example, if $\mathbf{b}=(n, ..., n)$ with its length $L$, the number of nodes in the draft tree will be $\sum_{l=0}^{L-1}n^l=O(n^{L-1})$, which is computationally prohibitive for large $n$ and $L$. Also, RSD-C constructs sequences at each level $l$ by using the \emph{myopic} token-wise log probabilities $\phi_l$ in \eqref{eq:per_log_probs}.
RSD-S addresses both issues by using Stochastic Beam Search~\cite{kool2019stochastic} that early-truncates unlikely sequences and utilizes \emph{far-sighted} sequence log probabilities.

Let us define the \emph{maximum} draft sequence length $L$ and the beamwidth $W$.
We also define $\btau_0^{(1)}$ as the input sequence similar to RSD-C. 
At each level $l\in\{0, ..., L-1\}$, SBS uses \emph{beam}
\begin{align*}
\mathcal{B}_l
&:=(t_l^{(1)},...,t_l^{(W)}),\\
t_l^{(k)}
&:=
(
\btau_{l}^{(k)}, \phi_{l-1}(\btau_{l}^{(k)}), \psi_{l-1}(\btau_{l}^{(k)})
)
\end{align*}
generated from the previous level $l-1$\footnote{For $l=0$, $\phi_{-1}(\btau_{0}^{(1)})=\phi_{-1}(\btau_{0}^{(1)})=0$ is used with $\mathcal{B}_0:=(t_0^{(1)})$~\cite{kool2019stochastic}.}.
Here, each tuple $t_l^{(k)}$ for $k\in\{1, ..., W\}$ consists of
\emph{(a)} a sequence $\btau_{l}^{(k)}$,
\emph{(b)} its \emph{sequence} log probability $\phi_{l-1}(\btau_{l}^{(k)})$ of $\btau_{l}^{(k)}$, 
and
\emph{(c)} the \emph{transformed (perturbed and truncated)} sequence log probability $\psi_{l-1}(\btau_{l}^{(k)})$, 
respectively. 

For each tuple $t_l^{(k)}$ in the beam $\mathcal{B}_l$, we evaluate the (next-level) sequence log probabilities $\phi_{l}(\btau_l^{(k)}, \cdot)$ and the perturbed sequence log probabilities $\tilde{\phi}_{l}(\btau_l^{(k)}, \cdot)$. Specifically for i.i.d. Gumbel samples $G_l^{(k)}(x), x\in\mathcal{X}$, we compute
\begin{align*}
    \phi_{l}
    (\btau_l^{(k)}, \cdot)
    &:=
    \Red{\phi_{l-1}(\btau_l^{(k)})}+
    \log p(\cdot|\Red{\btau_l^{(k)}}),\\
    \tilde{\phi}_{l}(\btau_l^{(k)}, \cdot)
    &:=\phi_{l}(\btau_l^{(k)}, \cdot)+G_l^{(k)},
\end{align*}
where the terms \Red{$\btau_{l}^{(k)}$} and \Red{$\phi_{l-1}(\btau_{l}^{(k)})$} within the tuple $t_l^{(k)}$ of within the beam $\mathcal{B}_l$ are reused.
Similar to RSD-C,  both log probabilities and Gumbel samples can be parallelly computed with positional encodings and attention masking~\cite{medusa, miao2023specinfer}.
In addition to the perturbed log probabilities, SBS in RSD-S transforms $\tilde{\phi}_{l}(\btau_l^{(k)}, \cdot)$ into the truncated function
\begin{align}
    \psi_l(\btau_l^{(k)}, \cdot)
    &:=T(\Red{\psi_{l-1}(\btau_l^{(k)})}, \tilde{\phi}_{l}(\btau_l^{(k)}, \cdot)),
    \label{eq:transformation}
    \\
    T(u, \phi)&:=
    -\log
    \left(
        e^{-u}
        -
        e^{-\mathrm{max}\phi}
        +
        e^{-\phi(\cdot)}
    \right)
    \label{eq:mapping}
\end{align}
for $\max\phi:=\max_{x\in\mathcal{X}}\phi(x)$
by reusing \Red{$\psi_{l-1}(\btau_l^{(k)})$} in $t_l^{(k)}$ . 
Note that $T(u,\phi)$ in \eqref{eq:mapping} is \emph{monotonically increasing} w.r.t. $\phi$ and transforms $\phi$ to the function with the upper bound $u$~\cite{kool2019stochastic}\footnote{In Appendix B.3 of \citet{kool2019stochastic}, a numerical stable way of evaluating the function $T$ in \eqref{eq:mapping} is provided.}

After evaluating $\psi_l(\btau_l^{(k)}, \cdot)$ for all parent sequences $\btau_l^{(k)}$s, SBS selects top-$W$ pairs $(\hX_{l+1}, p_{l+1})$ of draft token and parent sequence index \emph{across the beam $\mathcal{B}_l$}, i.e., 
\begin{align}
    O_{l+1}^{(1)}, ..., O_{l+1}^{(W)}
    :=\underset{(x, k)\in\mathcal{X}\times\mathcal{K}}{\mathrm{argtop}\text{-}W}
    \left(\psi_l(\btau_l^{(k)},x)\right)
    \label{eq:argtop_pair}
\end{align}
for $O_{l+1}^{(k)}:=(\hX_{l+1}^{(k)}, p_{l+1}^{(k)})$ and $\mathcal{K}:=\{1, ..., W\}$.
The output pairs $O_{l+1}^{(1)}, ..., O_{l+1}^{(W)}$ are given by \emph{corresponding values $\psi_l(\btau_l^{(k)},\hX_{l+1}^{(k)})$ in the decreasing order}. 
Finally, we construct the next beam
\begin{align*}
    \mathcal{B}_{l+1}&:=(t_{l+1}^{(1)}, ..., t_{l+1}^{(W)}),\\
    t_{l+1}^{(k)}&:=(
        (\hat{\btau}_{l+1}^{(k)}, \hX_{l+1}^{(k)}),
        \phi_l(\hat{\btau}_{l+1}^{(k)}, \hX_{l+1}^{(k)}),
        \psi_l(\hat{\btau}_{l+1}^{(k)}, \hX_{l+1}^{(k)})
    )
\end{align*}
for $k=1, ..., W$, where $\hat{\btau}_{l+1}^{(k)}:=\btau_{l}^{(p_{l+1}^{(k)})}$ is the selected parent sequence. 
Intuitively, SBS at the level $l$ evaluates scores $\psi_l^{(k)}(\btau_l^{(k)},x), x\in\mathcal{X},k\in\mathcal{K},$ by considering all child nodes from the beam $\mathcal{B}_l$. 
SBS selects $W$ nodes among all child nodes having top-$W$ scores.
Note that the above process is theoretically equivalent to sample top-$W$ length-$(l+1)$ sequences \emph{without replacement}~\cite{kool2019stochastic} and efficiently truncates sequences that are unlikely to be generated. (See \emph{Figure \ref{fig:rsd_example}(b)}.)

We store the \emph{ordered} sequence of pairs $O_{l+1}^{(1)}, ..., O_{l+1}^{(W)}$ for all levels $l=0, ..., L-1$, which is used for draft-tree verification. 
As in RSD-C, we show the following property:
\begin{theorem}[Tokens from the same sequence follow sampling without replacement in RSD-S]
\label{thm:2}
In RSD-S, any non-empty subsequence of the sequence $\hat{X}_{l+1}^{(1)}, ..., \hat{X}_{l+1}^{(W)}$ of draft tokens (from 
$O_{l+1}^{(1)}, ..., O_{l+1}^{(W)}$ in \eqref{eq:argtop_pair}) such that each element of the subsequence has the same parent $\btau_l^{(k)}$ follows sampling without replacement from $p(\cdot|\btau_l^{(k)})$\footnote{We define a subsequence of a sequence as any sequence acquired by removing its elements \emph{while maintaining the order in the original sequence}.}.
\end{theorem}
\vspace{-0.15in}
\begin{proof}
    See Appendix~\ref{appendix:A:2}.
\end{proof}

\subsubsection{Draft-tree evaluation and verification}

\textbf{Tree evaluation with target model.} After the draft-tree construction, we have sequences of pairs
\begin{align*}
    (O_l^{(1)}, ..., O_l^{(N_l)}), l=1, ..., L,
\end{align*}
where $N_l=\prod_{l'=0}^l b_{l'}$ for RSD-C and $N_l=W$ for RSD-S, respectively ($N_0:=1$ for both). Those pairs include the node-connection information of the draft tree and can be used to \emph{parallelly} evaluate the draft tree via the target model by utilizing appropriate attention masking and positional encodings. From the evaluation process, we acquire the target log probabilities for all sequences $\btau_l^{(k_l)}$ in the draft tree, i.e., 
\begin{align*}
    q(\cdot|\btau_l^{(k_l)}), l=0,...,L, k_l=1,...,N_l.
\end{align*}
\textbf{Verification via recursive rejection sampling.} Earlier, we show that tokens in the tree having the same parent sequence $\btau_l^{(k_l)}$ follows the sampling-without-replacement distribution from $p(\cdot|\btau_l^{(k_l)})$ for both RSD-C and RSD-S. 
Thus, one can apply recursive rejection sampling iteratively at each tree level. 

Specifically, at the level $l\in\{0, 1, ..., L\}$, we begin with a sequence $\btau_l^{(k_l')}$ where $k_l'$ is the index of the parent sequence accepted in the previous level ($k_0'=1$ at the level $l=0$). 
Within the \emph{ordered} sequence 
$(O_{l+1}^{(1)}, ..., O_{l+1}^{(N_{l+1})})$ of pairs, we find the subsequence $\mathbf{o}_{l+1}^{(k_l')}$ having $\btau_l^{(k_l')}$ as parent, which can be validated by checking the second element of each pair $O_{l+1}^{(k)}$, and the token sequence $\mathbf{x}_{l+1}^{(k_l')}$ in $\mathbf{o}_{l+1}^{(k_l')}$.
Earlier, we show that tokens $\mathbf{x}_{l+1}^{(k_l')}$ follows sampling-without-replacement distribution in its order, so we can apply recursive rejection sampling to those tokens with draft and target distributions, $p(\cdot|\btau_l^{(k_l')})$ and $q(\cdot|\btau_l^{(k_l')})$, respectively. 
If any token $x$ in $\mathbf{x}_{l+1}^{(k_l')}$ is accepted, we set $k_{l+1}'$ that corresponds to $\btau_l^{(k_{l+1}')}:=(\btau_l^{(k_l')}, x)$, and we continue to the next-level verification if child nodes exist.
If all tokens are rejected, we sample from the last residual distribution of recursiver rejection sampling. 
If there is no child node, we sample from the target $q(\cdot|\btau_l^{(k_l')})$ similar to the single-sequence speculative decoding~\cite{chen2023accelerating,leviathan2023fast}.
We provide detailed descriptions for RSD-C (\textbf{Algorithm \ref{alg:rsd_c}}) and for RSD-S (\textbf{Algorithm \ref{alg:recursive_speculative_decoding_with_stochastic_beam_search}}) in Appendix~\ref{appendix:sec:algorithm}.

\begin{figure*}[t]
\begin{center}
\def\x{\columnwidth+0.2}
\begin{tikzpicture}
    \node at (0,0) {\includegraphics[width=0.95\textwidth]{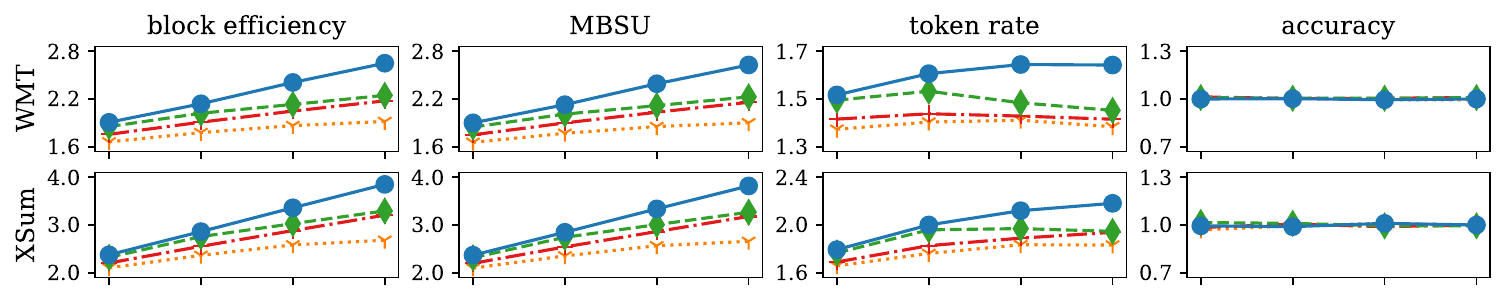}};
    \node [rotate=90] at (-\x, 0.0) {\small Llama 2-70B};
    \node at (0,-4.2) {\includegraphics[width=0.95\textwidth]{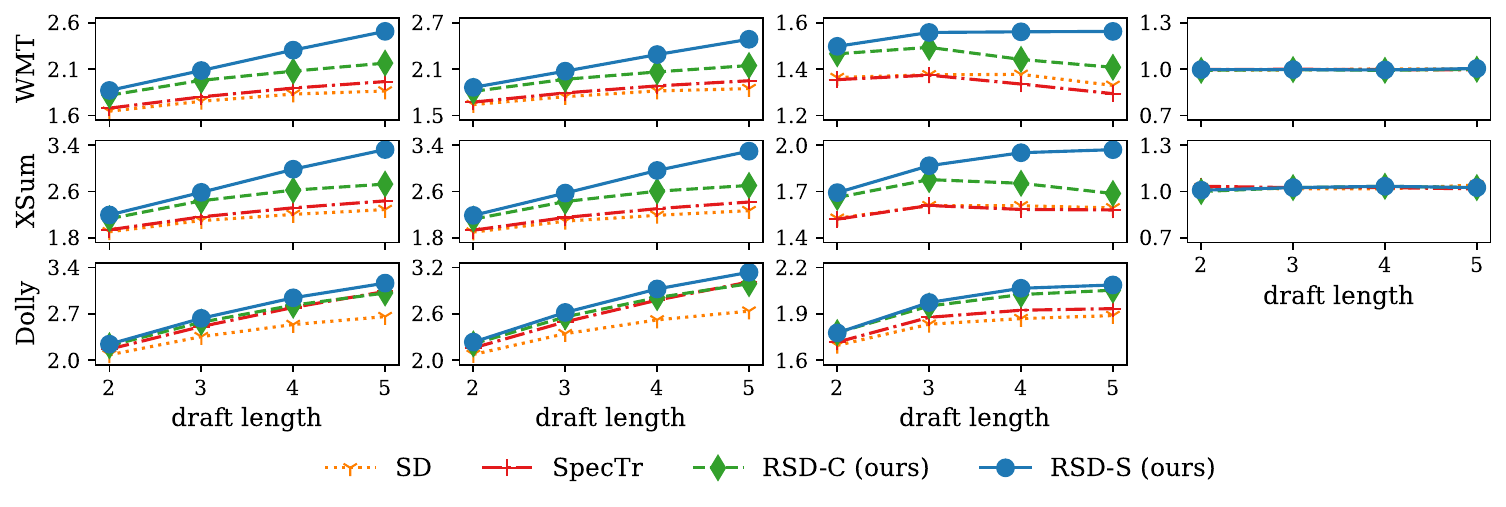}};
    \node [rotate=90] at (-\x, -3.5) {\small Llama 2-Chat-70B};
    \draw [dotted] (-\x, -1.52) -- (\x, -1.52);
\end{tikzpicture}
\vspace{-0.2in}
\caption{Block efficiency, MBSU, token rate and accuracy for various lengths ($2,3,4,5$) of draft sequences are given. We consider two target models, Llama 2-70B and Llama 2-Chat-70B, each of which has a corresponding smaller draft model for speculative decoding. All results are normalized by the corresponding numbers from auto-regressive decoding. RSD-S always outperforms SD, SpecTr and RSD-C.
All methods including auto-regressive decoding show similar accuracy for WMT and XSum.
}
\label{fig:main:exp1}
\end{center}
\vspace{-0.2in}
\end{figure*}
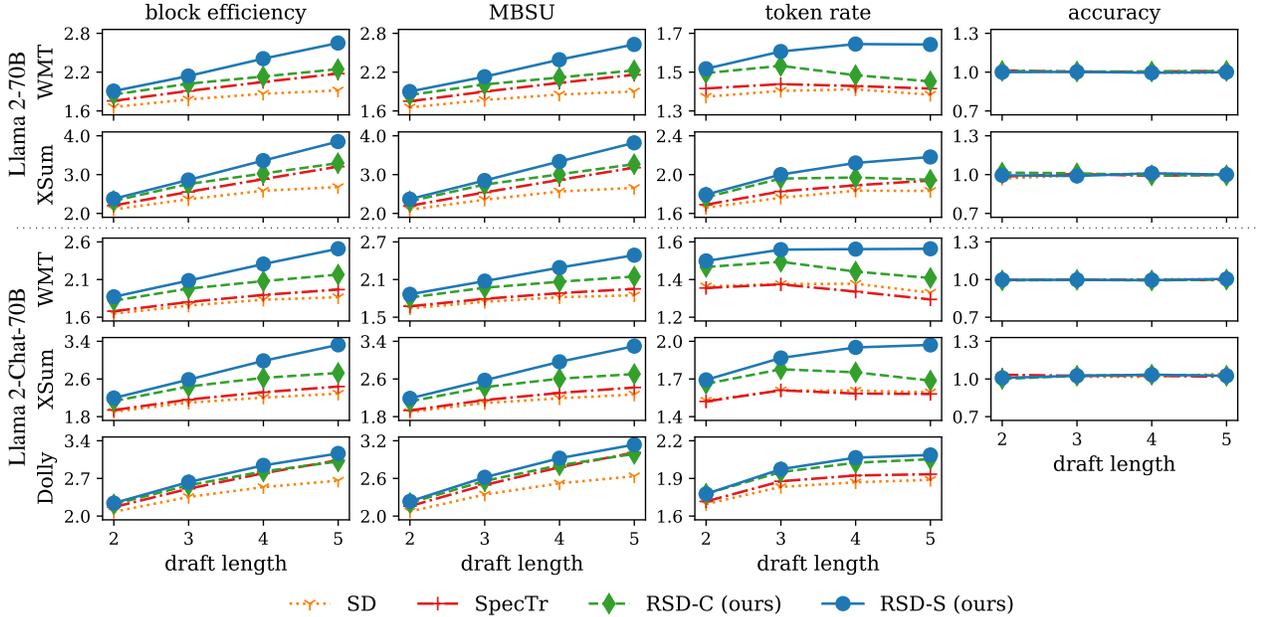

\begin{figure*}[t]
\begin{center}
\def\x{\columnwidth+0.2}
\begin{tikzpicture}
    \node at (0,0) {\includegraphics[width=0.95\textwidth]{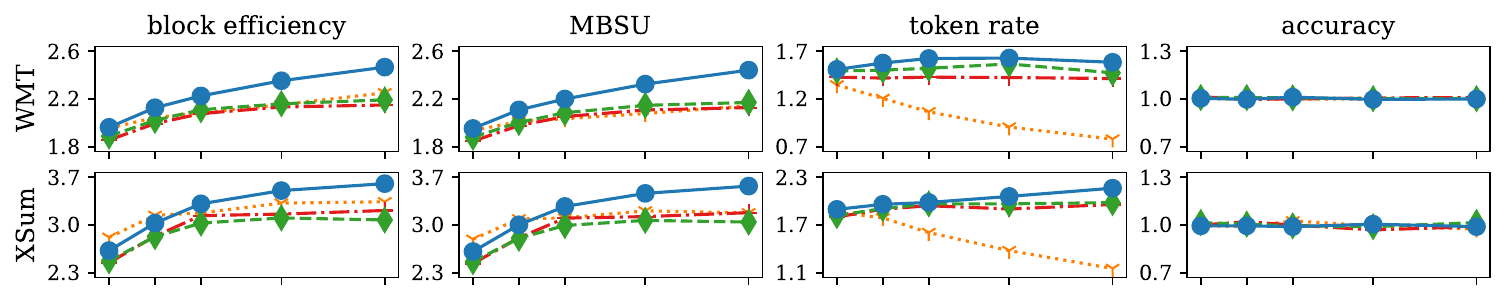}};
    \node [rotate=90] at (-\x, 0.0) {\small Llama 2-70B};
    \node at (0,-4.2) {\includegraphics[width=0.95\textwidth]{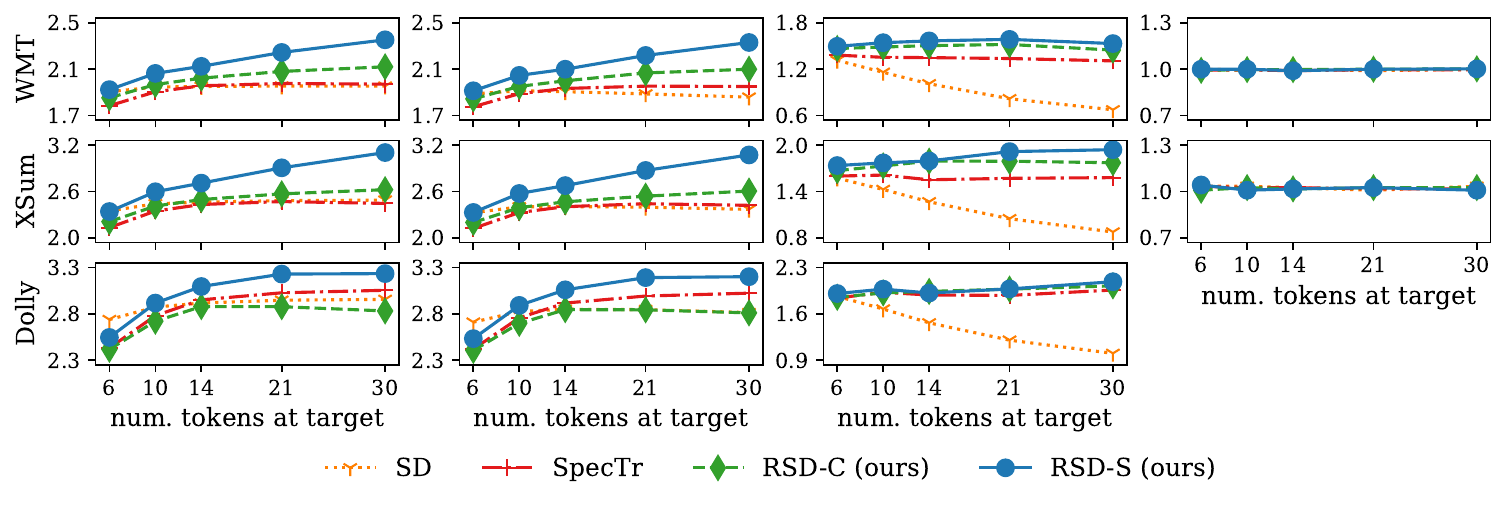}};
    \node [rotate=90] at (-\x, -3.5) {\small Llama 2-Chat-70B};
    \draw [dotted] (-\x, -1.52) -- (\x, -1.52);
\end{tikzpicture}
\vspace{-0.2in}
\caption{Block efficiency, MBSU, token rate and accuracy for various target computational budgets (the numbers $6, 10, 14, 21, 30$ of draft tokens processed at the target model) are given. We consider two target models, Llama 2-70B and Llama 2-Chat-70B, each of which has a corresponding smaller draft model for speculative decoding. All results are normalized by the corresponding numbers from auto-regressive decoding. RSD-S outperforms SD, SpecTr and RSD-C in the majority of cases. 
All methods including auto-regressive decoding show similar accuracy for both WMT and XSum.
}
\label{fig:main:exp2}
\end{center}
\vspace{-0.1in}
\end{figure*}

\section{Related Works}

Many recent works have aimed to address the inference bottleneck of LLMs caused by auto-regressive decoding.
Speculative decoding methods~\cite{leviathan2023fast, chen2023accelerating,sun2023spectr,miao2023specinfer} use the target model (LLM) with a draft model (a small language model), while recovering target distribution via rejection sampling. See the recent survey on speculative decoding~\cite{xia2024unlocking} for more comprehensive understanding.

Other than speculative decoding methods, 
BiLD~\cite{kim2023big} is another method to accelerate inference, where it uses a fallback policy which determines when to invoke the target model and a rollback policy to review and correct draft tokens. 
Medusa~\cite{cai2024medusa} uses multiple decoding heads to predict future tokens in parallel, constructs the draft-token tree and uses a typical acceptance criteria. 
Lookahead decoding~\cite{fu2023lookahead} caches the historical $n$-grams  generated on-the-fly instead of having a draft model and performs parallel decoding using Jacobi iteration and verifies $n$-grams from the cache.
While showing promising results with greedy sampling, these works do not guarantee target distribution recovery in contrast to speculative decoding methods.

\section{Experiments}

We evaluate RSD-C and RSD-S together with our baselines including speculative decoding (SD)~\cite{chen2023accelerating,leviathan2023fast} and SpecTr~\cite{sun2023spectr}, where a single draft model is assumed\footnote{We exclude SpecInfer~\cite{miao2023specinfer} from our baselines since it uses multiple draft models.}.
We consider the following perspectives during our experiments: 
\textbf{(\textit{Exp1})} How will the performance be affected by \emph{the length of draft sequences}? 
\textbf{(\textit{Exp2})} How will the performance be affected by \emph{the target computational budget}, i.e., the number of tokens processed at the target model?
While \textbf{(\textit{Exp1})} has been frequently investigated by existing tree-based speculative decoding methods~\cite{sun2023spectr,miao2023specinfer}, \textbf{(\textit{Exp2})} has not been considered in prior works but has practical importance when running the target model on resource-bounded devices. 
\newline
\textbf{Models.} We consider the following target models; \textbf{Llama 2} and \textbf{Llama 2-Chat}~\cite{touvron2023llama} with \textbf{7B}, \textbf{13B} and \textbf{70B} parameters;
\textbf{OPT}~\cite{zhang2022opt} with \textbf{13B}, \textbf{30B} and \textbf{66B} parameters.
Each class of target models adopts corresponding draft model; see Appendix~\ref{sec:appendix:C:draft_models}.
In this section, we only present Llama 2-70B and Llama 2-Chat-70B results, and other results (Llama 2 with other sizes and OPT) can be found in Appendix~\ref{sec:appendix:C:experiments}.
\newline
\textbf{Tasks.} Our methods and baselines are evaluated for \textbf{WMT}18-DeEn~\citep[translation]{bojar-EtAl:2018:WMT1} and \textbf{XSum}~\citep[summarization]{narayan2018don} for each target model, while we report accuracy scores (BLEU for WMT and ROUGE-2 for XSum) to confirm if the target model's distribution is recovered; 
Databricks-\textbf{Dolly}-15k~\citep[question and answering]{DatabricksBlog2023DollyV2} is used only for Llama 2-Chat without accuracy evaluation. 
We use temperature 0.3 for both XSum and WMT and 1.0 for Dolly, where we further apply nucleus (top-$p$) sampling~\cite{holtzman2019curious} with $p=0.95$ for Dolly.
\newline
\textbf{Performance metrics.} We evaluate \textbf{block efficiency}~\cite{leviathan2023fast}, Memory-Bound Speed-Up (\textbf{MBSU})~\cite{zhou2023distillspec} and \textbf{token rate} (tokens/sec) on A100 GPUs; see Appendix~\ref{sec:appendix:C:metrics} for details.

\vspace{-0.05in}
\subsection{\textbf{\textit{(Exp 1)}} Fixed draft sequence length}

We fix (maximum) draft sequence length as the value in $\{2, 3, 4, 5\}$ and evaluate our methods and baselines, which is summarized in \emph{Figure~\ref{fig:main:exp1}}.
Regarding the tree structures of each decoding methods, we let both SpecTr and RSD-S always use draft-token trees, the size of which is smaller than or equal to that of RSD-C's tree; see Appendix~\ref{sec:appendix:tree_structure:exp1} for details. 
Our results show that tree-based methods (SpecTr, RSD-C and RSD-S) always outperform SD in terms of block efficiency and MBSU, whereas
token rates for SpecTr and RSD-C can be lower than that for SD; 
this is since block efficiencies for both SpecTr and RSD-C are relatively low and there is additional computational overhead to process the tree.
On the other hand, \emph{RSD-S strictly outperforms both SD and SpecTr for all performance metrics}, showing the superiority of RSD-S over our baselines and the importance of early-truncating unlikely draft sequences.
We also observe that there is no strong correlation between MBSU and token rate; this is since A100 GPUs used to measure token rates are \emph{not} memory-bound.
Furthermore, token rates in many cases are shown to decrease as the length of draft-token sequence becomes higher, which is due to the increased computation overhead to execute draft models with the longer draft sequence; however, one needs to be cautious since this result may not generally hold since token rate is hugely affected by the efficiency of software implementation and the devices which we execute the methods on.
Finally, in WMT and XSum, BLEU and ROUGE-2 scores are similar across different methods, respectively, which implies that all methods recover the distributions of target LLMs.

\vspace{-0.05in}
\subsection{\textbf{(\textit{Exp2})} Fixed target computational budget}
We select target computational budget, i.e., the number of draft tokens processed at the target model in parallel for each speculative decoding iteration, among values in $\{6, 10, 14, 21, 30\}$ and evaluate our proposed methods and baselines; we summarize the results in \emph{Figure~\ref{fig:main:exp2}} and describe tree structures in Appendix~\ref{sec:appendix:tree_structure:exp2}.
While RSD-S achieves higher block efficiency and MBSU than SD and SpecTr in most cases, SD beats RSD-C in the relatively low budget regime, e.g., $\{6, 10\}$ with Llama 2-70B and XSum, and $\{6\}$ with Llama 2-Chat-70B and Dolly. 
We believe that our draft models are well-aligned with corresponding target models for those cases (from the observation that block efficiencies of SD close to 3.0, which are significantly higher than the numbers in other cases, are achieved), and increasing the depth rather than the width of the tree could quickly increase the acceptance rate in such cases. 
In the high budget regime, on the other hand, RSD-S beats SD for both block efficiency and MBSU. 
In terms of token rate, RSD-S strictly outperforms our baselines, whereas SD's token rate severely decreases for higher target computation budgets due to the computational overhead caused by the draft model's auto-regressive decoding with the longer draft sequence. 

\vspace{-0.05in}
\section{Conclusion}

We present RSD algorithms, a novel tree-based speculative decoding method leveraging the full diversifiability of the draft-token tree; RSD-C efficiently samples draft tokens without replacement via Gumbel-Top-$k$ trick, while RSD-S uses Stochastic Beam Search and samples draft-token sequences without replacement. We also propose recursive rejection sampling that can verify the tree built by the sampling-without-replacement process and recovers the exact target model distribution. We show that RSD outperforms the baselines in most cases, supporting the importance of diverse drafting when accelerating LLM inference.


\bibliography{icml2024}

\begin{thebibliography}{25}
\providecommand{\natexlab}[1]{#1}
\providecommand{\url}[1]{\texttt{#1}}
\expandafter\ifx\csname urlstyle\endcsname\relax
  \providecommand{\doi}[1]{doi: #1}\else
  \providecommand{\doi}{doi: \begingroup \urlstyle{rm}\Url}\fi

\bibitem[Achiam et~al.(2023)Achiam, Adler, Agarwal, Ahmad, Akkaya, Aleman, Almeida, Altenschmidt, Altman, Anadkat, et~al.]{achiam2023gpt}
Achiam, J., Adler, S., Agarwal, S., Ahmad, L., Akkaya, I., Aleman, F.~L., Almeida, D., Altenschmidt, J., Altman, S., Anadkat, S., et~al.
\newblock {GPT}-4 technical report.
\newblock \emph{arXiv preprint arXiv:2303.08774}, 2023.

\bibitem[Bojar et~al.(2018)Bojar, Federmann, Fishel, Graham, Haddow, Huck, Koehn, and Monz]{bojar-EtAl:2018:WMT1}
Bojar, O.~r., Federmann, C., Fishel, M., Graham, Y., Haddow, B., Huck, M., Koehn, P., and Monz, C.
\newblock Findings of the 2018 conference on machine translation (wmt18).
\newblock In \emph{Proceedings of the Third Conference on Machine Translation, Volume 2: Shared Task Papers}, pp.\  272--307, Belgium, Brussels, October 2018. Association for Computational Linguistics.
\newblock URL \url{http://www.aclweb.org/anthology/W18-6401}.

\bibitem[Brown et~al.(2020)Brown, Mann, Ryder, Subbiah, Kaplan, Dhariwal, Neelakantan, Shyam, Sastry, Askell, et~al.]{brown2020language}
Brown, T., Mann, B., Ryder, N., Subbiah, M., Kaplan, J.~D., Dhariwal, P., Neelakantan, A., Shyam, P., Sastry, G., Askell, A., et~al.
\newblock Language models are few-shot learners.
\newblock \emph{Advances in Neural Information Processing Systems (NeurIPS)}, 33:\penalty0 1877--1901, 2020.

\bibitem[Cai et~al.(2023)Cai, Li, Geng, Peng, and Dao]{medusa}
Cai, T., Li, Y., Geng, Z., Peng, H., and Dao, T.
\newblock Medusa: Simple framework for accelerating llm generation with multiple decoding heads.
\newblock \url{https://github.com/FasterDecoding/Medusa}, 2023.

\bibitem[Cai et~al.(2024)Cai, Li, Geng, Peng, Lee, Chen, and Dao]{cai2024medusa}
Cai, T., Li, Y., Geng, Z., Peng, H., Lee, J.~D., Chen, D., and Dao, T.
\newblock Medusa: Simple llm inference acceleration framework with multiple decoding heads.
\newblock \emph{arXiv preprint arXiv:2401.10774}, 2024.

\bibitem[Chen et~al.(2023)Chen, Borgeaud, Irving, Lespiau, Sifre, and Jumper]{chen2023accelerating}
Chen, C., Borgeaud, S., Irving, G., Lespiau, J.-B., Sifre, L., and Jumper, J.
\newblock Accelerating large language model decoding with speculative sampling.
\newblock \emph{arXiv preprint arXiv:2302.01318}, 2023.

\bibitem[Conover et~al.(2023)Conover, Hayes, Mathur, Xie, Wan, Shah, Ghodsi, Wendell, Zaharia, and Xin]{DatabricksBlog2023DollyV2}
Conover, M., Hayes, M., Mathur, A., Xie, J., Wan, J., Shah, S., Ghodsi, A., Wendell, P., Zaharia, M., and Xin, R.
\newblock Free dolly: Introducing the world's first truly open instruction-tuned llm, 2023.
\newblock URL \url{https://www.databricks.com/blog/2023/04/12/dolly-first-open-commercially-viable-instruction-tuned-llm}.

\bibitem[Fu et~al.(2023)Fu, Bailis, Stoica, and Zhang]{fu2023lookahead}
Fu, Y., Bailis, P., Stoica, I., and Zhang, H.
\newblock Breaking the sequential dependency of {LLM} inference using lookahead decoding, November 2023.
\newblock URL \url{https://lmsys.org/blog/2023-11-21-lookahead-decoding/}.

\bibitem[Holtzman et~al.(2019)Holtzman, Buys, Du, Forbes, and Choi]{holtzman2019curious}
Holtzman, A., Buys, J., Du, L., Forbes, M., and Choi, Y.
\newblock The curious case of neural text degeneration.
\newblock \emph{arXiv preprint arXiv:1904.09751}, 2019.

\bibitem[Jiang et~al.(2023)Jiang, Sablayrolles, Mensch, Bamford, Chaplot, Casas, Bressand, Lengyel, Lample, Saulnier, et~al.]{jiang2023mistral}
Jiang, A.~Q., Sablayrolles, A., Mensch, A., Bamford, C., Chaplot, D.~S., Casas, D. d.~l., Bressand, F., Lengyel, G., Lample, G., Saulnier, L., et~al.
\newblock Mistral 7{B}.
\newblock \emph{arXiv preprint arXiv:2310.06825}, 2023.

\bibitem[Kim et~al.(2023)Kim, Mangalam, Malik, Mahoney, Gholami, and Keutzer]{kim2023big}
Kim, S., Mangalam, K., Malik, J., Mahoney, M.~W., Gholami, A., and Keutzer, K.
\newblock Big little transformer decoder.
\newblock \emph{arXiv preprint arXiv:2302.07863}, 2023.

\bibitem[Kool et~al.(2019)Kool, Van~Hoof, and Welling]{kool2019stochastic}
Kool, W., Van~Hoof, H., and Welling, M.
\newblock Stochastic beams and where to find them: The {G}umbel-{T}op-$k$ trick for sampling sequences without replacement.
\newblock In \emph{Proceedings of the 36th International Conference on Machine Learning (ICML)}, pp.\  3499--3508. PMLR, 2019.

\bibitem[Leviathan et~al.(2023)Leviathan, Kalman, and Matias]{leviathan2023fast}
Leviathan, Y., Kalman, M., and Matias, Y.
\newblock Fast inference from transformers via speculative decoding.
\newblock In \emph{Proceedings of the 40th International Conference on Machine Learning (ICML)}, 2023.

\bibitem[Miao et~al.(2023)Miao, Oliaro, Zhang, Cheng, Wang, Wong, Chen, Arfeen, Abhyankar, and Jia]{miao2023specinfer}
Miao, X., Oliaro, G., Zhang, Z., Cheng, X., Wang, Z., Wong, R. Y.~Y., Chen, Z., Arfeen, D., Abhyankar, R., and Jia, Z.
\newblock Spec{I}nfer: Accelerating generative {LLM} serving with speculative inference and token tree verification.
\newblock \emph{arXiv preprint arXiv:2305.09781}, 2023.

\bibitem[Narayan et~al.(2018)Narayan, Cohen, and Lapata]{narayan2018don}
Narayan, S., Cohen, S.~B., and Lapata, M.
\newblock Don't give me the details, just the summary! {T}opic-aware convolutional neural networks for extreme summarization.
\newblock \emph{arXiv preprint arXiv:1808.08745}, 2018.

\bibitem[Shazeer(2019)]{shazeer2019fast}
Shazeer, N.
\newblock Fast transformer decoding: One write-head is all you need.
\newblock \emph{arXiv preprint arXiv:1911.02150}, 2019.

\bibitem[Sun et~al.(2023)Sun, Suresh, Ro, Beirami, Jain, and Yu]{sun2023spectr}
Sun, Z., Suresh, A.~T., Ro, J.~H., Beirami, A., Jain, H., and Yu, F.
\newblock Spec{T}r: Fast speculative decoding via optimal transport.
\newblock In \emph{Advances in Neural Information Processing Systems (NeurIPS)}, 2023.

\bibitem[Touvron et~al.(2023)Touvron, Martin, Stone, Albert, Almahairi, Babaei, Bashlykov, Batra, Bhargava, Bhosale, et~al.]{touvron2023llama}
Touvron, H., Martin, L., Stone, K., Albert, P., Almahairi, A., Babaei, Y., Bashlykov, N., Batra, S., Bhargava, P., Bhosale, S., et~al.
\newblock Llama 2: Open foundation and fine-tuned chat models.
\newblock \emph{arXiv preprint arXiv:2307.09288}, 2023.

\bibitem[Vaswani et~al.(2017)Vaswani, Shazeer, Parmar, Uszkoreit, Jones, Gomez, Kaiser, and Polosukhin]{vaswani2017attention}
Vaswani, A., Shazeer, N., Parmar, N., Uszkoreit, J., Jones, L., Gomez, A.~N., Kaiser, {\L}., and Polosukhin, I.
\newblock Attention is all you need.
\newblock \emph{Advances in Neural Information Processing Systems (NeurIPS)}, 30, 2017.

\bibitem[Vieira(2014)]{vieira2014gumbel}
Vieira, T.
\newblock Gumbel-max trick and weighted reservoir sampling.
\newblock 2014.
\newblock URL \url{https://timvieira.github.io/blog/post/2014/08/01/gumbel-max-trick-andweighted-reservoir-sampling/}.

\bibitem[Villani et~al.(2009)]{villani2009optimal}
Villani, C. et~al.
\newblock \emph{Optimal transport: old and new}, volume 338.
\newblock Springer, 2009.

\bibitem[Xia et~al.(2024)Xia, Yang, Dong, Wang, Li, Ge, Liu, Li, and Sui]{xia2024unlocking}
Xia, H., Yang, Z., Dong, Q., Wang, P., Li, Y., Ge, T., Liu, T., Li, W., and Sui, Z.
\newblock Unlocking efficiency in large language model inference: A comprehensive survey of speculative decoding.
\newblock \emph{arXiv preprint arXiv:2401.07851}, 2024.

\bibitem[Xiao et~al.(2023)Xiao, Wu, Guo, Li, Zhang, Qin, and Liu]{xiao2023survey}
Xiao, Y., Wu, L., Guo, J., Li, J., Zhang, M., Qin, T., and Liu, T.-y.
\newblock A survey on non-autoregressive generation for neural machine translation and beyond.
\newblock \emph{IEEE Transactions on Pattern Analysis and Machine Intelligence}, 2023.

\bibitem[Zhang et~al.(2022)Zhang, Roller, Goyal, Artetxe, Chen, Chen, Dewan, Diab, Li, Lin, et~al.]{zhang2022opt}
Zhang, S., Roller, S., Goyal, N., Artetxe, M., Chen, M., Chen, S., Dewan, C., Diab, M., Li, X., Lin, X.~V., et~al.
\newblock {OPT}: Open pre-trained transformer language models.
\newblock \emph{arXiv preprint arXiv:2205.01068}, 2022.

\bibitem[Zhou et~al.(2023)Zhou, Lyu, Rawat, Menon, Rostamizadeh, Kumar, Kagy, and Agarwal]{zhou2023distillspec}
Zhou, Y., Lyu, K., Rawat, A.~S., Menon, A.~K., Rostamizadeh, A., Kumar, S., Kagy, J.-F., and Agarwal, R.
\newblock Distillspec: {I}mproving speculative decoding via knowledge distillation.
\newblock \emph{arXiv preprint arXiv:2310.08461}, 2023.

\end{thebibliography}
\bibliographystyle{icml2024}

\clearpage
\appendix

\onecolumn
\section{Theorems and proofs}
\subsection{Proof of \textbf{Theorem~\ref{thm:1}}}
\label{appendix:A:1}
\begin{manualtheorem}{3.1}[Recursive rejection sampling recovers target distribution] The random variable $Z\in\mathcal{X}$ defining recursive rejection sampling rule \eqref{eq:recursive_rejection_sampling} follows the target distribution $q$, i.e., 
\begin{align*}
    \prob\left\{Z=z\right\}=q(z), z\in\mathcal{X}.
\end{align*}
\end{manualtheorem}
\begin{proof}

We remain a sketch of the proof here and the formal proof is given in the next paragraph. We first consider the case where $\hat{X}^{(1)}, ..., \hat{X}^{(K-1)}$ are rejected and see whether we accept $\hat{X}^{(K)}$ or not; we either accept $\hat{X}^{(K)}$ with probability $\Theta^{(K)}$ in \eqref{eq:threshold} or sample a new token $Y\sim q^{(K+1)}(\cdot|\hX^{(1:K-1)})$ when all draft tokens are rejected. Since $q^{(K+1)}$ is the residual distribution from $q^{(K)}$, one can regard it as the simple sampling by \citet{chen2023accelerating} and \citet{leviathan2023fast}, which recovers $q^{(K)}$. The same idea is applied to $\hX^{(K-1)}, ..., \hX^{(1)}$ in the reversed order until we recover $q=q^{(1)}$ at the end. 

Let us desribe the formal proof. From the definition of recursive rejection sampling \eqref{eq:recursive_rejection_sampling}, we have
\begin{align}
    \prob\left\{
        Z=z
    \right\}=
    \sum_{k=1}^K
    \underbrace{
        \prob\left\{
            A^{(1:k-1)}=\reject^{k-1}, \hX^{(k)}=z, A^{(k)}=\accept
        \right\}
    }_{
        =:\Sigma_{1,k}
    }   
    +
    \underbrace{
        \prob\left\{A^{(1:K)}=\reject^{K}, \tilde{X}^{(K+1)}=z\right\}
    }_{
        =:\Sigma_{2,K}
    }.
    \nonumber
\end{align}
It can be shown that the following equality holds for each $k$:
\begin{align}
    \Sigma_{2,k-1}=\Sigma_{1,k}+\Sigma_{2,k}.
\end{align}
Let us first consider $k=K$, then, 
\begin{align*}
    &\Sigma_{1,K}+\Sigma_{2,K}\\
    &=\sum_{x^{(1)}, ..., x^{(K-1)}}
    \prob\left\{\hX^{(1:K-1)}=x^{(1:K-1)}\right\}
    \times
    \prob\left\{
        A^{(1:K-1)}=\reject^{K-1}, \hX^{(K)}=z, A^{(K)}=\accept\bigg|\hX^{(1:K-1)}=x^{(1:K-1)}
    \right\}\\
    &+
    \sum_{x^{(1)}, ..., x^{(K)}}
    \prob\left\{\hX^{(1:K)}=x^{(1:K)}\right\}
    \times
    \prob\left\{
        A^{(1:K)}=\reject^{K}, \hX^{(K+1)}=z\bigg|\hX^{(1:K)}=x^{(1:K)}
    \right\}\\
    &=\sum_{x^{(1)}, ..., x^{(K-1)}}
    \prob\left\{
        \hX^{(1:K-1)}=x^{(1:K-1)}
    \right\}
    \Bigg(
    \underbrace{
        \prob\left\{
            A^{(1:K-1)}=\reject^{K-1}, \hX^{(K)}=z, A^{(K)}=\accept\bigg|\hX^{(1:K-1)}=x^{(1:K-1)}
        \right\}
    }_{=:T_1(K)}
    \\
    &
    ~~~~~~~~~~~~~~~~~~~~~~~~~
    +
    \underbrace{
        \sum_{x^{(K)}}
        \prob\left\{\hX^{(K)}=x^{(K)}\bigg|\hX^{(1:K-1)}=x^{(1:K-1)}\right\}
        \times
        \prob\left\{
            A^{(1:K)}=\reject^{K}, \hX^{(K+1)}=z\bigg|\hX^{(1:K)}=x^{(1:K)}
        \right\}
    }_{=:T_2(K)}
    \Bigg).
\end{align*}
One can represent $T_{1,K}$ and $T_{2,K}$ as follows:
\begin{align*}
    T_{1,K}
    &=
    \prob\left\{
        A^{(1:K-1)}=\reject^{K-1}\bigg|\hX^{(1:K-1)}=x^{(1:K-1)}
    \right\}\\
    &~~~~~\times
    \prob\left\{
        \hX^{(K)}=z\bigg|\hX^{(1:K-1)}=x^{(1:K-1)}
    \right\}
    \times
    \prob\left\{
        A^{(K)}=\accept\Bigg|\hX^{(1:K)}=(x^{(1:K-1)}, z)
    \right\}\\
    &=    
    \prob\left\{
        A^{(1:K-1)}=\reject^{K-1}\bigg|\hX^{(1:K-1)}=x^{(1:K-1)}
    \right\}
    p^{(K)}(z|x^{(1:K-1)})
    \min\left\{
        1, \frac{q^{(K)}(z|x^{(1:K-2)})}{p^{(K)}(z|x^{(1:K-1)})}
    \right\}\\
    &=
    \prob\left\{
        A^{(1:K-1)}=\reject^{K-1}\bigg|\hX^{(1:K-1)}=x^{(1:K-1)}
    \right\}
    \min\left\{
        p^{(K)}(z|x^{(1:K-1)}), q^{(K)}(z|x^{(1:K-2)})
    \right\},
\end{align*}
\begin{align*}
    T_{2,K}
    &=
    \sum_{x^{(K)}}
    \prob\left\{\hX^{(K)}=x^{(K)}\bigg|\hX^{(1:K-1)}=x^{(1:K-1)}\right\}
    \times
    \prob\left\{
        A^{(1:K)}=\reject^{K}, \hX^{(K+1)}=z\bigg|\hX^{(1:K)}=x^{(1:K)}
    \right\}\\
    &=
    \sum_{x^{(K)}}
    p^{(K)}(x^{(K)}|x^{(1:K-1)})
    \times
    \prob\left\{
        A^{(1:K)}=\reject^{K}\bigg|\hX^{(1:K)}=x^{(1:K)}
    \right\}
    \times
    \prob\left\{
        \hX^{(K+1)}=z\bigg|\hX^{(1:K)}=x^{(1:K)}
    \right\}\\
    &=
    \sum_{x^{(K)}}
    p^{(K)}(x^{(K)}|x^{(1:K-1)})
    \times
    \prob\left\{
        A^{(1:K)}=\reject^{K}\bigg|\hX^{(1:K)}=x^{(1:K)}
    \right\}
    \times
    q^{(K+1)}(z|x^{(1:K-1)})\\
    &=
    \sum_{x^{(K)}}
    p^{(K)}(x^{(K)}|x^{(1:K-1)})
    \times
    \prob\left\{
        A^{(1:K-1)}=\reject^{K-1}\bigg|\hX^{(1:K-1)}=x^{(1:K-1)}
    \right\}\\
    &~~~~~~~~~~\times
    \prob\left\{
        A^{(K)}=\reject\bigg|\hX^{(1:K)}=x^{(1:K)}
    \right\}
    \times
    q^{(K+1)}(z|x^{(1:K-1)})\\
    &=
    \prob\left\{
        A^{(1:K-1)}=\reject^{K-1}\bigg|\hX^{(1:K-1)}=x^{(1:K-1)}
    \right\}
    q^{(K+1)}(z|x^{(1:K-1)})\\
    &~~~~~~~~~~\times
    \sum_{x^{(K)}}
    p^{(K)}(x^{(K)}|x^{(1:K-1)})
    \times
    \prob\left\{
        A^{(K)}=\reject\bigg|\hX^{(1:K)}=x^{(1:K)}
    \right\}
    \\
    &=
    \prob\left\{
        A^{(1:K-1)}=\reject^{K-1}\bigg|\hX^{(1:K-1)}=x^{(1:K-1)}
    \right\}
    q^{(K+1)}(z|x^{(1:K-1)})\\
    &~~~~~~~~~~\times
    \sum_{x^{(K)}}
    p^{(K)}(x^{(K)}|x^{(1:K-1)})
    \times
    \left(
        1
        -
        \prob\left\{
            A^{(K)}=\accept\bigg|\hX^{(1:K)}=x^{(1:K)}
        \right\}
    \right)
    \\
    &=
    \prob\left\{
        A^{(1:K-1)}=\reject^{K-1}\bigg|\hX^{(1:K-1)}=x^{(1:K-1)}
    \right\}
    q^{(K+1)}(z|x^{(1:K-1)})\\
    &~~~~~~~~~~\times
    \sum_{x^{(K)}}
    p^{(K)}(x^{(K)}|x^{(1:K-1)})
    \times
    \left(
        1
        -
        \min\left\{
            1, \frac{q^{(K)}(x^{(K)}|x^{(1:K-2)})}{p^{(K)}(x^{(K)}|x^{(1:K-1)})}
        \right\}
    \right)
    \\
    &=
    \prob\left\{
        A^{(1:K-1)}=\reject^{K-1}\bigg|\hX^{(1:K-1)}=x^{(1:K-1)}
    \right\}
    \frac{
        \max\left\{0, q^{(K)}(z|x^{(1:K-2)})-p^{(K)}(z|x^{(1:K-1)})\right\}
    }
    {
        \sum_{x^{(K)}}\max\left\{0, q^{(K)}(x^{(k)}|x^{(1:K-2)})-p^{(K)}(x^{(K)}|x^{(1:K-1)})\right\}
    }
    \\
    &~~~~~~~~~~\times
    \sum_{x^{(K)}}
    \max\left\{
        0, 
        q^{(K)}(x^{(K)}|x^{(1:K-2)})-p^{(K)}(x^{(K)}|x^{(1:K-1)})
    \right\}
    \\
    &=
    \prob\left\{
        A^{(1:K-1)}=\reject^{K-1}\bigg|\hX^{(1:K-1)}=x^{(1:K-1)}
    \right\}
    \max\left\{0, q^{(K)}(z|x^{(1:K-2)})-p^{(K)}(z|x^{(1:K-1)})\right\}.
\end{align*}
Therefore, we have
\begin{align*}
    &T_{1,K}+T_{2,K}\\
    &=
    \prob\left\{
        A^{(1:K-1)}=\reject^{K-1}\bigg|\hX^{(1:K-1)}=x^{(1:K-1)}
    \right\}\\
    &~~~~~~~~~~\times
    \left(
        \min\left\{
        p^{(K)}(z|x^{(1:K-1)}), q^{(K)}(z|x^{(1:K-2)})
        \right\}
        +
        \max\left\{0, q^{(K)}(z|x^{(1:K-2)})-p^{(K)}(z|x^{(1:K-1)})\right\}
    \right)
    \\
    &=
    \prob\left\{
        A^{(1:K-1)}=\reject^{K-1}\bigg|\hX^{(1:K-1)}=x^{(1:K-1)}
    \right\}q^{(K)}(z|x^{(1:K-2)})
    \\
    &=
    \prob\left\{
        A^{(1:K-1)}=\reject^{K-1}, \tilde{X}^{(K)}=z\bigg|\hX^{(1:K-1)}=x^{(1:K-1)}
    \right\},
\end{align*}
where we define a random variable $\tilde{X}^{(K)}$ such that 
\begin{align*}
    \prob
    \left\{
        \tilde{X}^{(K)}=z\bigg|\hX^{(1:K-1)}=x^{(1:K-1)}
    \right\}
    :=q^{(K)}(z|x^{(1:K-1)}),
\end{align*}
which leads to
\begin{align*}
    &\Sigma_{1,K}+\Sigma_{2,K}\\
    &=\sum_{x^{(1)}, ..., x^{(K-1)}}
    \prob\left\{
        \hX^{(1:K-1)}=x^{(1:K-1)}
    \right\}
    (T_{1,K}+T_{2,K})\\
    &=\sum_{x^{(1)}, ..., x^{(K-1)}}
    \prob\left\{
        \hX^{(1:K-1)}=x^{(1:K-1)}
    \right\}
    \times
    \prob\left\{
        A^{(1:K-1)}=\reject^{K-1}, \tilde{X}^{(K)}=z\bigg|\hX^{(1:K-1)}=x^{(1:K-1)}
    \right\}\\
    &=
    \prob\left\{
        A^{(1:K-1)}=\reject^{K-1}, \tilde{X}^{(K)}=z
    \right\}\\
    &=\Sigma_{2,K-1}.
\end{align*}
Since the same derivation can be done for $k=2, ..., K-1$, we have
\begin{align*}
    \prob\left\{Z=z\right\}
    =
    \sum_{k=1}^K\Sigma_{1,k}+\Sigma_{2,K}
    =
    \sum_{k=1}^{K-1}\Sigma_{1,k}+\Sigma_{2,K-1}
    =
    \cdots
    =
    \Sigma_{1,1}+\Sigma_{2,1}=q(z),
\end{align*}
where the last equality holds from the derivation of original speculative decoding by~\cite{chen2023accelerating,leviathan2023fast}.
\end{proof}

\subsection{Proof of \textbf{Theorem~\ref{thm:2}}}
\label{appendix:A:2}
\begin{manualtheorem}{3.2}[Tokens from the same sequence follow sampling without replacement in RSD-S]
\label{thm:2}
In RSD-S, any non-empty subsequence of the sequence $\hat{X}_{l+1}^{(1)}, ..., \hat{X}_{l+1}^{(W)}$ of draft tokens (from 
$O_{l+1}^{(1)}, ..., O_{l+1}^{(W)}$ in \eqref{eq:argtop_pair}) such that each element of the subsequence has the same parent $\btau_l^{(k)}$ follows sampling without replacement from $p(\cdot|\btau_l^{(k)})$. 
\end{manualtheorem}
\begin{proof}
    For fixed $\btau_l^{(k)}$, consider a sequence of tokens
\begin{align*}
    \bar{\mathbf{X}}_{l+1}^{(k)}:=\underset{x\in\mathcal{X}}{\mathrm{argsort}}~
    \psi_l(\btau_l^{(k)},x)
    =\underset{x\in\mathcal{X}}{\mathrm{argsort}}~
    \tilde{\phi}_l(\btau_l^{(k)},x),
\end{align*}
where the last equality holds since $T$ in \eqref{eq:transformation} is monotonically increasing w.r.t. $\tilde{\phi}_l(\btau_l^{(k)},\cdot)$ for fixed $\btau_l^{(k)}$. Thus, $\bar{\mathbf{X}}_{l+1}^{(k)}$ can be seen as samples from $p(\cdot|\btau_l^{(k)})$ without replacement.

For a length-$l_k$ subsequence $\mathbf{o}_{l+1}^{(k)}$ of $(O_{l+1}^{(1)}, ..., O_{l+1}^{(W)})$ in \eqref{eq:argtop_pair}, where each element of the subsequence have $\btau_l^{(k)}$ as its parent, the token sequence in 
$\mathbf{o}_{l+1}^{(k)}$ is a subsequence of $\bar{\mathbf{X}}_{l+1}^{(k)}$, i.e., those tokens are top-$l_k$ samples without replacement from $p(\cdot|\btau_l^{(k)})$.
\end{proof}

\clearpage
\section{Algorithm}
\label{appendix:sec:algorithm}
\subsection{Recursive Speculative Decoding with Constant Branching Factors (RSD-C)}
\label{appendix:sec:rsd_c}

\begin{algorithm}[h]
   \caption{Recursive Speculative Decoding with Constant Branching Factors (RSD-C)}
   \label{alg:rsd_c}
\begin{algorithmic}[1]
   \STATE {\bfseries Input:}
   The length $\draftlen$ of draft sequences (depth of the draft tree),
   a sequence $\inputseq$ of input tokens, 
   a list $\branchingfactor:=[b_0, ..., b_{\draftlen-1}]$ of constant branching factors in the draft tree,
   the maximum length $\outputlen$ of the output sequence.
   \STATE
   \Comment{// Get the length of the input sequence.} \\
   $\inputlen\subs\getlength(\inputseq)$.
   \STATE
   \Comment{// Initialize empty KV caches for draft and target models.} \\
   $\kvcachedraft\leftarrow\emptyset$, $\kvcachetarget\leftarrow\emptyset$.
   \WHILE {$\inputlen < \outputlen$}
   \STATE
   \Comment{// (STEP 1) Create a draft tree by using the draft model.} \\
   $\drafttree, \inputseq, \kvcachedraft, \attentionmask, \positionids, \listnumnodes \subs \createdrafttreeconst(\inputseq, \kvcachedraft, \branchingfactor, \draftlen).$ 
   
   \STATE
   \Comment{// (STEP 2) Evaluate draft tokens by using the target model.} \\
   \Comment{// - Apply $\attentionmask$ to the right below corner of attention weights.} \\
   \Comment{// - The target log probability $\logprobstarget$ is a $\mathtt{GetLength}(\inputseq)\times \vocabsize$ tensor.} \\
   \Comment{// - $\vocabsize$ is the vocabulary size.} \\
   $\logprobstarget, \kvcachetarget \subs \targetmodelforwardpass(\inputseq, \kvcachetarget, \positionids, \attentionmask)$.
   \STATE
   \Comment{// - Convert the log probability tensor into the list of log probabilities for each level of the tree.} \\
   $\listlogprobstarget \subs \splittensor(\logprobstarget[-\mathtt{Sum}(\listnumnodes):, :], \listnumnodes, \dimension=0)$
   \STATE 
   \Comment{// (STEP 3) Run Recursive Rejection Sampling for each level of the tree.} \\
   $\acceptedtokens, \lasttoken, \flatnodeidsaccepted \subs \recursiverejectionsampling(\drafttree, \listlogprobstarget)$

   \STATE
   \Comment{// (STEP 4) Use KV caches that are accepted, and prepare for the next round.} \\
   $\kvcachedraft, \kvcachetarget \subs \filterkvcache(\kvcachedraft, \kvcachetarget, \inputlen, \flatnodeidsaccepted)$

   \STATE
   $\inputseq \subs \concatenate([\inputseq[:\inputlen], \acceptedtokens, \lasttoken])$
      
   \STATE
   $\inputlen \subs \getlength(\inputseq)$
 
   \ENDWHILE
   \STATE {\bfseries Output:} a sequence $\inputseq$ that includes both input tokens and generated output tokens.
\end{algorithmic}
\end{algorithm}

\begin{algorithm}[h]
   \caption{$\createdrafttreeconst(\inputseq, \kvcachedraft, \branchingfactor, \draftlen)$}
   \label{alg:createdrafttreeconst}
\begin{algorithmic}[1]

   \STATE {\bfseries Input:} 
   An input sequence $\inputseq$, 
   the draft KV cache $\kvcachedraft$, 
   the branching factor $\branchingfactor:=[b_0, ..., b_{\draftlen-1}]$, 
   the draft length $\draftlen$
   \STATE
   \Comment{// Get the length of the input sequence.} \\
   $\inputlen\subs\getlength(\inputseq)$.
   \STATE
   \Comment{// Initialize lists for 1) draft log probabilities, 2) flattened node IDs, 3) parent node ids (within each level of the draft tree), 4) draft tokens, 5) numbers of nodes (for all levels of the tree), respectively.} \\
   $\listlogprobsdraft \subs [~], \listflatnodeids \subs [~], \listparentids \subs [~], \listdrafttokens \subs [~], \listnumnodes \subs [~]$.
   \STATE 
   \Comment{// Initialize a draft tree.} \\
   $
   \drafttree \subs (\listlogprobsdraft, \listflatnodeids, \listparentids, \listdrafttokens)
   $.
   \STATE 
   \Comment{// Set an empty attention mask, and position ids; inclusive for $\astart$ and exclusive for $\aend$.} \\
   $\attentionmask\subs\emptyset$, $\positionids\subs\arange(\astart=0, \aend=\inputlen)$.
   \STATE 
   \Comment{// Set the counter to check the number of nodes in the tree.} \\
   $\numnodeprev\subs0$, $\numnodecurr\subs0$. 
   \STATE 
   \Comment{// Set the number of nodes at the current level of the tree.} \\
   $\numnodes\subs1$, $\listnumnodes\append(\numnodes)$. 
   \FOR{$\currdraftlen=0$ {\bfseries to} $\draftlen - 1$}
   \STATE
   \Comment{// Apply $\attentionmask$ to the right below corner of attention weights.} \\
   \Comment{// The draft log probability $\logprobsdraft$ is a $\mathtt{GetLength}(\inputseq)\times \vocabsize$ tensor.} \\
   \Comment{// $\vocabsize$ is the vocabulary size.} \\
   $\logprobsdraft, \kvcachedraft \subs \draftmodelforwardpass(\inputseq, \kvcachedraft, \positionids, \attentionmask)$.
   
   \STATE
   \Comment{// Sample $b_{\currdraftlen}$ nodes without replacement, independently for $\numnodes$ nodes.}\\
   \Comment{// NOTE: Outputs are sorted w.r.t. the value of perturbed log probabilities and flattened.}\\
   $\drafttokens, \parentids\subs\samplewithgumbeltopk(\logprobsdraft[-\numnodes:, :], b_{\currdraftlen})$.
   \STATE
   \Comment{// Update the input sequence of tokens.}\\
   $\inputseq \subs \concatenate ([\inputseq, \drafttokens])$.
   \STATE
   \Comment{// Get the number of newly added nodes.}\\
   $\numnodes \subs \getlength(\drafttokens)$.
   \STATE 
   \Comment{// Build attention mask reflecting tree topology.}\\
   $\attentionmask \subs \buildattentionmask(\attentionmask, \parentids, \numnodes, \numnodeprev, \numnodecurr)$.
   \STATE
   \Comment{// Update counters.}\\
   $\numnodeprev \subs \numnodecurr$, 
   $\numnodecurr \subs \numnodecurr + \numnodes$.
   
   \STATE
   \Comment{// Update position IDs. }\\
   $\positionids \subs \concatenate([\positionids, (\inputlen + \currdraftlen)\times\mathbf{1}_{\numnodes}])$.
   \STATE
   \Comment{// Get node IDs considering the flattened draft tree.}\\
   \Comment{// This is used to update KV caches.}\\
   $\flatnodeids \subs \arange(\astart=\inputlen+\numnodeprev,\aend=\inputlen+\numnodecurr)$.
   \STATE
   \Comment{// Update the lists of the tree.}\\
   $\listlogprobsdraft\append(\logprobsdraft)$,
   $\listflatnodeids\append(\flatnodeids)$,
   $\listparentids\append(\parentids)$,\\
   $\listdrafttokens\append(\drafttokens)$,
   $\listnumnodes\append(\numnodes)$.
   \ENDFOR
   \STATE {\bfseries Output:} 
   $\drafttree, \inputseq, \kvcachedraft, \attentionmask, \positionids, \listnumnodes$.
    
\end{algorithmic}
\end{algorithm}

\begin{algorithm}[h]
   \caption{$\samplewithgumbeltopk(\logprobs, K)$}
   \label{alg:samplewithgumbeltopk}
\begin{algorithmic}[1]
    \STATE {\bfseries Input:} a $\numnodes \times \vocabsize$ log probabilities $\logprobs$, the number $K$ of desired samples without replacement.
    \STATE
    \Comment{// Sample a matrix where elements are i.i.d. standard Gumbel random variables.}\\
    $\gumbel \subs [g_{ij}]$, $g_{ij}\subs\mathtt{SampleStandardGumbel}(), i=0, ..., \numnodes-1, j=0, ..., \vocabsize-1$.
    \STATE
    \Comment{// Perturb log probabilities with Gumbel random variables.}\\
    $\perturbedlogprobs \subs \logprobs + \gumbel$.
    \STATE
    \Comment{// Get top-$K$ elements corresponding to the $K$ largest perturb log probabilities.}\\
    \Comment{// Outputs are sorted (in each row) w.r.t. the values of perturbed log probabilities and flattened.}\\
    $\tokens \subs \mathrm{argtop}^{(K)}(\perturbedlogprobs, \dimension=-1)\flatten()$.
    \STATE
    \Comment{// Set parent ids.}\\
    $\parentids \subs \concatenate([0\cdot\mathbf{1}_{K}, 1\cdot\mathbf{1}_{K}, ..., (\numnodes - 1)\cdot\mathbf{1}_{K}])$.
    \STATE
    \Comment{// When probability filtering methods(e.g., top-$p$, top-$k$) were applied, filter some elements of $\tokens$ and $\parentids$ if corresponding log probability is equal to $-\infty$.}
    \STATE {\bfseries Output}: $\tokens, \parentids$.
\end{algorithmic}
\end{algorithm}

\begin{algorithm}[h]
   \caption{$\buildattentionmask(\attentionmask, \parentids, \numnodes, \numnodeprev, \numnodecurr)$}
   \label{alg:buildattentionmask}
\begin{algorithmic}[1]
    \STATE {\bfseries Input:} previous attention mask $\attentionmask$, parent node ids $\parentids$ for newly added nodes, the number $\numnodes$ of nodes newly added to the tree, the total number $\numnodeprev$ of nodes in the previous-iteration tree, the total number $\numnodecurr$ of nodes in the current-iteration tree. 
    \IF{$\attentionmask == \emptyset$}
        \STATE 
        \Comment{// If the attention mask is empty, we initialize with zeros.} \\
        $\attentionmask \subs \mathbf{0}_{\numnodes\times\numnodes}$.
    \ELSE
        \STATE
        \Comment{// If the attention mask exists, we zero-pad.} \\
        $\attentionmask \subs \zeropadding(\attentionmask, \mathtt{right}=\numnodes, \mathtt{bottom}=\numnodes)$.
        \FOR{$i=0$ {\bfseries to} $\numnodes-1$}
            \STATE 
            \Comment{// Copy the row about paraent nodes to the row about child nodes.} \\
            $\attentionmask[\numnodecurr + i, :] \subs \attentionmask[\numnodeprev + \parentids[i], :]$.
        \ENDFOR
    \ENDIF
    \STATE 
    \Comment{// Set diagonal elements equal to 1.} \\
    $\attentionmask \subs \attentionmask\filldiagonal(1)$
    
    \STATE {\bfseries Output:} the new attention mask $\attentionmask$. 
\end{algorithmic}
\end{algorithm}

\begin{algorithm}[h]
   \caption{$\recursiverejectionsampling(\drafttree, \listlogprobstarget)$}
   \label{alg:recursiverejectionsampling}
\begin{algorithmic}[1]
    \STATE {\bfseries Input:} the draft tree $\drafttree$, the list $\listlogprobstarget$ of target log probabilities 
    \STATE
    \Comment{// Get lists from the draft tree.} \\
    $\listlogprobsdraft, \listflatnodeids, \listparentids, \listdrafttokens \subs \drafttree$
    \STATE
    \Comment{// Set the current node id.} \\
    $\nodeid \subs 0$
    \STATE 
    \Comment{// Initialize the lists to store accepted draft tokens and flattened node ids (for KV cache update).} \\
    $\listdrafttokensaccepted \subs [~], \listflatnodeidsaccepted \subs [~]$.
    \FOR{$\currdraftlen=0$ {\bfseries to} $\draftlen-1$}
        \STATE 
        \Comment{// Get log probabilities at the current node.} \\
        \Comment{// Both are $1\times \vocabsize$ tensors, where $\vocabsize$ is the vocabulary size.} \\
        $\logprobsdraft \subs \listlogprobsdraft[\currdraftlen][\nodeid:(\nodeid+1), :]$, $\logprobstarget \subs \listlogprobstarget[\currdraftlen][\nodeid:(\nodeid+1), :]$
        \STATE 
        \Comment{// Get draft tokens, flattened node IDs, parent IDs at the current level.} \\
        $\drafttokens \subs \listdrafttokens[\currdraftlen]$, $\flatnodeids \subs \listflatnodeids[\currdraftlen]$, $\parentids \subs \listparentids[\currdraftlen]$
        \STATE 
        \Comment{// Initialize an acceptance indicator as False.} \\
        $\tokenisaccepted \subs \isfalse$

        \FOR{$i$ {\bfseries in} $\parentids$}
            \IF{$i\ne \nodeid$}
                \CONTINUE
            \ENDIF
            \STATE 
            \Comment{// Get the current draft token.} \\
            $x_d \subs \drafttokens[i]$.
            \STATE 
            \Comment{// Sample a uniform random variable.} \\
            $U \sim \mathrm{Uniform}[0, 1]$.
            \IF{$U < \min\{1, \exp(\logprobstarget[0, x_d] - \logprobsdraft[0, x_d])\}$}
                \STATE 
                \Comment{// Set the indicator as True is the token is accepted.} \\
                $\tokenisaccepted \subs \istrue$.
                \STATE 
                \Comment{// Store the accepted token and corresponding flattened node ID.} \\
                $\listdrafttokensaccepted\append(x_d)$, $\listdrafttokensaccepted\append(\flatnodeids[i])$.
                \STATE
                $\nodeid \subs i$.
                \BREAK
            \ENDIF

            \STATE 
            \Comment{// Get clamped target log probability. } \\
            $\logprobstarget \subs \log((\exp(\logprobstarget) - \exp(\logprobsdraft))\clamp(\mathtt{min}=0)\})$
            \STATE
            \Comment{// Normalize the clamped target log probability.} \\
            $\logprobstarget \subs \logprobstarget - \logsumexp(\logprobstarget)$
            \STATE
            \Comment{// Neglect draft log probability of already sampled token.} \\ 
            $\logprobsdraft[i] \subs - \infty$
            \STATE
            \Comment{// Normalize the draft log probability.} \\
            $\logprobsdraft \subs \logprobsdraft - \logsumexp(\logprobsdraft)$
            
        \ENDFOR

        \IF{$\tokenisaccepted==\isfalse$}
            \BREAK
        \ENDIF
    \ENDFOR
    \IF{$\tokenisaccepted$}
        \STATE
        \Comment{// At the leaf node when all tokens are accepted, we use target log probability to draw a sample.} \\
        $\logprobstarget \subs \listlogprobstarget[l_d][\nodeid:(\nodeid+1), :]$
    \ENDIF
    \STATE
        $\lasttoken\sim\samplewithgumbeltopk(\logprobstarget, 1)$
    \STATE
    $\acceptedtokens \subs \stack(\listdrafttokensaccepted)$,
    $\flatnodeidsaccepted \subs \stack(\listdrafttokensaccepted)$.
    \STATE {\bfseries Output: $\acceptedtokens, \lasttoken, \flatnodeidsaccepted$} 
\end{algorithmic}
\end{algorithm}

\clearpage 
\subsection{Recursive Speculative Decoding with Stochastic Beam Search (RSD-S)}
\label{appendix:sec:rsd_s}

We highlight the difference w.r.t. RSD-C.

\begin{algorithm}[h]
   \caption{Recursive Speculative Decoding with Stochastic Beam Search (RSD-S)}
   \label{alg:recursive_speculative_decoding_with_stochastic_beam_search}
\begin{algorithmic}[1]
   \STATE {\bfseries Input:}
   The length $\draftlen$ of draft sequences (depth of the draft tree),
   a sequence $\inputseq$ of input tokens, 
   \Red{the beamwidth $W$,}
   the maximum length $\outputlen$ of the output sequence.
   \STATE
   \Comment{// Get the length of the input sequence.} \\
   $\inputlen\subs\getlength(\inputseq)$.
   \STATE
   \Comment{// Initialize empty KV caches for draft and target models.} \\
   $\kvcachedraft\leftarrow\emptyset$, $\kvcachetarget\leftarrow\emptyset$.
   \WHILE {$\inputlen < \outputlen$}
   \STATE
   \Comment{// (STEP 1) Create a draft tree by using the draft model.} \\
   $\drafttree, \inputseq, \kvcachedraft, \attentionmask, \positionids, \listnumnodes$ \\
   $\subs {\Red{\createdrafttreesbs}}(\inputseq, \kvcachedraft, \Red{W}, \draftlen).$ 
   
   \STATE
   \Comment{// (STEP 2) Evaluate draft tokens by using the target model.} \\
   \Comment{// - Apply $\attentionmask$ to the right below corner of attention weights.} \\
   \Comment{// - The target log probability $\logprobstarget$ is a $\mathtt{GetLength}(\inputseq)\times \vocabsize$ tensor.} \\
   \Comment{// - $\vocabsize$ is the vocabulary size.} \\
   $\logprobstarget, \kvcachetarget \subs \targetmodelforwardpass(\inputseq, \kvcachetarget, \positionids, \attentionmask)$.
   \STATE
   \Comment{// - Convert the log probability tensor into the list of log probabilities for each level of the tree.} \\
   $\listlogprobstarget \subs \splittensor(\logprobstarget[-\mathtt{Sum}(\listnumnodes):, :], \listnumnodes, \dimension=0)$
   \STATE 
   \Comment{// (STEP 3) Run Recursive Rejection Sampling for each level of the tree.} \\
   $\acceptedtokens, \lasttoken, \flatnodeidsaccepted \subs \recursiverejectionsampling(\drafttree, \listlogprobstarget)$

   \STATE
   \Comment{// (STEP 4) Use KV caches that are accepted, and prepare for the next round.} \\
   $\kvcachedraft, \kvcachetarget \subs \filterkvcache(\kvcachedraft, \kvcachetarget, \inputlen, \flatnodeidsaccepted)$

   \STATE
   $\inputseq \subs \concatenate([\inputseq[:\inputlen], \acceptedtokens, \lasttoken])$
      
   \STATE
   $\inputlen \subs \getlength(\inputseq)$
 
   \ENDWHILE
   \STATE {\bfseries Output:} a sequence $\inputseq$ that includes both input tokens and generated output tokens.
\end{algorithmic}
\end{algorithm}

\begin{algorithm}[h]
   \caption{$\Red{\createdrafttreesbs}(\inputseq, \kvcachedraft, \Red{W}, \draftlen)$}
   \label{alg:create_draft_tree_stochastic_beam_search}
\begin{algorithmic}[1]

   \STATE {\bfseries Input:} 
   An input sequence $\inputseq$, 
   the draft KV cache $\kvcachedraft$, 
   \Red{the beamwidth $W$,}
   the draft length $\draftlen$
   \STATE
   \Comment{// Get the length of the input sequence.} \\
   $\inputlen\subs\getlength(\inputseq)$.
   \STATE
   \Comment{// Initialize lists for 1) draft log probabilities, 2) flattened node IDs, 3) parent node ids (within each level of the draft tree), 4) draft tokens, 5) numbers of nodes (for all levels of the tree), respectively.} \\
   $\listlogprobsdraft \subs [~], \listflatnodeids \subs [~], \listparentids \subs [~], \listdrafttokens \subs [~], \listnumnodes \subs [~]$.
   \STATE 
   \Comment{// Initialize a draft tree.} \\
   $
   \drafttree \subs (\listlogprobsdraft, \listflatnodeids, \listparentids, \listdrafttokens)
   $.
   \STATE 
   \Comment{// Set an empty attention mask, and position ids; inclusive for $\astart$ and exclusive for $\aend$.} \\
   $\attentionmask\subs\emptyset$, $\positionids\subs\arange(\astart=0, \aend=\inputlen)$.
   \STATE 
   \Comment{// Set the counter to check the number of nodes in the tree.} \\
   $\numnodeprev\subs0$, $\numnodecurr\subs0$. 
   \STATE 
   \Comment{// Set the number of nodes at the current level of the tree.} \\
   $\numnodes\subs1$, $\listnumnodes\append(\numnodes)$. 
   
   \STATE 
   \Comment{\Red{// Set stochastic beam parameters: sum log probabilities $\sumlogprob$ and truncated Gumbels $\truncatedgumbel$ for each node in the current level of draft tree}} \\
   \Red{$
   \sumlogprob \subs \textbf{0}_{\numnodes\times1}, 
   \truncatedgumbel \subs \textbf{0}_{\numnodes\times1}
   $.}
   
   \FOR{$\currdraftlen=0$ {\bfseries to} $\draftlen - 1$}
   \STATE
   \Comment{// Apply $\attentionmask$ to the right below corner of attention weights.} \\
   \Comment{// The draft log probability $\logprobsdraft$ is a $\mathtt{GetLength}(\inputseq)\times \vocabsize$ tensor.} \\
   \Comment{// $\vocabsize$ is the vocabulary size.} \\
   $\logprobsdraft, \kvcachedraft \subs \draftmodelforwardpass(\inputseq, \kvcachedraft, \positionids, \attentionmask)$.
   
   \STATE
   \Comment{\Red{// Sample $b_{\currdraftlen}$ nodes without replacement, independently for $\numnodes$ nodes.}}\\
   \Comment{\Red{// NOTE: Outputs are sorted w.r.t. the value of perturbed log probabilities and flattened.}}\\
   \Red{$\drafttokens, \parentids, \sumlogprob, \truncatedgumbel \subs \samplewithstochasticbeam(\logprobsdraft[-\numnodes:, :], \sumlogprob, \truncatedgumbel, W)$.}
   \STATE
   \Comment{// Update the input sequence of tokens.}\\
   $\inputseq \subs \concatenate ([\inputseq, \drafttokens])$.
   \STATE
   \Comment{// Get the number of newly added nodes.}\\
   $\numnodes \subs \getlength(\drafttokens)$.
   \STATE 
   \Comment{// Build attention mask reflecting tree topology.}\\
   $\attentionmask \subs \buildattentionmask(\attentionmask, \parentids, \numnodes, \numnodeprev, \numnodecurr)$.
   \STATE
   \Comment{// Update counters.}\\
   $\numnodeprev \subs \numnodecurr$, 
   $\numnodecurr \subs \numnodecurr + \numnodes$.

   \STATE
   \Comment{// Update position IDs. }\\
   $\positionids \subs \concatenate([\positionids, (\inputlen + \currdraftlen)\times\mathbf{1}_{\numnodes}])$.
   \STATE
   \Comment{// Get node IDs considering the flattened draft tree.}\\
   \Comment{// This is used to update KV caches.}\\
   $\flatnodeids \subs \arange(\astart=\inputlen+\numnodeprev,\aend=\inputlen+\numnodecurr)$.
   \STATE
   \Comment{// Update the lists of the tree.}\\
   $\listlogprobsdraft\append(\logprobsdraft)$,
   $\listflatnodeids\append(\flatnodeids)$,
   $\listparentids\append(\parentids)$,\\
   $\listdrafttokens\append(\drafttokens)$,
   $\listnumnodes\append(\numnodes)$.
   \ENDFOR
   \STATE {\bfseries Output:} 
   $\drafttree, \inputseq, \kvcachedraft, \attentionmask, \positionids, \listnumnodes$.
    
\end{algorithmic}
\end{algorithm}

\begin{algorithm}[h]
   \caption{$\Red{\samplewithstochasticbeam(\logprobs, \sumlogprob, \truncatedgumbel, K)}$}
   \label{alg:sample_with_stochastic_beam}
\begin{algorithmic}[1]
    \STATE {\bfseries Input:} a $\numnodes \times \vocabsize$ log probabilities $\logprobs$, \Red{a $\numnodes \times 1$ sum log probabilities $\sumlogprob$, a $\numnodes \times 1$ truncated Gumbels $\truncatedgumbel$, the beamwidth $K$.}
    \STATE
    \Comment{\Red{// Get sum log probs up to child nodes.}}\\
    \Red{$\logprobs \subs \logprobs + \sumlogprob \mathbf{1}_{1\times \vocabsize}$.}
    \STATE
    \Comment{// Sample a matrix where elements are i.i.d. standard Gumbel random variables.}\\
    $\gumbel \subs [g_{ij}]$, $g_{ij}\subs\mathtt{SampleStandardGumbel}(), i=0, ..., \numnodes-1, j=0, ..., \vocabsize-1$.
    
    \STATE
    \Comment{// Perturb \Red{sum} log probabilities with Gumbel random variables.}\\
    $\perturbedlogprobs \subs \logprobs + \gumbel$.
    
    \STATE
    \Comment{\Red{// Compute row-wise maximum value of perturbed sum log probabilities.}}\\
    \Comment{\Red{// The output size is $\numnodes\times1$.}}\\
    \Red{$\maxperturbedlogprobs \subs \perturbedlogprobs.\mathtt{max}(\dimension=-1, \mathtt{keepdim}=\istrue)$.}
    
    \STATE
    \Comment{\Red{// Get truncated Gumbels for all expansion.}}\\
    \Comment{\Red{// The output size is $\numnodes\times \vocabsize$.}}\\
    \Comment{\Red{// NOTE: the numerical stable way of computing this quantity was described in the original Stochastic Beam Search paper.}}\\
    \Red{$\tilde{\truncatedgumbel} \subs - \log ( \exp ( - \truncatedgumbel \mathbf{1}_{1\times \vocabsize} ) - \exp ( - \maxperturbedlogprobs \mathbf{1}_{1\times \vocabsize}) + \exp ( - \perturbedlogprobs))   $}
    
    \STATE
    \Comment{// Get top-$K$ elements \Red{and} the $K$ largest \Red{truncated Gumbels.}}\\
    \Comment{\Red{// NOTE: we consider top-$K$ elements for all elements in $\tilde{\truncatedgumbel}$, so both parent node IDs and token IDs can be acquired. Make sure that both output IDs are sorted w.r.t. the corresponding values in $\tilde{\truncatedgumbel}$.}}\\
    $\parentids, \tokens, \truncatedgumbel \subs \mathrm{argtop}\text{-}K(\perturbedlogprobs)$.
    \STATE
    \Comment{\Red{// Get sum log probs for top-$K$ elements.}}\\
    \Red{$\sumlogprob \subs \logprobs[\parentids, \tokens]$.}
    \STATE
    \Comment{// When probability filtering methods(e.g., top-$p$, top-$k$) were applied, filter some elements of \Red{$\tokens, \parentids, \sumlogprob, \truncatedgumbel$} if corresponding log probability is equal to $-\infty$.}
    \STATE {\bfseries Output: $\tokens, \parentids, \sumlogprob, \truncatedgumbel$} 
\end{algorithmic}
\end{algorithm}

\clearpage
\section{Experiments}\label{appendix:C}

\subsection{Draft models}\label{sec:appendix:C:draft_models}

The following draft models are used:
\begin{itemize}
    \item For Llama 2 target models, we use the 115M \textbf{Llama 2 drafter} and \textbf{Llama 2-Chat drafter} for \textbf{Llama 2} and \textbf{Llama 2-Chat} target models, respectively. 
    \begin{itemize}
        \item \textbf{Llama 2 drafter} uses smaller Llama archiecture~\cite{touvron2023llama} and is pre-trained on the 600B-token dataset
        \item \textbf{Llama 2-Chat drafter} is the model fine-tuned from Llama 2-drafter so that it can be aligned with Llama 2-Chat-7B via distillation. 
    \end{itemize}
    \item For OPT target models, we use \textbf{OPT} with \textbf{125M} and \textbf{350M} parameters for target OPT models. 
    
\end{itemize}

\subsection{Performance Metrics}\label{sec:appendix:C:metrics}
In the experiments, we consider three metrics (except accuracy) for all target models. 
\begin{itemize}
    \item \textbf{Block efficiency}~\cite{leviathan2023fast} is the average number of tokens generated per target model call. Within a single target call, auto-regressive decoding always generates a single token, while speculative decoding methods generates 
    \begin{align*}
        \text{(Number of accepted tokens)} + 1.
    \end{align*}
    The block efficiency $\eta$ is the average over all target calls. 

    \item \textbf{Memory-Bound Speed Up (MBSU)} is the fictitious inference speed-up relative to auto-regressive decoding, where we assume each model's runtime is proportional to the model size. Specifically, 
    let $L$ denote the (maximum) length of draft sequences, which is the depth of the draft-token tree for tree-based speculative decoding methods, and $r$ denote the relative speed of running the draft model to that of the target model. The walltime improvement~\cite{leviathan2023fast,zhou2023distillspec} is
    \begin{align*}
        \frac{\eta}{L\times r + 1}.
    \end{align*}
    MBSU considers a specific case where $r$ is equal to $\text{(Size of the target model)}/\text{(Size of the draft model)}$, considering practical scenarios in memory-bound devices where loading model weights takes significant amount time, often proportional to their size. 

    \item \textbf{Token rate} is the measure of average number of generated tokens per second while running on A100 GPUs. It shows different results from MBSU since running A100 GPUs is far from memory-bound scenarios. 

\end{itemize}

\subsection{Tree Structure}
\subsubsection{Experiment for various lengths of draft sequence}
\label{sec:appendix:tree_structure:exp1}
The following trees are used for draft sequence length $L$, where SD uses a single draft sequence with length $L$. 
For each $L$, we first set RSD-C with constant branching factors always equal to 2 and set the draft-tree sizes for SpecTr and RSD-S \emph{always less than or equal to} the tree size of RSD-C.
Then, we add RSD-C with the branching factor $\mathbf{b}:=[n, 1, ..., 1]$ where $n$ is properly set to have the draft-tree size equal to that of SpecTr and RSD-S. 
In \emph{Figure~\ref{fig:main:exp1}}, we show the best results across all tree structures for each $L$ and algorithm. 
\begin{itemize}
    \item $L=2$: 
    \begin{itemize}
        \item SpecTr and RSD-S: $(K, L)\in\{(2, 2), (3, 2)\}$,
              where $K$ becomes the number of independent draft sequences for SpecTr and the beamwidth for RSD-S
        \item RSD-C: $\mathbf{b}\in\{[2, 2], [2, 1], [3, 1]\}$
              for a vector $\mathbf{b}$ of branching factors.
    \end{itemize}
    \item $L=3$
    \begin{itemize}
        \item SpecTr and RSD-S: $(K, L)\in\{(3, 3), (4, 3)\}$,
              where $K$ becomes the number of independent draft sequences for SpecTr and the beamwidth for RSD-S
        \item RSD-C: $\mathbf{b}\in\{[2, 2, 2], [3, 1, 1], [4, 1, 1]\}$
              for a vector $\mathbf{b}$ of branching factors.
    \end{itemize}
    \item $L=4$
    \begin{itemize}
        \item SpecTr and RSD-S: $(K, L)\in\{(5, 4), (7, 4)\}$,
              where $K$ becomes the number of independent draft sequences for SpecTr and the beamwidth for RSD-S
        \item RSD-C: $\mathbf{b}\in\{[2, 2, 2, 2], [5, 1, 1, 1], [7, 1, 1, 1]\}$
              for a vector $\mathbf{b}$ of branching factors.
    \end{itemize}
    \item $L=5$
    \begin{itemize}
        \item SpecTr and RSD-S: $(K, L)\in\{(6, 5), (12, 5)\}$,
              where $K$ becomes the number of independent draft sequences for SpecTr and the beamwidth for RSD-S
        \item RSD-C: $\mathbf{b}\in\{[2, 2, 2, 2, 2], [6, 1, 1, 1, 1], [12, 1, 1, 1, 1]\}$
              for a vector $\mathbf{b}$ of branching factors.
    \end{itemize}
\end{itemize}

\subsubsection{Experiment for vairous target computational budget}
\label{sec:appendix:tree_structure:exp2}
The following trees are used for target computational budgets $B$, i.e., the number of tokens to process at the target model, where $B$ becomes the draft length of SD. 
In \emph{Figure~\ref{fig:main:exp2}}, we show the best results across all tree structures for each $B$ and algorithm. 
\begin{itemize}
    \item
    $B=6$
    \begin{itemize}
        \item SpecTr and RSD-S: $(K, L)\in\{(2, 3), (3, 2)\}$,
              where $K$ becomes the number of independent draft sequences for SpecTr and the beamwidth for RSD-S
        \item RSD-C: $\mathbf{b}\in\{[2, 1, 1], [2, 2], [3, 1]\}$
              for a vector $\mathbf{b}$ of branching factors.
    \end{itemize}
    
    \item
    $B=10$
    \begin{itemize}
        \item SpecTr and RSD-S: $(K, L)\in\{(2, 5), (5, 2)\}$,
              where $K$ becomes the number of independent draft sequences for SpecTr and the beamwidth for RSD-S
        \item RSD-C: $\mathbf{b}\in\{[2, 1, 1, 1, 1], [2, 2, 1], [5, 1]\}$
              for a vector $\mathbf{b}$ of branching factors.
    \end{itemize}
    
    \item
    $B=14$
    \begin{itemize}
        \item SpecTr and RSD-S: $(K, L)\in\{(2, 7), (7, 2)\}$,
              where $K$ becomes the number of independent draft sequences for SpecTr and the beamwidth for RSD-S
        \item RSD-C: $\mathbf{b}\in\{[2, 1, 1, 1, 1, 1, 1], [2, 2, 2], [7, 1]\}$
              for a vector $\mathbf{b}$ of branching factors.
    \end{itemize}
    
    \item
    $B=21$
    \begin{itemize}
        \item SpecTr and RSD-S: $(K, L)\in\{(3, 7), (7, 3)\}$,
              where $K$ becomes the number of independent draft sequences for SpecTr and the beamwidth for RSD-S
        \item RSD-C: $\mathbf{b}\in\{[3, 1, 1, 1, 1, 1, 1], [3, 2, 2], [7, 1, 1]\}$
              for a vector $\mathbf{b}$ of branching factors.
    \end{itemize}
    
    \item
    $B=30$
    \begin{itemize}
        \item SpecTr and RSD-S: $(K, L)\in\{(5, 6), (6, 5)\}$,
              where $K$ becomes the number of independent draft sequences for SpecTr and the beamwidth for RSD-S
        \item RSD-C: $\mathbf{b}\in\{[2, 2, 2, 2], [5, 1, 1, 1, 1, 1], [6, 1, 1, 1, 1]\}$
              for a vector $\mathbf{b}$ of branching factors.
    \end{itemize}
    
\end{itemize}

\clearpage
\subsection{Experiment results with plots}
\label{sec:appendix:C:experiments}
\subsubsection{Block efficiency, MBSU, token rate and accuracy for various lengths of draft sequence}
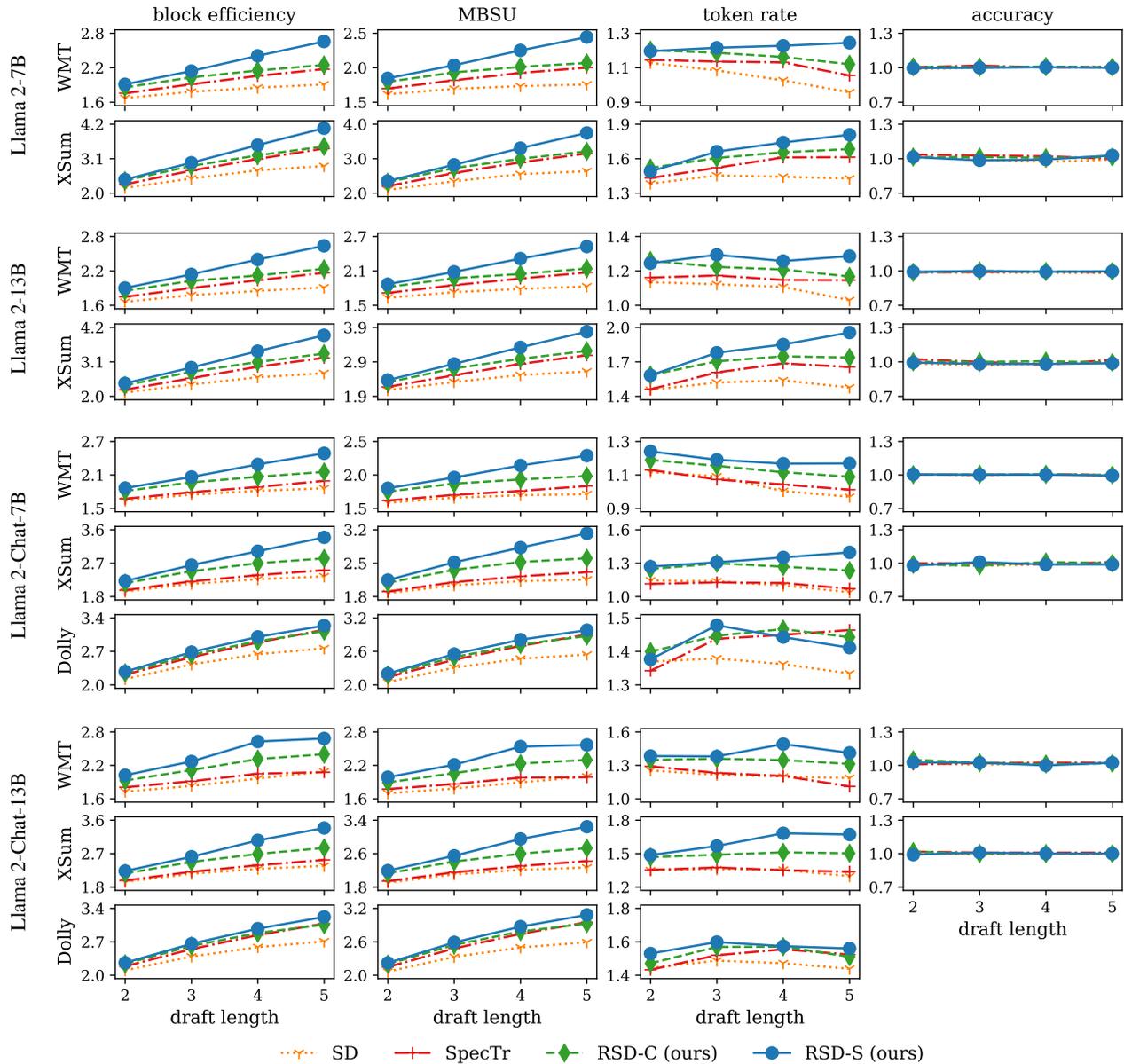
\begin{figure*}[hbt!]
\begin{center}
\def\x{\columnwidth/2+0.2}
\begin{tikzpicture}
    \node at (0,0) {\includegraphics[width=0.95\textwidth]{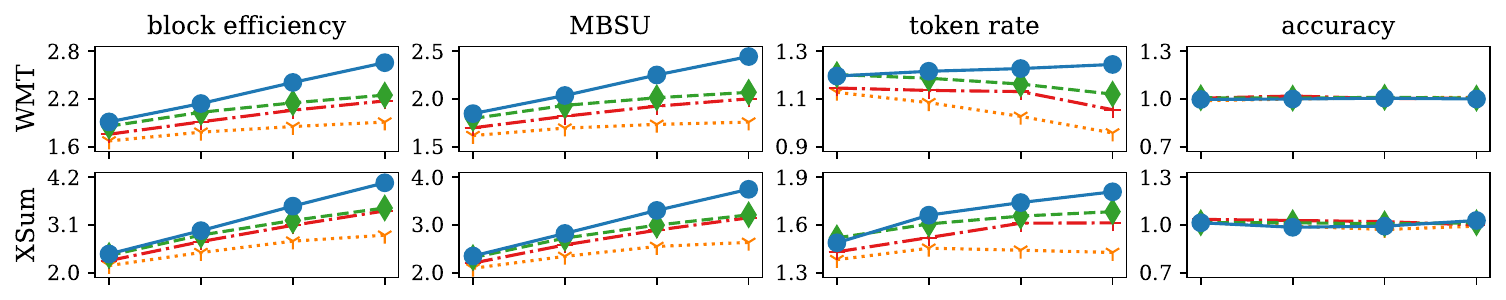}};
    \node [rotate=90] at (-\x, 0.0) {\small Llama 2-7B};
    
    \node at (0,-3.2) {\includegraphics[width=0.95\textwidth]{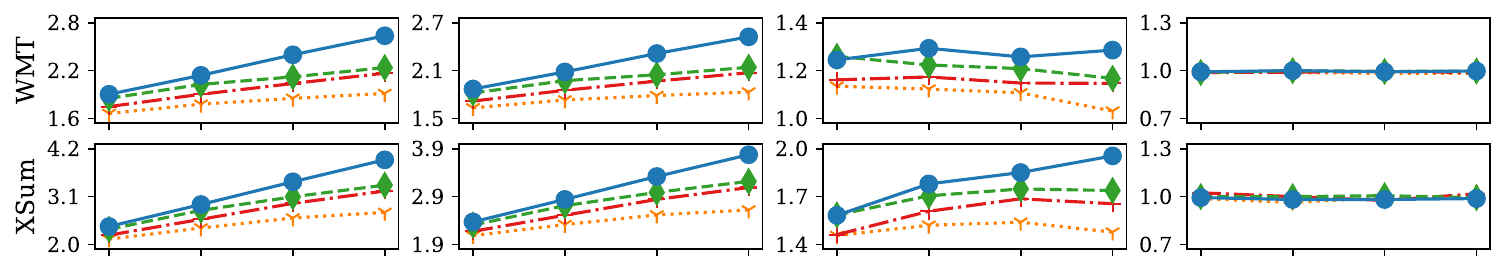}};
    \node [rotate=90] at (-\x, -3.2) {\small Llama 2-13B};
    
    \node at (0,-6.9) {\includegraphics[width=0.95\textwidth]{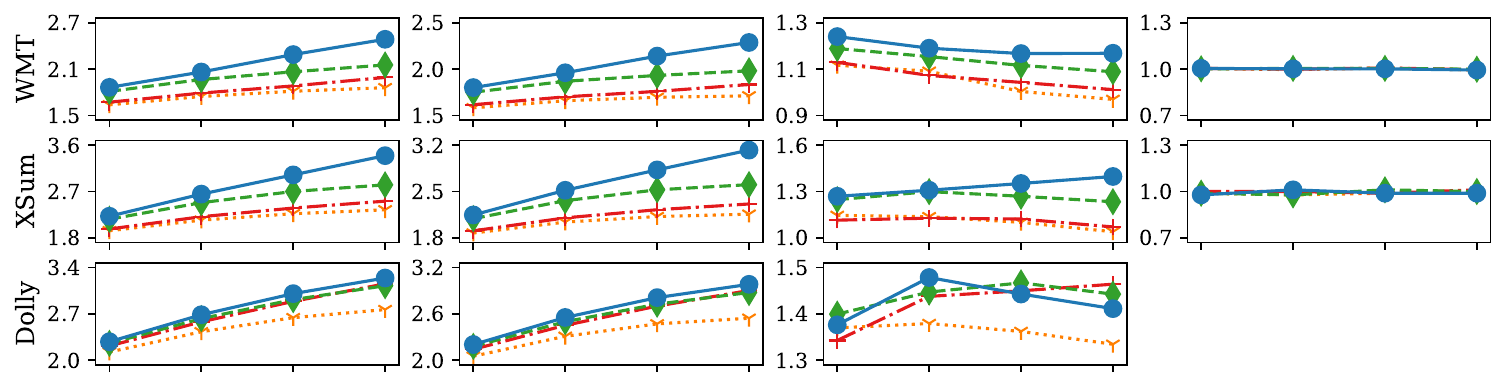}};
    \node [rotate=90] at (-\x, -6.9) {\small Llama 2-Chat-7B};
    
    \node at (0,-11.9) {\includegraphics[width=0.95\textwidth]{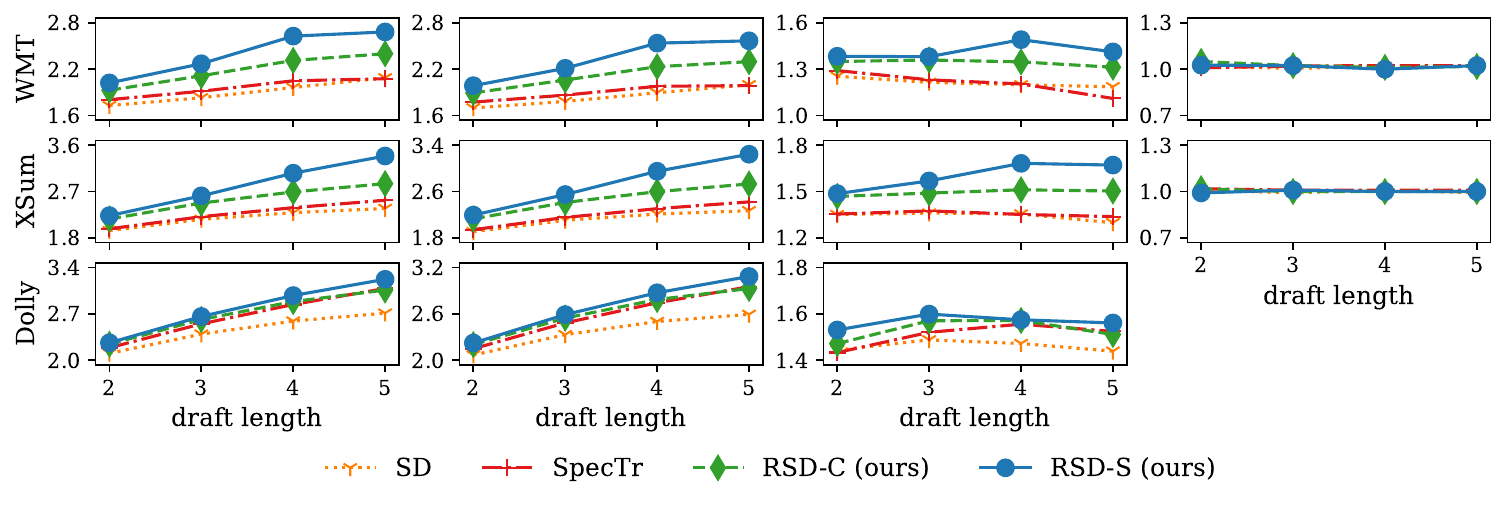}};
    \node [rotate=90] at (-\x, -11.9+0.7) {\small Llama 2-Chat-13B};
\end{tikzpicture}
\vspace{-0.2in}
\caption{Block efficiency, MBSU, token rate and accuracy for varying lengths of draft sequence are given for multiple target models: Llama 2-7B, Llama 2-13B, Llama 2-Chat-7B, Llama 2-Chat-13B. Chat models use the same draft model, while the other models use the same draft model different from the one for chat models. All results are normalized w.r.t. the values of AR decoding.}
\label{}
\end{center}
\end{figure*}

\begin{figure*}[h]
\begin{center}
\def\x{\columnwidth/2+0.2}
\begin{tikzpicture}
    \node at (0,0) {\includegraphics[width=0.95\textwidth]{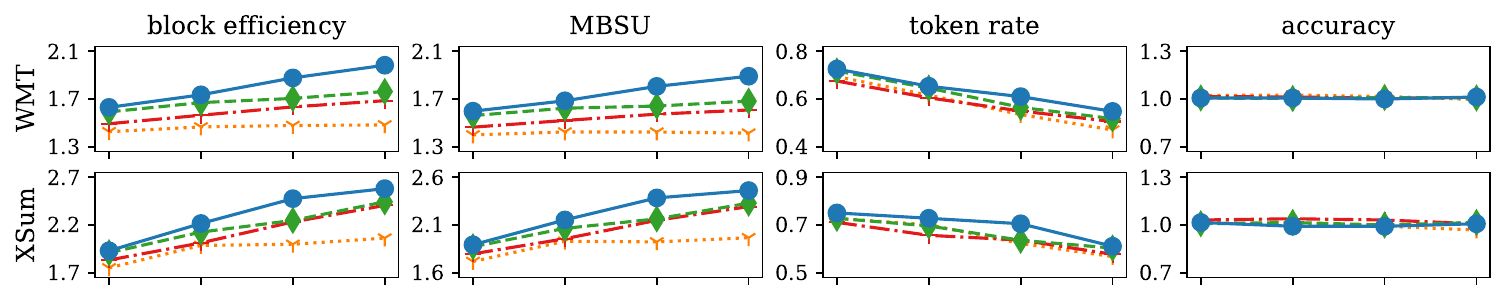}};
    \node [rotate=90] at (-\x, -0.2) {\small OPT-125M-13B};
    
    \node at (0,-3.1) {\includegraphics[width=0.95\textwidth]{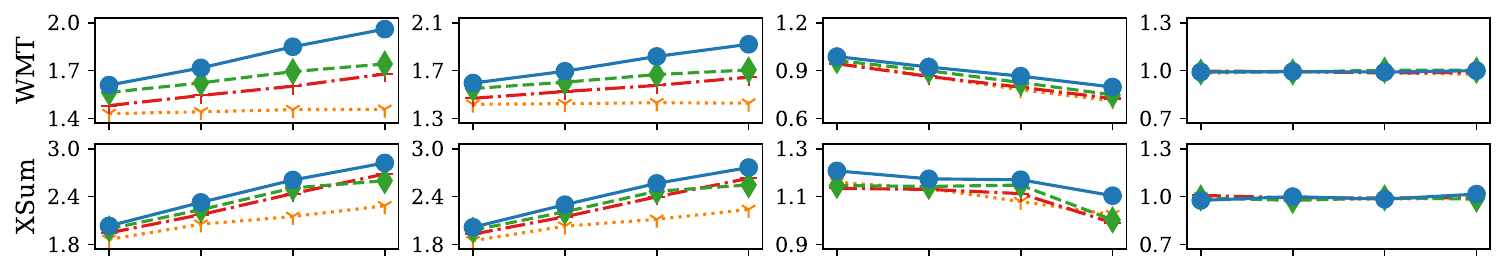}};
    \node [rotate=90] at (-\x, -3.1) {\small OPT-125M-30B};
    
    \node at (0,-6.2) {\includegraphics[width=0.95\textwidth]{main_figure/plot_OPT-125M-30B_v2_main.pdf}};
    \node [rotate=90] at (-\x, -6.2) {\small OPT-125M-66B};
    
    \node at (0,-9.3) {\includegraphics[width=0.95\textwidth]{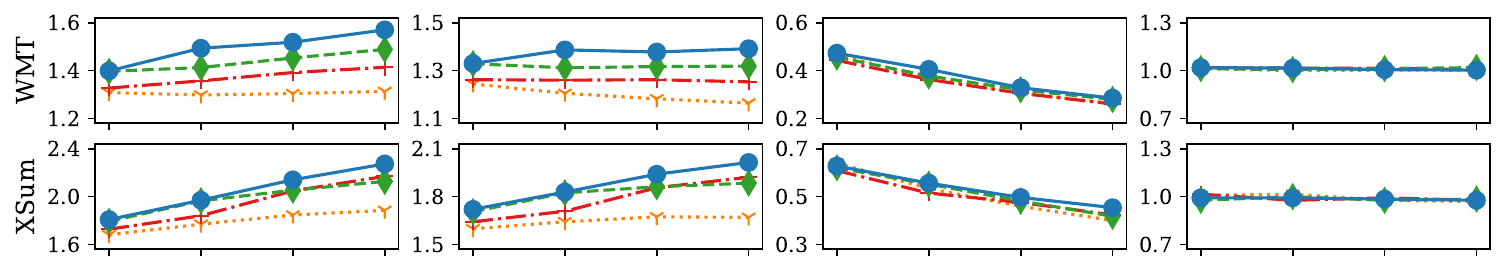}};
    \node [rotate=90] at (-\x, -9.3) {\small OPT-350M-13B};
    
    \node at (0,-12.4) {\includegraphics[width=0.95\textwidth]{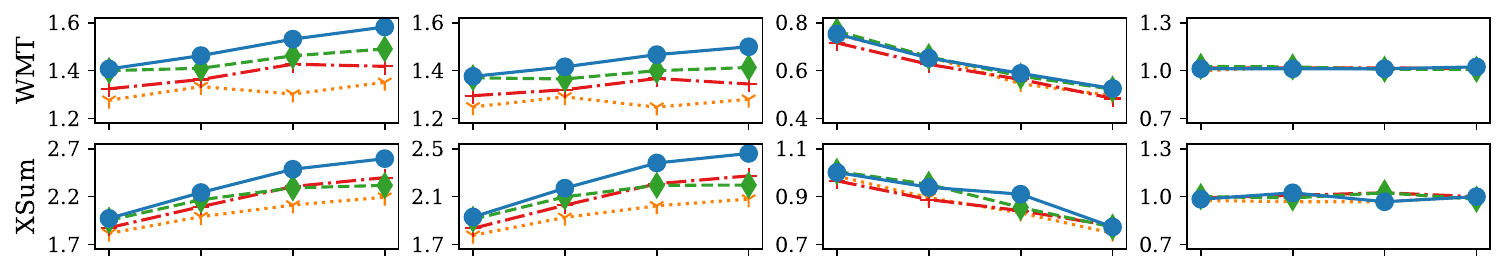}};
    \node [rotate=90] at (-\x, -12.4) {\small OPT-350M-30B};
    
    \node at (0,-16.1) {\includegraphics[width=0.95\textwidth]{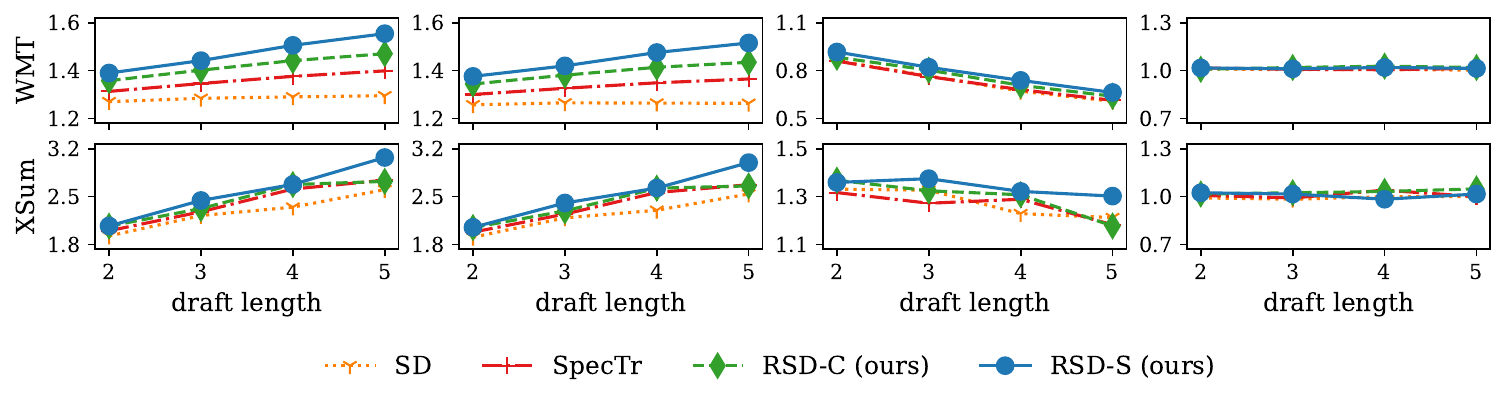}};
    \node [rotate=90] at (-\x, -15.4) {\small OPT-350M-66B};
\end{tikzpicture}
\vspace{-0.2in}
\caption{Block efficiency, MBSU, token rate and accuracy for varying lengths of draft sequence are given for multiple pairs of draft and target models: the size of draft model is in $\{\text{125M, 350M}\}$, and the size of target model is in $\{\text{13B, 30B, 66B}\}$. All results are normalized w.r.t. the values of AR decoding.}
\label{}
\end{center}
\end{figure*}

\clearpage
\subsubsection{Block efficiency, MBSU, token rate and accuracy for various target computational budget}
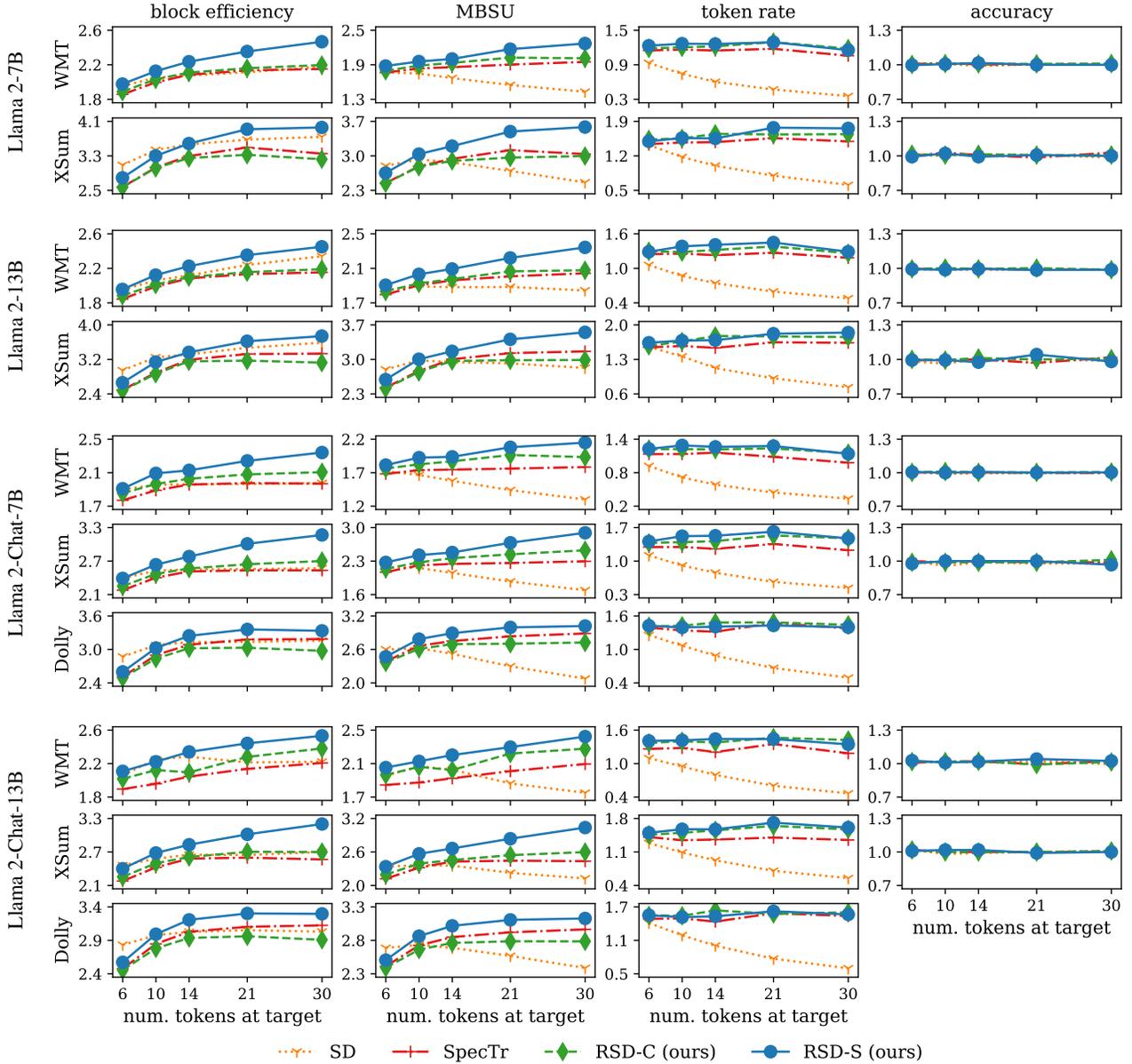
\begin{figure*}[hbt!]
\begin{center}
\def\x{\columnwidth/2+0.2}
\begin{tikzpicture}
    \node at (0,0) {\includegraphics[width=0.95\textwidth]{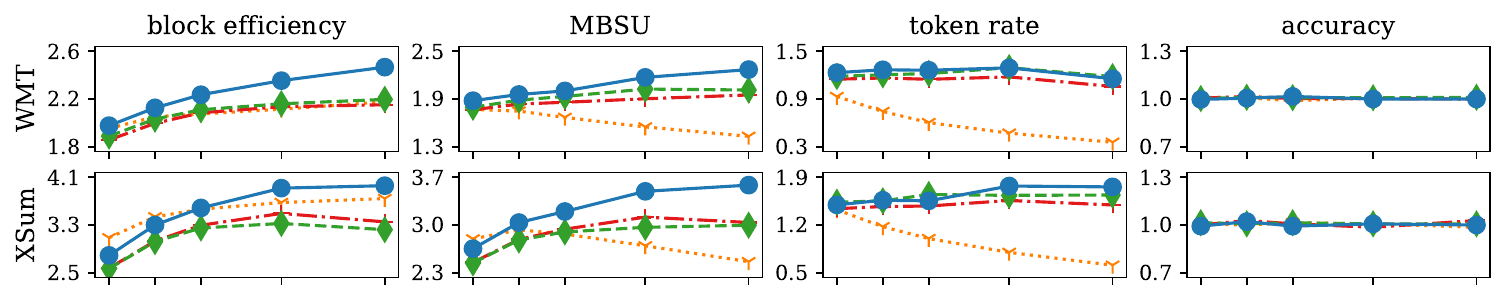}};
    \node [rotate=90] at (-\x, 0.0) {\small Llama 2-7B};
    
    \node at (0,-3.2) {\includegraphics[width=0.95\textwidth]{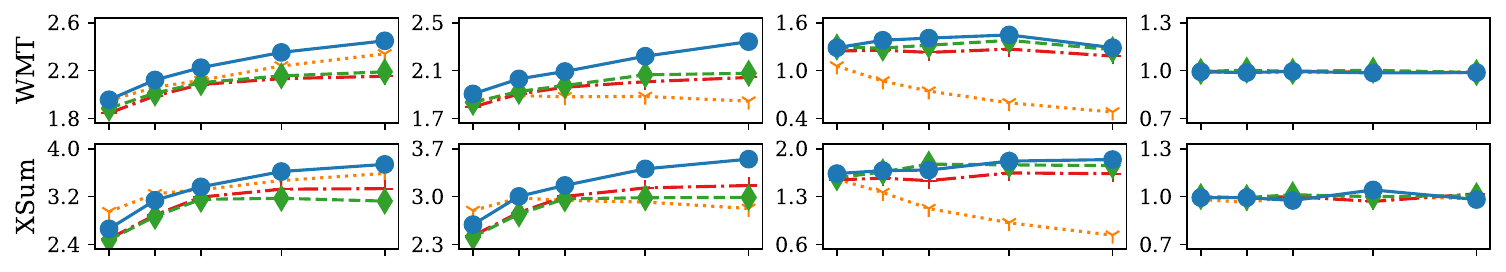}};
    \node [rotate=90] at (-\x, -3.2) {\small Llama 2-13B};
    
    \node at (0,-6.9) {\includegraphics[width=0.95\textwidth]{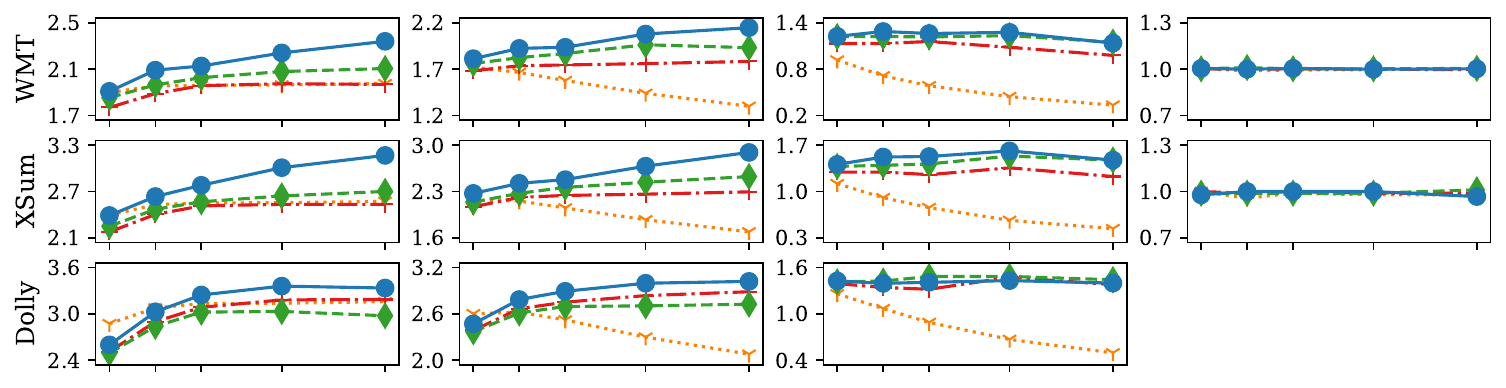}};
    \node [rotate=90] at (-\x, -6.9) {\small Llama 2-Chat-7B};
    
    \node at (0,-11.9) {\includegraphics[width=0.95\textwidth]{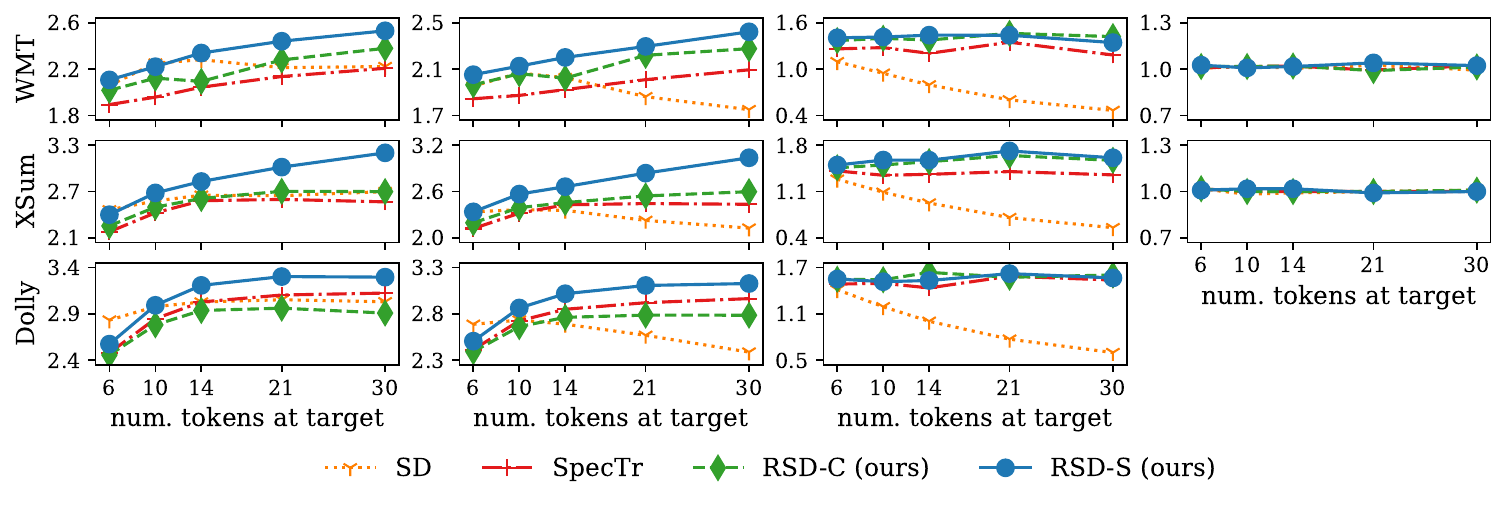}};
    \node [rotate=90] at (-\x, -11.9+0.7) {\small Llama 2-Chat-13B};
\end{tikzpicture}
\vspace{-0.2in}
\caption{Block efficiency, MBSU, token rate and accuracy for varying numbers of tokens processed at the target model are given for multiple target models: Llama 2-7B, Llama 2-13B, Llama 2-Chat-7B, Llama 2-Chat-13B. Chat models use the same draft model, while the other models use the same draft model different from the one for chat models.  All results are normalized w.r.t. the values of AR decoding.}
\label{}
\end{center}
\end{figure*}

\begin{figure*}[hbt!]
\begin{center}
\def\x{\columnwidth/2+0.2}
\begin{tikzpicture}
    \node at (0,0) {\includegraphics[width=0.95\textwidth]{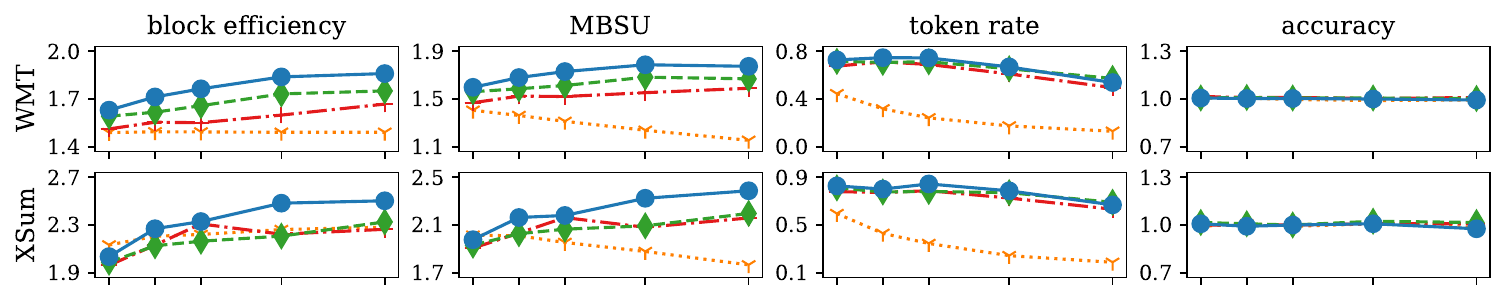}};
    \node [rotate=90] at (-\x, -0.2) {\small OPT-125M-13B};
    
    \node at (0,-3.2) {\includegraphics[width=0.95\textwidth]{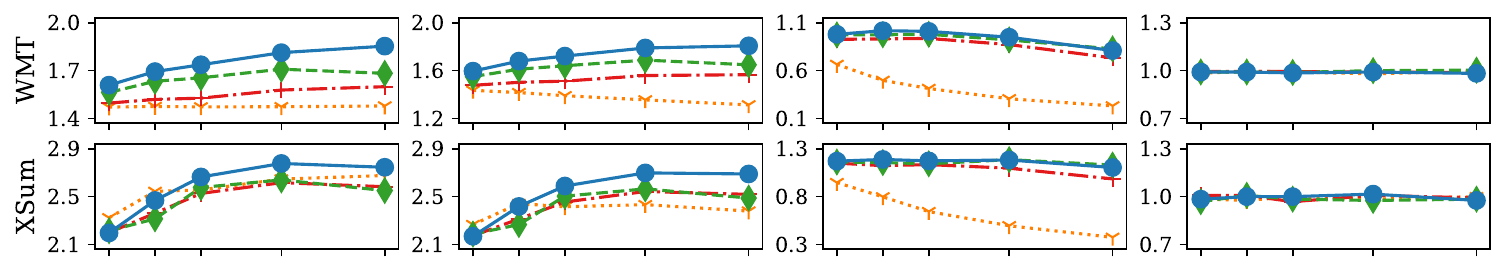}};
    \node [rotate=90] at (-\x, -3.2) {\small OPT-125M-30B};
    
    \node at (0,-6.4) {\includegraphics[width=0.95\textwidth]{main_figure/plot_OPT-125M-30B_v1_main.pdf}};
    \node [rotate=90] at (-\x, -6.4) {\small OPT-125M-66B};
    
    \node at (0,-9.6) {\includegraphics[width=0.95\textwidth]{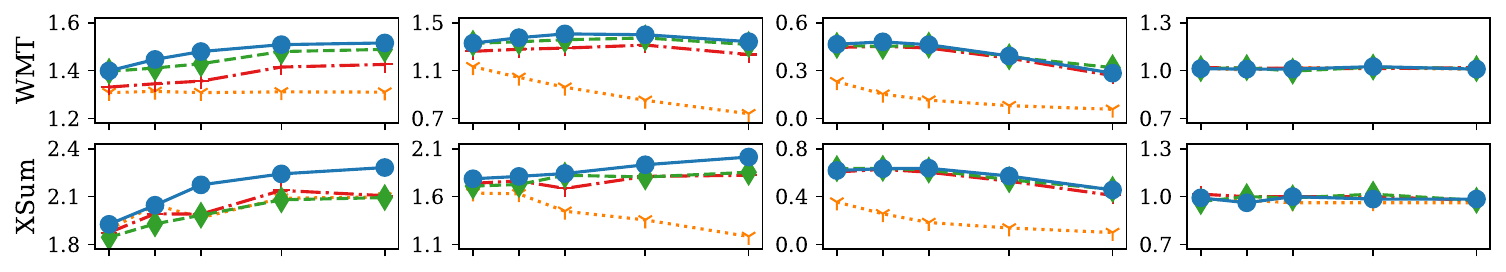}};
    \node [rotate=90] at (-\x, -9.6) {\small OPT-350M-13B};
    
    \node at (0,-12.8) {\includegraphics[width=0.95\textwidth]{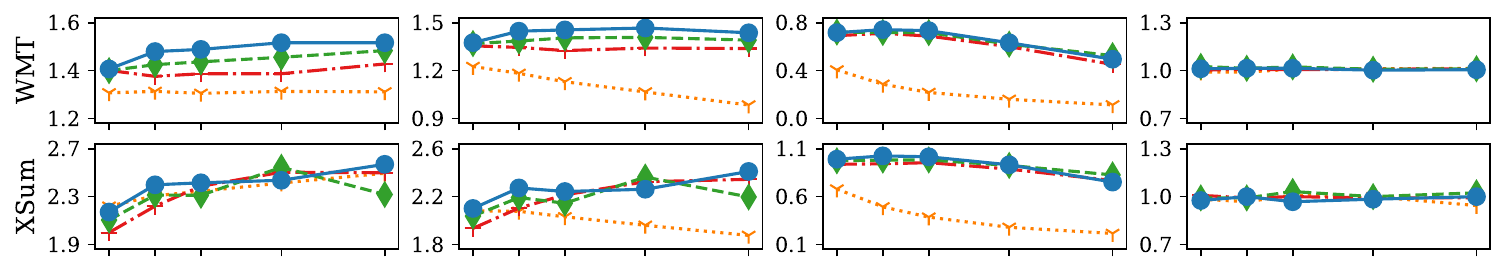}};
    \node [rotate=90] at (-\x, -12.8) {\small OPT-350M-30B};
    
    \node at (0,-16.7) {\includegraphics[width=0.95\textwidth]{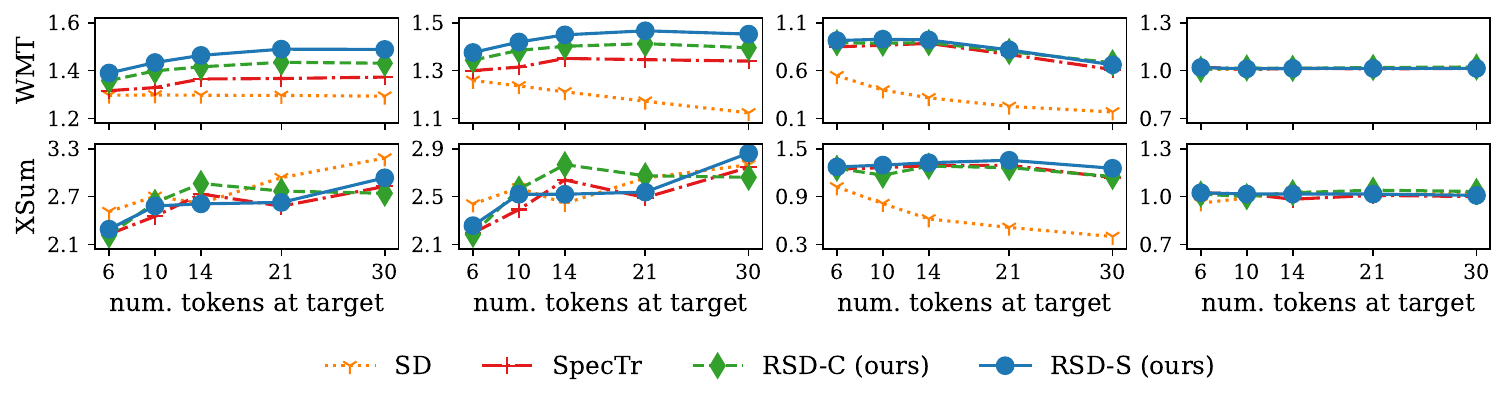}};
    \node [rotate=90] at (-\x, -16.0) {\small OPT-350M-66B};
\end{tikzpicture}
\vspace{-0.2in}
\caption{Block efficiency, MBSU, token rate and accuracy for varying numbers of tokens processed at the target model are given for multiple pairs of draft and target models: the size of draft model is in $\{\text{125M, 350M}\}$, and the size of target model is in $\{\text{13B, 30B, 66B}\}$. All results are normalized w.r.t. the values of AR decoding.}
\label{}
\end{center}
\end{figure*}

\clearpage
\subsection{Experient results with tables}
For readers curious about raw numbers, we remain all the numbers used for plots as tables in this section. 

\subsubsection{Block efficiency, MBSU, token rate and accuracy for varying lengths of draft sequence}
\begin{itemize}

    \item Llama 2-7B (with 115M drafter)
    \begin{itemize}
        \item 
        XSum (\emph{Table}~\ref{table_7b-base_xsum_v2}),
        WMT (\emph{Table}~\ref{table_7b-base_wmt_v2})
    \end{itemize}
    
    \item Llama 2-13B (with 115M drafter)
    \begin{itemize}
        \item
        XSum (\emph{Table}~\ref{table_13b-base_xsum_v2}),
        WMT (\emph{Table}~\ref{table_13b-base_wmt_v2})
    \end{itemize}
    
    \item Llama 2-70B (with 115M drafter)
    \begin{itemize}
        \item
        XSum (\emph{Table}~\ref{table_70b-base_xsum_v2}),
        WMT (\emph{Table}~\ref{table_70b-base_wmt_v2})
    \end{itemize}

    \item Llama 2-Chat-7B (with 115M drafter)
    \begin{itemize}
        \item
        XSum (\emph{Table}~\ref{table_7b-chat_xsum_v2}),
        WMT (\emph{Table}~\ref{table_7b-chat_wmt_v2}), 
        Dolly (\emph{Table}~\ref{table_7b-chat_dolly_v2})
    \end{itemize}
    
    \item Llama 2-Chat-13B (with 115M drafter)
    \begin{itemize}
        \item 
        XSum (\emph{Table}~\ref{table_13b-chat_xsum_v2}),
        WMT (\emph{Table}~\ref{table_13b-chat_wmt_v2}), 
        Dolly (\emph{Table}~\ref{table_13b-chat_dolly_v2})
    \end{itemize}
    
    \item Llama 2-Chat-70B (with 115M drafter)
    \begin{itemize}
        \item 
        XSum (\emph{Table}~\ref{table_70b-chat_xsum_v2}),
        WMT (\emph{Table}~\ref{table_70b-chat_wmt_v2}), 
        Dolly (\emph{Table}~\ref{table_70b-chat_dolly_v2})
    \end{itemize}

    \item OPT-13B (with OPT-125M drafter)
    \begin{itemize}
        \item 
        XSum (\emph{Table}~\ref{table_125M_13B-base_xsum_v2}),
        WMT (\emph{Table}~\ref{table_125M_13B-base_wmt_v2})
    \end{itemize}  
    
    \item OPT-30B (with OPT-125M drafter)
    \begin{itemize}
        \item 
        XSum (\emph{Table}~\ref{table_125M_30B-base_xsum_v2}),
        WMT (\emph{Table}~\ref{table_125M_30B-base_wmt_v2})
    \end{itemize}  
    
    \item OPT-66B (with OPT-125M drafter)
    \begin{itemize}
        \item 
        XSum (\emph{Table}~\ref{table_125M_66B-base_xsum_v2}),
        WMT (\emph{Table}~\ref{table_125M_66B-base_wmt_v2})
    \end{itemize}

    \item OPT-13B (with OPT-350M drafter)
    \begin{itemize}
        \item 
        XSum (\emph{Table}~\ref{table_350M_13B-base_xsum_v2}),
        WMT (\emph{Table}~\ref{table_350M_13B-base_wmt_v2})
    \end{itemize}  
    
    \item OPT-30B (with OPT-350M drafter)
    \begin{itemize}
        \item 
        XSum (\emph{Table}~\ref{table_350M_30B-base_xsum_v2}),
        WMT (\emph{Table}~\ref{table_350M_30B-base_wmt_v2})
    \end{itemize}  
    
    \item OPT-66B (with OPT-350M drafter)
    \begin{itemize}
        \item 
        XSum (\emph{Table}~\ref{table_350M_66B-base_xsum_v2}),
        WMT (\emph{Table}~\ref{table_350M_66B-base_wmt_v2})
    \end{itemize}  
    
\end{itemize}

\clearpage

\subsubsection{Block efficiency, MBSU, token rate and accuracy for varying numbers of tokens processed at the target model}
\label{sec:appendix:exp1:table}
\begin{itemize}

    \item Llama 2-7B (with 115M drafter)
    \begin{itemize}
        \item 
        XSum (\emph{Table}~\ref{table_7b-base_xsum_v1}),
        WMT (\emph{Table}~\ref{table_7b-base_wmt_v1})
    \end{itemize}
    
    \item Llama 2-13B (with 115M drafter)
    \begin{itemize}
        \item
        XSum (\emph{Table}~\ref{table_13b-base_xsum_v1}),
        WMT (\emph{Table}~\ref{table_13b-base_wmt_v1})
    \end{itemize}
    
    \item Llama 2-70B (with 115M drafter)
    \begin{itemize}
        \item
        XSum (\emph{Table}~\ref{table_70b-base_xsum_v1}),
        WMT (\emph{Table}~\ref{table_70b-base_wmt_v1})
    \end{itemize}

    \item Llama 2-Chat-7B (with 115M drafter)
    \begin{itemize}
        \item
        XSum (\emph{Table}~\ref{table_7b-chat_xsum_v1}),
        WMT (\emph{Table}~\ref{table_7b-chat_wmt_v1}), 
        Dolly (\emph{Table}~\ref{table_7b-chat_dolly_v1})
    \end{itemize}
    
    \item Llama 2-Chat-13B (with 115M drafter)
    \begin{itemize}
        \item 
        XSum (\emph{Table}~\ref{table_13b-chat_xsum_v1}),
        WMT (\emph{Table}~\ref{table_13b-chat_wmt_v1}), 
        Dolly (\emph{Table}~\ref{table_13b-chat_dolly_v1})
    \end{itemize}
    
    \item Llama 2-Chat-70B (with 115M drafter)
    \begin{itemize}
        \item 
        XSum (\emph{Table}~\ref{table_70b-chat_xsum_v1}),
        WMT (\emph{Table}~\ref{table_70b-chat_wmt_v1}), 
        Dolly (\emph{Table}~\ref{table_70b-chat_dolly_v1})
    \end{itemize}

    \item OPT-13B (with OPT-125M drafter)
    \begin{itemize}
        \item 
        XSum (\emph{Table}~\ref{table_125M_13B-base_xsum_v1}),
        WMT (\emph{Table}~\ref{table_125M_13B-base_wmt_v1})
    \end{itemize}  
    
    \item OPT-30B (with OPT-125M drafter)
    \begin{itemize}
        \item 
        XSum (\emph{Table}~\ref{table_125M_30B-base_xsum_v1}),
        WMT (\emph{Table}~\ref{table_125M_30B-base_wmt_v1})
    \end{itemize}  
    
    \item OPT-66B (with OPT-125M drafter)
    \begin{itemize}
        \item 
        XSum (\emph{Table}~\ref{table_125M_66B-base_xsum_v1}),
        WMT (\emph{Table}~\ref{table_125M_66B-base_wmt_v1})
    \end{itemize}

    \item OPT-13B (with OPT-350M drafter)
    \begin{itemize}
        \item 
        XSum (\emph{Table}~\ref{table_350M_13B-base_xsum_v1}),
        WMT (\emph{Table}~\ref{table_350M_13B-base_wmt_v1})
    \end{itemize}  
    
    \item OPT-30B (with OPT-350M drafter)
    \begin{itemize}
        \item 
        XSum (\emph{Table}~\ref{table_350M_30B-base_xsum_v1}),
        WMT (\emph{Table}~\ref{table_350M_30B-base_wmt_v1})
    \end{itemize}  
    
    \item OPT-66B (with OPT-350M drafter)
    \begin{itemize}
        \item 
        XSum (\emph{Table}~\ref{table_350M_66B-base_xsum_v1}),
        WMT (\emph{Table}~\ref{table_350M_66B-base_wmt_v1})
    \end{itemize}  
    
\end{itemize}

\clearpage

\begin{table}[h]
\centering
\caption{We summarize experiment results with Llama 2-7B target and 115M draft for the XSum task with various draft lengths.
                               Draft Length (DL) means the (maximum) length for all draft token sequences generated by the draft model.
                               For decoders (Dec.), we consider Auto-Regressive sampling (AR), Speculative Decoding (SD), SpecTr, 
                               Recursive Speculative Decoding with Constant branching factors (RSD-C),
                               Recursive Speculative Decoding with Stochastic Beam Search (RSD-S).
                               The contents in decoder specification (Spec.) have different meanings for each decoder;
                               $K\times L$ means the number $K$ of draft paths and draft length $L$ for SpecTr;
                               it describes constant branching factors for each level of the tree (from root to leaf) for RSD-C;
                               $K\times L$ means the beamwidth $K$ and draft length $L$ for RSD-S.
                               Block efficiency (Eff.), Memory Bound Speed Up (MBSU), Token Rate (TR), and Accuracy (Acc.) are given for each algorithm.
                               For each group of rows having the same complexity, we highlight the best values for all columns except accuracy.
                            }
\label{table_7b-base_xsum_v2}
\vskip 0.15in
\begin{tabular}{lllllrrrr}
\toprule
 &  &  &  &  & Eff. & MBSU & TR & Acc. \\
Model & Task & DL & Dec. & Spec. &  &  &  &  \\
\midrule
\multirow[c]{33}{*}{Llama 2-7B} & \multirow[c]{33}{*}{XSum} & 0 & AR & - & 1.000 & 1.000 & 37.269 & 0.141 \\
\cline{3-9} \cline{4-9}
 &  & \multirow[c]{8}{*}{2} & SD & 2 & 2.166 & 2.093 & 51.515 & 0.143 \\
\cline{4-9}
 &  &  & \multirow[c]{2}{*}{SpecTr} & 2$\times$2 & 2.218 & 2.143 & 52.950 & 0.146 \\
 &  &  &  & 3$\times$2 & 2.279 & 2.202 & 53.346 & 0.139 \\
\cline{4-9}
 &  &  & \multirow[c]{3}{*}{RSD-C} & 2$-$1 & 2.267 & 2.191 & 53.980 & 0.142 \\
 &  &  &  & 2$-$2 & 2.398 & 2.317 & \color{red} \bfseries 56.609 & 0.143 \\
 &  &  &  & 3$-$1 & 2.291 & 2.214 & 53.930 & 0.140 \\
\cline{4-9}
 &  &  & \multirow[c]{2}{*}{RSD-S} & 2$\times$2 & 2.367 & 2.288 & 54.586 & 0.143 \\
 &  &  &  & 3$\times$2 & \color{red} \bfseries 2.432 & \color{red} \bfseries 2.350 & 55.465 & 0.140 \\
\cline{3-9} \cline{4-9}
 &  & \multirow[c]{8}{*}{3} & SD & 3 & 2.465 & 2.343 & 54.195 & 0.140 \\
\cline{4-9}
 &  &  & \multirow[c]{2}{*}{SpecTr} & 3$\times$3 & 2.644 & 2.513 & 55.273 & 0.140 \\
 &  &  &  & 4$\times$3 & 2.718 & 2.583 & 56.688 & 0.145 \\
\cline{4-9}
 &  &  & \multirow[c]{3}{*}{RSD-C} & 2$-$2$-$2 & 2.868 & 2.726 & 59.879 & 0.141 \\
 &  &  &  & 3$-$1$-$1 & 2.641 & 2.511 & 55.384 & 0.143 \\
 &  &  &  & 4$-$1$-$1 & 2.688 & 2.555 & 57.518 & 0.140 \\
\cline{4-9}
 &  &  & \multirow[c]{2}{*}{RSD-S} & 3$\times$3 & 2.927 & 2.782 & 58.843 & 0.139 \\
 &  &  &  & 4$\times$3 & \color{red} \bfseries 2.970 & \color{red} \bfseries 2.823 & \color{red} \bfseries 61.937 & 0.136 \\
\cline{3-9} \cline{4-9}
 &  & \multirow[c]{8}{*}{4} & SD & 4 & 2.728 & 2.551 & 53.731 & 0.137 \\
\cline{4-9}
 &  &  & \multirow[c]{2}{*}{SpecTr} & 5$\times$4 & 2.974 & 2.781 & 56.002 & 0.144 \\
 &  &  &  & 7$\times$4 & 3.093 & 2.892 & 60.053 & 0.139 \\
\cline{4-9}
 &  &  & \multirow[c]{3}{*}{RSD-C} & 2$-$2$-$2$-$2 & 3.205 & 2.997 & 61.723 & 0.142 \\
 &  &  &  & 5$-$1$-$1$-$1 & 2.898 & 2.710 & 56.343 & 0.141 \\
 &  &  &  & 7$-$1$-$1$-$1 & 2.974 & 2.781 & 58.423 & 0.137 \\
\cline{4-9}
 &  &  & \multirow[c]{2}{*}{RSD-S} & 5$\times$4 & 3.427 & 3.205 & \color{red} \bfseries 64.887 & 0.140 \\
 &  &  &  & 7$\times$4 & \color{red} \bfseries 3.535 & \color{red} \bfseries 3.306 & 64.456 & 0.140 \\
\cline{3-9} \cline{4-9}
 &  & \multirow[c]{8}{*}{5} & SD & 5 & 2.865 & 2.636 & 53.199 & 0.140 \\
\cline{4-9}
 &  &  & \multirow[c]{2}{*}{SpecTr} & 6$\times$5 & 3.209 & 2.953 & 55.424 & 0.141 \\
 &  &  &  & 12$\times$5 & 3.425 & 3.152 & 60.133 & 0.141 \\
\cline{4-9}
 &  &  & \multirow[c]{3}{*}{RSD-C} & 2$-$2$-$2$-$2$-$2 & 3.492 & 3.213 & 62.753 & 0.143 \\
 &  &  &  & 6$-$1$-$1$-$1$-$1 & 3.133 & 2.883 & 55.796 & 0.142 \\
 &  &  &  & 12$-$1$-$1$-$1$-$1 & 3.249 & 2.990 & 58.352 & 0.141 \\
\cline{4-9}
 &  &  & \multirow[c]{2}{*}{RSD-S} & 6$\times$5 & 3.811 & 3.507 & 65.160 & 0.141 \\
 &  &  &  & 12$\times$5 & \color{red} \bfseries 4.073 & \color{red} \bfseries 3.748 & \color{red} \bfseries 67.409 & 0.145 \\
\cline{1-9} \cline{2-9} \cline{3-9} \cline{4-9}
\bottomrule
\end{tabular}
\end{table}
\begin{table}[h]
\centering
\caption{We summarize experiment results with Llama 2-7B target and 115M draft for the WMT task with various draft lengths.
                               Draft Length (DL) means the (maximum) length for all draft token sequences generated by the draft model.
                               For decoders (Dec.), we consider Auto-Regressive sampling (AR), Speculative Decoding (SD), SpecTr, 
                               Recursive Speculative Decoding with Constant branching factors (RSD-C),
                               Recursive Speculative Decoding with Stochastic Beam Search (RSD-S).
                               The contents in decoder specification (Spec.) have different meanings for each decoder;
                               $K\times L$ means the number $K$ of draft paths and draft length $L$ for SpecTr;
                               it describes constant branching factors for each level of the tree (from root to leaf) for RSD-C;
                               $K\times L$ means the beamwidth $K$ and draft length $L$ for RSD-S.
                               Block efficiency (Eff.), Memory Bound Speed Up (MBSU), Token Rate (TR), and Accuracy (Acc.) are given for each algorithm.
                               For each group of rows having the same complexity, we highlight the best values for all columns except accuracy.
                            }
\label{table_7b-base_wmt_v2}
\vskip 0.15in
\begin{tabular}{lllllrrrr}
\toprule
 &  &  &  &  & Eff. & MBSU & TR & Acc. \\
Model & Task & DL & Dec. & Spec. &  &  &  &  \\
\midrule
\multirow[c]{33}{*}{Llama 2-7B} & \multirow[c]{33}{*}{WMT} & 0 & AR & - & 1.000 & 1.000 & 37.631 & 0.374 \\
\cline{3-9} \cline{4-9}
 &  & \multirow[c]{8}{*}{2} & SD & 2 & 1.673 & 1.617 & 42.447 & 0.370 \\
\cline{4-9}
 &  &  & \multirow[c]{2}{*}{SpecTr} & 2$\times$2 & 1.727 & 1.669 & 42.013 & 0.370 \\
 &  &  &  & 3$\times$2 & 1.757 & 1.698 & 43.128 & 0.376 \\
\cline{4-9}
 &  &  & \multirow[c]{3}{*}{RSD-C} & 2$-$1 & 1.768 & 1.708 & 43.044 & 0.377 \\
 &  &  &  & 2$-$2 & 1.858 & 1.796 & \color{red} \bfseries 45.245 & 0.372 \\
 &  &  &  & 3$-$1 & 1.819 & 1.758 & 44.482 & 0.375 \\
\cline{4-9}
 &  &  & \multirow[c]{2}{*}{RSD-S} & 2$\times$2 & 1.824 & 1.763 & 43.536 & 0.370 \\
 &  &  &  & 3$\times$2 & \color{red} \bfseries 1.912 & \color{red} \bfseries 1.847 & 45.018 & 0.373 \\
\cline{3-9} \cline{4-9}
 &  & \multirow[c]{8}{*}{3} & SD & 3 & 1.783 & 1.695 & 40.816 & 0.374 \\
\cline{4-9}
 &  &  & \multirow[c]{2}{*}{SpecTr} & 3$\times$3 & 1.890 & 1.796 & 42.746 & 0.381 \\
 &  &  &  & 4$\times$3 & 1.913 & 1.819 & 41.990 & 0.379 \\
\cline{4-9}
 &  &  & \multirow[c]{3}{*}{RSD-C} & 2$-$2$-$2 & 2.033 & 1.933 & 44.669 & 0.372 \\
 &  &  &  & 3$-$1$-$1 & 1.940 & 1.844 & 42.981 & 0.370 \\
 &  &  &  & 4$-$1$-$1 & 1.981 & 1.883 & 43.791 & 0.376 \\
\cline{4-9}
 &  &  & \multirow[c]{2}{*}{RSD-S} & 3$\times$3 & 2.064 & 1.962 & 43.684 & 0.372 \\
 &  &  &  & 4$\times$3 & \color{red} \bfseries 2.143 & \color{red} \bfseries 2.037 & \color{red} \bfseries 45.766 & 0.374 \\
\cline{3-9} \cline{4-9}
 &  & \multirow[c]{8}{*}{4} & SD & 4 & 1.854 & 1.734 & 38.651 & 0.377 \\
\cline{4-9}
 &  &  & \multirow[c]{2}{*}{SpecTr} & 5$\times$4 & 2.023 & 1.892 & 41.134 & 0.375 \\
 &  &  &  & 7$\times$4 & 2.059 & 1.925 & 42.573 & 0.373 \\
\cline{4-9}
 &  &  & \multirow[c]{3}{*}{RSD-C} & 2$-$2$-$2$-$2 & 2.152 & 2.013 & 43.755 & 0.378 \\
 &  &  &  & 5$-$1$-$1$-$1 & 2.083 & 1.948 & 43.142 & 0.375 \\
 &  &  &  & 7$-$1$-$1$-$1 & 2.130 & 1.992 & 42.567 & 0.375 \\
\cline{4-9}
 &  &  & \multirow[c]{2}{*}{RSD-S} & 5$\times$4 & 2.311 & 2.161 & 44.367 & 0.375 \\
 &  &  &  & 7$\times$4 & \color{red} \bfseries 2.408 & \color{red} \bfseries 2.252 & \color{red} \bfseries 46.197 & 0.376 \\
\cline{3-9} \cline{4-9}
 &  & \multirow[c]{8}{*}{5} & SD & 5 & 1.910 & 1.758 & 36.041 & 0.378 \\
\cline{4-9}
 &  &  & \multirow[c]{2}{*}{SpecTr} & 6$\times$5 & 2.120 & 1.951 & 38.841 & 0.373 \\
 &  &  &  & 12$\times$5 & 2.176 & 2.002 & 39.702 & 0.376 \\
\cline{4-9}
 &  &  & \multirow[c]{3}{*}{RSD-C} & 2$-$2$-$2$-$2$-$2 & 2.234 & 2.056 & 42.161 & 0.375 \\
 &  &  &  & 6$-$1$-$1$-$1$-$1 & 2.171 & 1.998 & 40.331 & 0.372 \\
 &  &  &  & 12$-$1$-$1$-$1$-$1 & 2.249 & 2.070 & 41.585 & 0.376 \\
\cline{4-9}
 &  &  & \multirow[c]{2}{*}{RSD-S} & 6$\times$5 & 2.467 & 2.270 & 43.898 & 0.370 \\
 &  &  &  & 12$\times$5 & \color{red} \bfseries 2.657 & \color{red} \bfseries 2.445 & \color{red} \bfseries 46.843 & 0.374 \\
\cline{1-9} \cline{2-9} \cline{3-9} \cline{4-9}
\bottomrule
\end{tabular}
\end{table}
\begin{table}[h]
\centering
\caption{We summarize experiment results with Llama 2-13B target and 115M draft for the XSum task with various draft lengths.
                               Draft Length (DL) means the (maximum) length for all draft token sequences generated by the draft model.
                               For decoders (Dec.), we consider Auto-Regressive sampling (AR), Speculative Decoding (SD), SpecTr, 
                               Recursive Speculative Decoding with Constant branching factors (RSD-C),
                               Recursive Speculative Decoding with Stochastic Beam Search (RSD-S).
                               The contents in decoder specification (Spec.) have different meanings for each decoder;
                               $K\times L$ means the number $K$ of draft paths and draft length $L$ for SpecTr;
                               it describes constant branching factors for each level of the tree (from root to leaf) for RSD-C;
                               $K\times L$ means the beamwidth $K$ and draft length $L$ for RSD-S.
                               Block efficiency (Eff.), Memory Bound Speed Up (MBSU), Token Rate (TR), and Accuracy (Acc.) are given for each algorithm.
                               For each group of rows having the same complexity, we highlight the best values for all columns except accuracy.
                            }
\label{table_13b-base_xsum_v2}
\vskip 0.15in
\begin{tabular}{lllllrrrr}
\toprule
 &  &  &  &  & Eff. & MBSU & TR & Acc. \\
Model & Task & DL & Dec. & Spec. &  &  &  &  \\
\midrule
\multirow[c]{33}{*}{Llama 2-13B} & \multirow[c]{33}{*}{XSum} & 0 & AR & - & 1.000 & 1.000 & 28.141 & 0.166 \\
\cline{3-9} \cline{4-9}
 &  & \multirow[c]{8}{*}{2} & SD & 2 & 2.120 & 2.082 & 40.958 & 0.164 \\
\cline{4-9}
 &  &  & \multirow[c]{2}{*}{SpecTr} & 2$\times$2 & 2.172 & 2.133 & 40.870 & 0.170 \\
 &  &  &  & 3$\times$2 & 2.212 & 2.173 & 41.116 & 0.165 \\
\cline{4-9}
 &  &  & \multirow[c]{3}{*}{RSD-C} & 2$-$1 & 2.224 & 2.185 & 41.866 & 0.165 \\
 &  &  &  & 2$-$2 & 2.347 & 2.305 & \color{red} \bfseries 44.593 & 0.166 \\
 &  &  &  & 3$-$1 & 2.269 & 2.229 & 42.981 & 0.158 \\
\cline{4-9}
 &  &  & \multirow[c]{2}{*}{RSD-S} & 2$\times$2 & 2.311 & 2.271 & 43.533 & 0.165 \\
 &  &  &  & 3$\times$2 & \color{red} \bfseries 2.412 & \color{red} \bfseries 2.370 & 44.529 & 0.162 \\
\cline{3-9} \cline{4-9}
 &  & \multirow[c]{8}{*}{3} & SD & 3 & 2.377 & 2.315 & 42.777 & 0.160 \\
\cline{4-9}
 &  &  & \multirow[c]{2}{*}{SpecTr} & 3$\times$3 & 2.559 & 2.492 & 45.252 & 0.166 \\
 &  &  &  & 4$\times$3 & 2.578 & 2.510 & 44.703 & 0.164 \\
\cline{4-9}
 &  &  & \multirow[c]{3}{*}{RSD-C} & 2$-$2$-$2 & 2.784 & 2.711 & 47.985 & 0.166 \\
 &  &  &  & 3$-$1$-$1 & 2.560 & 2.493 & 44.855 & 0.164 \\
 &  &  &  & 4$-$1$-$1 & 2.593 & 2.525 & 44.639 & 0.161 \\
\cline{4-9}
 &  &  & \multirow[c]{2}{*}{RSD-S} & 3$\times$3 & 2.832 & 2.758 & 48.388 & 0.162 \\
 &  &  &  & 4$\times$3 & \color{red} \bfseries 2.919 & \color{red} \bfseries 2.842 & \color{red} \bfseries 50.092 & 0.163 \\
\cline{3-9} \cline{4-9}
 &  & \multirow[c]{8}{*}{4} & SD & 4 & 2.608 & 2.517 & 43.309 & 0.165 \\
\cline{4-9}
 &  &  & \multirow[c]{2}{*}{SpecTr} & 5$\times$4 & 2.880 & 2.780 & 45.940 & 0.161 \\
 &  &  &  & 7$\times$4 & 2.944 & 2.842 & 47.456 & 0.162 \\
\cline{4-9}
 &  &  & \multirow[c]{3}{*}{RSD-C} & 2$-$2$-$2$-$2 & 3.096 & 2.989 & 49.203 & 0.167 \\
 &  &  &  & 5$-$1$-$1$-$1 & 2.813 & 2.715 & 44.641 & 0.165 \\
 &  &  &  & 7$-$1$-$1$-$1 & 2.864 & 2.764 & 45.288 & 0.162 \\
\cline{4-9}
 &  &  & \multirow[c]{2}{*}{RSD-S} & 5$\times$4 & 3.347 & 3.231 & 50.517 & 0.163 \\
 &  &  &  & 7$\times$4 & \color{red} \bfseries 3.442 & \color{red} \bfseries 3.322 & \color{red} \bfseries 52.105 & 0.157 \\
\cline{3-9} \cline{4-9}
 &  & \multirow[c]{8}{*}{5} & SD & 5 & 2.738 & 2.621 & 41.562 & 0.165 \\
\cline{4-9}
 &  &  & \multirow[c]{2}{*}{SpecTr} & 6$\times$5 & 3.108 & 2.974 & 46.120 & 0.169 \\
 &  &  &  & 12$\times$5 & 3.230 & 3.091 & 46.587 & 0.166 \\
\cline{4-9}
 &  &  & \multirow[c]{3}{*}{RSD-C} & 2$-$2$-$2$-$2$-$2 & 3.365 & 3.220 & 48.923 & 0.165 \\
 &  &  &  & 6$-$1$-$1$-$1$-$1 & 3.014 & 2.885 & 43.751 & 0.163 \\
 &  &  &  & 12$-$1$-$1$-$1$-$1 & 3.069 & 2.937 & 44.262 & 0.164 \\
\cline{4-9}
 &  &  & \multirow[c]{2}{*}{RSD-S} & 6$\times$5 & 3.648 & 3.492 & 50.782 & 0.162 \\
 &  &  &  & 12$\times$5 & \color{red} \bfseries 3.948 & \color{red} \bfseries 3.778 & \color{red} \bfseries 55.044 & 0.164 \\
\cline{1-9} \cline{2-9} \cline{3-9} \cline{4-9}
\bottomrule
\end{tabular}
\end{table}
\begin{table}[h]
\centering
\caption{We summarize experiment results with Llama 2-13B target and 115M draft for the WMT task with various draft lengths.
                               Draft Length (DL) means the (maximum) length for all draft token sequences generated by the draft model.
                               For decoders (Dec.), we consider Auto-Regressive sampling (AR), Speculative Decoding (SD), SpecTr, 
                               Recursive Speculative Decoding with Constant branching factors (RSD-C),
                               Recursive Speculative Decoding with Stochastic Beam Search (RSD-S).
                               The contents in decoder specification (Spec.) have different meanings for each decoder;
                               $K\times L$ means the number $K$ of draft paths and draft length $L$ for SpecTr;
                               it describes constant branching factors for each level of the tree (from root to leaf) for RSD-C;
                               $K\times L$ means the beamwidth $K$ and draft length $L$ for RSD-S.
                               Block efficiency (Eff.), Memory Bound Speed Up (MBSU), Token Rate (TR), and Accuracy (Acc.) are given for each algorithm.
                               For each group of rows having the same complexity, we highlight the best values for all columns except accuracy.
                            }
\label{table_13b-base_wmt_v2}
\vskip 0.15in
\begin{tabular}{lllllrrrr}
\toprule
 &  &  &  &  & Eff. & MBSU & TR & Acc. \\
Model & Task & DL & Dec. & Spec. &  &  &  &  \\
\midrule
\multirow[c]{33}{*}{Llama 2-13B} & \multirow[c]{33}{*}{WMT} & 0 & AR & - & 1.000 & 1.000 & 30.467 & 0.413 \\
\cline{3-9} \cline{4-9}
 &  & \multirow[c]{8}{*}{2} & SD & 2 & 1.662 & 1.632 & 34.571 & 0.410 \\
\cline{4-9}
 &  &  & \multirow[c]{2}{*}{SpecTr} & 2$\times$2 & 1.717 & 1.686 & 35.383 & 0.405 \\
 &  &  &  & 3$\times$2 & 1.748 & 1.717 & 35.124 & 0.408 \\
\cline{4-9}
 &  &  & \multirow[c]{3}{*}{RSD-C} & 2$-$1 & 1.760 & 1.729 & 36.200 & 0.405 \\
 &  &  &  & 2$-$2 & 1.852 & 1.819 & \color{red} \bfseries 38.384 & 0.407 \\
 &  &  &  & 3$-$1 & 1.815 & 1.783 & 37.576 & 0.408 \\
\cline{4-9}
 &  &  & \multirow[c]{2}{*}{RSD-S} & 2$\times$2 & 1.810 & 1.778 & 35.906 & 0.404 \\
 &  &  &  & 3$\times$2 & \color{red} \bfseries 1.903 & \color{red} \bfseries 1.869 & 37.934 & 0.410 \\
\cline{3-9} \cline{4-9}
 &  & \multirow[c]{8}{*}{3} & SD & 3 & 1.778 & 1.731 & 34.213 & 0.408 \\
\cline{4-9}
 &  &  & \multirow[c]{2}{*}{SpecTr} & 3$\times$3 & 1.876 & 1.827 & 35.754 & 0.407 \\
 &  &  &  & 4$\times$3 & 1.903 & 1.853 & 35.127 & 0.409 \\
\cline{4-9}
 &  &  & \multirow[c]{3}{*}{RSD-C} & 2$-$2$-$2 & 2.027 & 1.974 & 36.916 & 0.407 \\
 &  &  &  & 3$-$1$-$1 & 1.929 & 1.878 & 37.279 & 0.413 \\
 &  &  &  & 4$-$1$-$1 & 1.965 & 1.914 & 35.558 & 0.408 \\
\cline{4-9}
 &  &  & \multirow[c]{2}{*}{RSD-S} & 3$\times$3 & 2.059 & 2.005 & 36.511 & 0.406 \\
 &  &  &  & 4$\times$3 & \color{red} \bfseries 2.141 & \color{red} \bfseries 2.084 & \color{red} \bfseries 39.415 & 0.413 \\
\cline{3-9} \cline{4-9}
 &  & \multirow[c]{8}{*}{4} & SD & 4 & 1.852 & 1.787 & 33.728 & 0.406 \\
\cline{4-9}
 &  &  & \multirow[c]{2}{*}{SpecTr} & 5$\times$4 & 2.004 & 1.935 & 34.950 & 0.409 \\
 &  &  &  & 7$\times$4 & 2.038 & 1.968 & 34.973 & 0.411 \\
\cline{4-9}
 &  &  & \multirow[c]{3}{*}{RSD-C} & 2$-$2$-$2$-$2 & 2.122 & 2.048 & 36.819 & 0.407 \\
 &  &  &  & 5$-$1$-$1$-$1 & 2.080 & 2.008 & 35.434 & 0.411 \\
 &  &  &  & 7$-$1$-$1$-$1 & 2.116 & 2.042 & 35.526 & 0.406 \\
\cline{4-9}
 &  &  & \multirow[c]{2}{*}{RSD-S} & 5$\times$4 & 2.304 & 2.224 & 37.842 & 0.408 \\
 &  &  &  & 7$\times$4 & \color{red} \bfseries 2.399 & \color{red} \bfseries 2.315 & \color{red} \bfseries 38.315 & 0.410 \\
\cline{3-9} \cline{4-9}
 &  & \multirow[c]{8}{*}{5} & SD & 5 & 1.913 & 1.831 & 31.396 & 0.406 \\
\cline{4-9}
 &  &  & \multirow[c]{2}{*}{SpecTr} & 6$\times$5 & 2.101 & 2.011 & 34.184 & 0.408 \\
 &  &  &  & 12$\times$5 & 2.168 & 2.074 & 34.936 & 0.408 \\
\cline{4-9}
 &  &  & \multirow[c]{3}{*}{RSD-C} & 2$-$2$-$2$-$2$-$2 & 2.198 & 2.103 & 34.472 & 0.412 \\
 &  &  &  & 6$-$1$-$1$-$1$-$1 & 2.163 & 2.070 & 34.502 & 0.408 \\
 &  &  &  & 12$-$1$-$1$-$1$-$1 & 2.238 & 2.142 & 35.575 & 0.410 \\
\cline{4-9}
 &  &  & \multirow[c]{2}{*}{RSD-S} & 6$\times$5 & 2.448 & 2.343 & 36.278 & 0.408 \\
 &  &  &  & 12$\times$5 & \color{red} \bfseries 2.638 & \color{red} \bfseries 2.525 & \color{red} \bfseries 39.182 & 0.412 \\
\cline{1-9} \cline{2-9} \cline{3-9} \cline{4-9}
\bottomrule
\end{tabular}
\end{table}
\begin{table}[h]
\centering
\caption{We summarize experiment results with Llama 2-70B target and 115M draft for the XSum task with various draft lengths.
                               Draft Length (DL) means the (maximum) length for all draft token sequences generated by the draft model.
                               For decoders (Dec.), we consider Auto-Regressive sampling (AR), Speculative Decoding (SD), SpecTr, 
                               Recursive Speculative Decoding with Constant branching factors (RSD-C),
                               Recursive Speculative Decoding with Stochastic Beam Search (RSD-S).
                               The contents in decoder specification (Spec.) have different meanings for each decoder;
                               $K\times L$ means the number $K$ of draft paths and draft length $L$ for SpecTr;
                               it describes constant branching factors for each level of the tree (from root to leaf) for RSD-C;
                               $K\times L$ means the beamwidth $K$ and draft length $L$ for RSD-S.
                               Block efficiency (Eff.), Memory Bound Speed Up (MBSU), Token Rate (TR), and Accuracy (Acc.) are given for each algorithm.
                               For each group of rows having the same complexity, we highlight the best values for all columns except accuracy.
                            }
\label{table_70b-base_xsum_v2}
\vskip 0.15in
\begin{tabular}{lllllrrrr}
\toprule
 &  &  &  &  & Eff. & MBSU & TR & Acc. \\
Model & Task & DL & Dec. & Spec. &  &  &  &  \\
\midrule
\multirow[c]{33}{*}{Llama 2-70B} & \multirow[c]{33}{*}{XSum} & 0 & AR & - & 1.000 & 1.000 & 9.079 & 0.194 \\
\cline{3-9} \cline{4-9}
 &  & \multirow[c]{8}{*}{2} & SD & 2 & 2.103 & 2.096 & 15.054 & 0.188 \\
\cline{4-9}
 &  &  & \multirow[c]{2}{*}{SpecTr} & 2$\times$2 & 2.164 & 2.157 & 15.171 & 0.189 \\
 &  &  &  & 3$\times$2 & 2.204 & 2.197 & 15.346 & 0.191 \\
\cline{4-9}
 &  &  & \multirow[c]{3}{*}{RSD-C} & 2$-$1 & 2.189 & 2.181 & 15.276 & 0.187 \\
 &  &  &  & 2$-$2 & 2.322 & 2.314 & 16.033 & 0.189 \\
 &  &  &  & 3$-$1 & 2.239 & 2.231 & 15.480 & 0.197 \\
\cline{4-9}
 &  &  & \multirow[c]{2}{*}{RSD-S} & 2$\times$2 & 2.288 & 2.280 & 15.719 & 0.189 \\
 &  &  &  & 3$\times$2 & \color{red} \bfseries 2.376 & \color{red} \bfseries 2.368 & \color{red} \bfseries 16.284 & 0.193 \\
\cline{3-9} \cline{4-9}
 &  & \multirow[c]{8}{*}{3} & SD & 3 & 2.365 & 2.353 & 15.992 & 0.193 \\
\cline{4-9}
 &  &  & \multirow[c]{2}{*}{SpecTr} & 3$\times$3 & 2.528 & 2.515 & 16.533 & 0.195 \\
 &  &  &  & 4$\times$3 & 2.554 & 2.541 & 16.586 & 0.193 \\
\cline{4-9}
 &  &  & \multirow[c]{3}{*}{RSD-C} & 2$-$2$-$2 & 2.757 & 2.743 & 17.790 & 0.188 \\
 &  &  &  & 3$-$1$-$1 & 2.551 & 2.538 & 16.837 & 0.189 \\
 &  &  &  & 4$-$1$-$1 & 2.543 & 2.531 & 16.617 & 0.196 \\
\cline{4-9}
 &  &  & \multirow[c]{2}{*}{RSD-S} & 3$\times$3 & 2.765 & 2.751 & 17.689 & 0.192 \\
 &  &  &  & 4$\times$3 & \color{red} \bfseries 2.862 & \color{red} \bfseries 2.848 & \color{red} \bfseries 18.163 & 0.186 \\
\cline{3-9} \cline{4-9}
 &  & \multirow[c]{8}{*}{4} & SD & 4 & 2.584 & 2.566 & 16.656 & 0.196 \\
\cline{4-9}
 &  &  & \multirow[c]{2}{*}{SpecTr} & 5$\times$4 & 2.844 & 2.825 & 17.159 & 0.192 \\
 &  &  &  & 7$\times$4 & 2.883 & 2.863 & 17.052 & 0.192 \\
\cline{4-9}
 &  &  & \multirow[c]{3}{*}{RSD-C} & 2$-$2$-$2$-$2 & 3.028 & 3.008 & 17.885 & 0.191 \\
 &  &  &  & 5$-$1$-$1$-$1 & 2.749 & 2.730 & 16.658 & 0.193 \\
 &  &  &  & 7$-$1$-$1$-$1 & 2.780 & 2.762 & 16.568 & 0.190 \\
\cline{4-9}
 &  &  & \multirow[c]{2}{*}{RSD-S} & 5$\times$4 & 3.242 & 3.220 & 18.965 & 0.196 \\
 &  &  &  & 7$\times$4 & \color{red} \bfseries 3.361 & \color{red} \bfseries 3.338 & \color{red} \bfseries 19.248 & 0.191 \\
\cline{3-9} \cline{4-9}
 &  & \multirow[c]{8}{*}{5} & SD & 5 & 2.680 & 2.658 & 16.634 & 0.194 \\
\cline{4-9}
 &  &  & \multirow[c]{2}{*}{SpecTr} & 6$\times$5 & 3.103 & 3.077 & 17.603 & 0.192 \\
 &  &  &  & 12$\times$5 & 3.206 & 3.179 & 17.344 & 0.194 \\
\cline{4-9}
 &  &  & \multirow[c]{3}{*}{RSD-C} & 2$-$2$-$2$-$2$-$2 & 3.295 & 3.268 & 17.675 & 0.191 \\
 &  &  &  & 6$-$1$-$1$-$1$-$1 & 2.935 & 2.910 & 16.809 & 0.193 \\
 &  &  &  & 12$-$1$-$1$-$1$-$1 & 3.004 & 2.978 & 16.371 & 0.193 \\
\cline{4-9}
 &  &  & \multirow[c]{2}{*}{RSD-S} & 6$\times$5 & 3.556 & 3.526 & 19.484 & 0.192 \\
 &  &  &  & 12$\times$5 & \color{red} \bfseries 3.851 & \color{red} \bfseries 3.819 & \color{red} \bfseries 19.808 & 0.194 \\
\cline{1-9} \cline{2-9} \cline{3-9} \cline{4-9}
\bottomrule
\end{tabular}
\end{table}
\begin{table}[h]
\centering
\caption{We summarize experiment results with Llama 2-70B target and 115M draft for the WMT task with various draft lengths.
                               Draft Length (DL) means the (maximum) length for all draft token sequences generated by the draft model.
                               For decoders (Dec.), we consider Auto-Regressive sampling (AR), Speculative Decoding (SD), SpecTr, 
                               Recursive Speculative Decoding with Constant branching factors (RSD-C),
                               Recursive Speculative Decoding with Stochastic Beam Search (RSD-S).
                               The contents in decoder specification (Spec.) have different meanings for each decoder;
                               $K\times L$ means the number $K$ of draft paths and draft length $L$ for SpecTr;
                               it describes constant branching factors for each level of the tree (from root to leaf) for RSD-C;
                               $K\times L$ means the beamwidth $K$ and draft length $L$ for RSD-S.
                               Block efficiency (Eff.), Memory Bound Speed Up (MBSU), Token Rate (TR), and Accuracy (Acc.) are given for each algorithm.
                               For each group of rows having the same complexity, we highlight the best values for all columns except accuracy.
                            }
\label{table_70b-base_wmt_v2}
\vskip 0.15in
\begin{tabular}{lllllrrrr}
\toprule
 &  &  &  &  & Eff. & MBSU & TR & Acc. \\
Model & Task & DL & Dec. & Spec. &  &  &  &  \\
\midrule
\multirow[c]{33}{*}{Llama 2-70B} & \multirow[c]{33}{*}{WMT} & 0 & AR & - & 1.000 & 1.000 & 9.706 & 0.439 \\
\cline{3-9} \cline{4-9}
 &  & \multirow[c]{8}{*}{2} & SD & 2 & 1.661 & 1.655 & 13.331 & 0.440 \\
\cline{4-9}
 &  &  & \multirow[c]{2}{*}{SpecTr} & 2$\times$2 & 1.732 & 1.726 & 13.742 & 0.445 \\
 &  &  &  & 3$\times$2 & 1.756 & 1.750 & 13.710 & 0.445 \\
\cline{4-9}
 &  &  & \multirow[c]{3}{*}{RSD-C} & 2$-$1 & 1.770 & 1.764 & 13.992 & 0.436 \\
 &  &  &  & 2$-$2 & 1.853 & 1.847 & 14.512 & 0.443 \\
 &  &  &  & 3$-$1 & 1.819 & 1.813 & 14.245 & 0.440 \\
\cline{4-9}
 &  &  & \multirow[c]{2}{*}{RSD-S} & 2$\times$2 & 1.819 & 1.813 & 14.211 & 0.438 \\
 &  &  &  & 3$\times$2 & \color{red} \bfseries 1.907 & \color{red} \bfseries 1.900 & \color{red} \bfseries 14.727 & 0.439 \\
\cline{3-9} \cline{4-9}
 &  & \multirow[c]{8}{*}{3} & SD & 3 & 1.778 & 1.769 & 13.620 & 0.442 \\
\cline{4-9}
 &  &  & \multirow[c]{2}{*}{SpecTr} & 3$\times$3 & 1.880 & 1.870 & 13.906 & 0.440 \\
 &  &  &  & 4$\times$3 & 1.909 & 1.899 & 13.959 & 0.437 \\
\cline{4-9}
 &  &  & \multirow[c]{3}{*}{RSD-C} & 2$-$2$-$2 & 2.021 & 2.010 & 14.875 & 0.438 \\
 &  &  &  & 3$-$1$-$1 & 1.940 & 1.930 & 14.474 & 0.441 \\
 &  &  &  & 4$-$1$-$1 & 1.968 & 1.958 & 14.483 & 0.439 \\
\cline{4-9}
 &  &  & \multirow[c]{2}{*}{RSD-S} & 3$\times$3 & 2.068 & 2.057 & 15.132 & 0.440 \\
 &  &  &  & 4$\times$3 & \color{red} \bfseries 2.140 & \color{red} \bfseries 2.129 & \color{red} \bfseries 15.591 & 0.437 \\
\cline{3-9} \cline{4-9}
 &  & \multirow[c]{8}{*}{4} & SD & 4 & 1.866 & 1.854 & 13.698 & 0.437 \\
\cline{4-9}
 &  &  & \multirow[c]{2}{*}{SpecTr} & 5$\times$4 & 2.016 & 2.003 & 13.865 & 0.440 \\
 &  &  &  & 7$\times$4 & 2.048 & 2.034 & 13.833 & 0.441 \\
\cline{4-9}
 &  &  & \multirow[c]{3}{*}{RSD-C} & 2$-$2$-$2$-$2 & 2.130 & 2.116 & 14.338 & 0.440 \\
 &  &  &  & 5$-$1$-$1$-$1 & 2.088 & 2.074 & 14.311 & 0.440 \\
 &  &  &  & 7$-$1$-$1$-$1 & 2.132 & 2.117 & 14.406 & 0.441 \\
\cline{4-9}
 &  &  & \multirow[c]{2}{*}{RSD-S} & 5$\times$4 & 2.309 & 2.293 & 15.550 & 0.434 \\
 &  &  &  & 7$\times$4 & \color{red} \bfseries 2.408 & \color{red} \bfseries 2.392 & \color{red} \bfseries 15.962 & 0.437 \\
\cline{3-9} \cline{4-9}
 &  & \multirow[c]{8}{*}{5} & SD & 5 & 1.917 & 1.901 & 13.430 & 0.437 \\
\cline{4-9}
 &  &  & \multirow[c]{2}{*}{SpecTr} & 6$\times$5 & 2.113 & 2.095 & 13.737 & 0.443 \\
 &  &  &  & 12$\times$5 & 2.177 & 2.159 & 13.619 & 0.442 \\
\cline{4-9}
 &  &  & \multirow[c]{3}{*}{RSD-C} & 2$-$2$-$2$-$2$-$2 & 2.232 & 2.214 & 13.988 & 0.443 \\
 &  &  &  & 6$-$1$-$1$-$1$-$1 & 2.173 & 2.154 & 14.095 & 0.440 \\
 &  &  &  & 12$-$1$-$1$-$1$-$1 & 2.246 & 2.227 & 14.022 & 0.440 \\
\cline{4-9}
 &  &  & \multirow[c]{2}{*}{RSD-S} & 6$\times$5 & 2.451 & 2.430 & 15.575 & 0.439 \\
 &  &  &  & 12$\times$5 & \color{red} \bfseries 2.650 & \color{red} \bfseries 2.628 & \color{red} \bfseries 15.941 & 0.439 \\
\cline{1-9} \cline{2-9} \cline{3-9} \cline{4-9}
\bottomrule
\end{tabular}
\end{table}
\begin{table}[h]
\centering
\caption{We summarize experiment results with Llama 2-Chat-7B target and 115M draft for the XSum task with various draft lengths.
                               Draft Length (DL) means the (maximum) length for all draft token sequences generated by the draft model.
                               For decoders (Dec.), we consider Auto-Regressive sampling (AR), Speculative Decoding (SD), SpecTr, 
                               Recursive Speculative Decoding with Constant branching factors (RSD-C),
                               Recursive Speculative Decoding with Stochastic Beam Search (RSD-S).
                               The contents in decoder specification (Spec.) have different meanings for each decoder;
                               $K\times L$ means the number $K$ of draft paths and draft length $L$ for SpecTr;
                               it describes constant branching factors for each level of the tree (from root to leaf) for RSD-C;
                               $K\times L$ means the beamwidth $K$ and draft length $L$ for RSD-S.
                               Block efficiency (Eff.), Memory Bound Speed Up (MBSU), Token Rate (TR), and Accuracy (Acc.) are given for each algorithm.
                               For each group of rows having the same complexity, we highlight the best values for all columns except accuracy.
                            }
\label{table_7b-chat_xsum_v2}
\vskip 0.15in
\begin{tabular}{lllllrrrr}
\toprule
 &  &  &  &  & Eff. & MBSU & TR & Acc. \\
Model & Task & DL & Dec. & Spec. &  &  &  &  \\
\midrule
\multirow[c]{33}{*}{Llama 2-Chat-7B} & \multirow[c]{33}{*}{XSum} & 0 & AR & - & 1.000 & 1.000 & 41.651 & 0.092 \\
\cline{3-9} \cline{4-9}
 &  & \multirow[c]{8}{*}{2} & SD & 2 & 1.938 & 1.873 & 47.708 & 0.091 \\
\cline{4-9}
 &  &  & \multirow[c]{2}{*}{SpecTr} & 2$\times$2 & 1.961 & 1.896 & 46.422 & 0.092 \\
 &  &  &  & 3$\times$2 & 1.972 & 1.906 & 45.886 & 0.092 \\
\cline{4-9}
 &  &  & \multirow[c]{3}{*}{RSD-C} & 2$-$1 & 2.048 & 1.979 & 49.725 & 0.091 \\
 &  &  &  & 2$-$2 & 2.162 & 2.090 & 51.949 & 0.089 \\
 &  &  &  & 3$-$1 & 2.100 & 2.030 & 50.115 & 0.091 \\
\cline{4-9}
 &  &  & \multirow[c]{2}{*}{RSD-S} & 2$\times$2 & 2.129 & 2.058 & 50.315 & 0.090 \\
 &  &  &  & 3$\times$2 & \color{red} \bfseries 2.220 & \color{red} \bfseries 2.146 & \color{red} \bfseries 52.867 & 0.090 \\
\cline{3-9} \cline{4-9}
 &  & \multirow[c]{8}{*}{3} & SD & 3 & 2.144 & 2.038 & 47.360 & 0.090 \\
\cline{4-9}
 &  &  & \multirow[c]{2}{*}{SpecTr} & 3$\times$3 & 2.202 & 2.093 & 46.211 & 0.092 \\
 &  &  &  & 4$\times$3 & 2.211 & 2.102 & 46.960 & 0.091 \\
\cline{4-9}
 &  &  & \multirow[c]{3}{*}{RSD-C} & 2$-$2$-$2 & 2.484 & 2.362 & 54.127 & 0.090 \\
 &  &  &  & 3$-$1$-$1 & 2.311 & 2.196 & 49.424 & 0.090 \\
 &  &  &  & 4$-$1$-$1 & 2.345 & 2.229 & 50.509 & 0.089 \\
\cline{4-9}
 &  &  & \multirow[c]{2}{*}{RSD-S} & 3$\times$3 & 2.525 & 2.400 & 51.902 & 0.093 \\
 &  &  &  & 4$\times$3 & \color{red} \bfseries 2.650 & \color{red} \bfseries 2.519 & \color{red} \bfseries 54.496 & 0.091 \\
\cline{3-9} \cline{4-9}
 &  & \multirow[c]{8}{*}{4} & SD & 4 & 2.269 & 2.122 & 45.826 & 0.091 \\
\cline{4-9}
 &  &  & \multirow[c]{2}{*}{SpecTr} & 5$\times$4 & 2.366 & 2.212 & 46.726 & 0.091 \\
 &  &  &  & 7$\times$4 & 2.379 & 2.225 & 46.287 & 0.089 \\
\cline{4-9}
 &  &  & \multirow[c]{3}{*}{RSD-C} & 2$-$2$-$2$-$2 & 2.701 & 2.526 & 52.867 & 0.093 \\
 &  &  &  & 5$-$1$-$1$-$1 & 2.503 & 2.341 & 49.052 & 0.089 \\
 &  &  &  & 7$-$1$-$1$-$1 & 2.562 & 2.396 & 52.106 & 0.089 \\
\cline{4-9}
 &  &  & \multirow[c]{2}{*}{RSD-S} & 5$\times$4 & 2.921 & 2.732 & 54.744 & 0.091 \\
 &  &  &  & 7$\times$4 & \color{red} \bfseries 3.023 & \color{red} \bfseries 2.827 & \color{red} \bfseries 56.318 & 0.087 \\
\cline{3-9} \cline{4-9}
 &  & \multirow[c]{8}{*}{5} & SD & 5 & 2.345 & 2.158 & 43.302 & 0.092 \\
\cline{4-9}
 &  &  & \multirow[c]{2}{*}{SpecTr} & 6$\times$5 & 2.455 & 2.259 & 44.595 & 0.091 \\
 &  &  &  & 12$\times$5 & 2.513 & 2.312 & 44.089 & 0.093 \\
\cline{4-9}
 &  &  & \multirow[c]{3}{*}{RSD-C} & 2$-$2$-$2$-$2$-$2 & 2.830 & 2.604 & 51.392 & 0.092 \\
 &  &  &  & 6$-$1$-$1$-$1$-$1 & 2.615 & 2.407 & 47.560 & 0.090 \\
 &  &  &  & 12$-$1$-$1$-$1$-$1 & 2.669 & 2.456 & 47.987 & 0.089 \\
\cline{4-9}
 &  &  & \multirow[c]{2}{*}{RSD-S} & 6$\times$5 & 3.142 & 2.891 & 54.142 & 0.089 \\
 &  &  &  & 12$\times$5 & \color{red} \bfseries 3.397 & \color{red} \bfseries 3.126 & \color{red} \bfseries 58.208 & 0.091 \\
\cline{1-9} \cline{2-9} \cline{3-9} \cline{4-9}
\bottomrule
\end{tabular}
\end{table}
\begin{table}[h]
\centering
\caption{We summarize experiment results with Llama 2-Chat-7B target and 115M draft for the WMT task with various draft lengths.
                               Draft Length (DL) means the (maximum) length for all draft token sequences generated by the draft model.
                               For decoders (Dec.), we consider Auto-Regressive sampling (AR), Speculative Decoding (SD), SpecTr, 
                               Recursive Speculative Decoding with Constant branching factors (RSD-C),
                               Recursive Speculative Decoding with Stochastic Beam Search (RSD-S).
                               The contents in decoder specification (Spec.) have different meanings for each decoder;
                               $K\times L$ means the number $K$ of draft paths and draft length $L$ for SpecTr;
                               it describes constant branching factors for each level of the tree (from root to leaf) for RSD-C;
                               $K\times L$ means the beamwidth $K$ and draft length $L$ for RSD-S.
                               Block efficiency (Eff.), Memory Bound Speed Up (MBSU), Token Rate (TR), and Accuracy (Acc.) are given for each algorithm.
                               For each group of rows having the same complexity, we highlight the best values for all columns except accuracy.
                            }
\label{table_7b-chat_wmt_v2}
\vskip 0.15in
\begin{tabular}{lllllrrrr}
\toprule
 &  &  &  &  & Eff. & MBSU & TR & Acc. \\
Model & Task & DL & Dec. & Spec. &  &  &  &  \\
\midrule
\multirow[c]{33}{*}{Llama 2-Chat-7B} & \multirow[c]{33}{*}{WMT} & 0 & AR & - & 1.000 & 1.000 & 37.093 & 0.377 \\
\cline{3-9} \cline{4-9}
 &  & \multirow[c]{8}{*}{2} & SD & 2 & 1.639 & 1.584 & 41.440 & 0.379 \\
\cline{4-9}
 &  &  & \multirow[c]{2}{*}{SpecTr} & 2$\times$2 & 1.664 & 1.608 & 40.657 & 0.379 \\
 &  &  &  & 3$\times$2 & 1.673 & 1.617 & 41.907 & 0.378 \\
\cline{4-9}
 &  &  & \multirow[c]{3}{*}{RSD-C} & 2$-$1 & 1.739 & 1.681 & 43.511 & 0.378 \\
 &  &  &  & 2$-$2 & 1.813 & 1.752 & 43.929 & 0.375 \\
 &  &  &  & 3$-$1 & 1.784 & 1.724 & 44.122 & 0.378 \\
\cline{4-9}
 &  &  & \multirow[c]{2}{*}{RSD-S} & 2$\times$2 & 1.786 & 1.726 & 44.139 & 0.378 \\
 &  &  &  & 3$\times$2 & \color{red} \bfseries 1.865 & \color{red} \bfseries 1.802 & \color{red} \bfseries 46.030 & 0.379 \\
\cline{3-9} \cline{4-9}
 &  & \multirow[c]{8}{*}{3} & SD & 3 & 1.747 & 1.660 & 40.480 & 0.376 \\
\cline{4-9}
 &  &  & \multirow[c]{2}{*}{SpecTr} & 3$\times$3 & 1.783 & 1.695 & 39.483 & 0.377 \\
 &  &  &  & 4$\times$3 & 1.791 & 1.702 & 39.811 & 0.374 \\
\cline{4-9}
 &  &  & \multirow[c]{3}{*}{RSD-C} & 2$-$2$-$2 & 1.967 & 1.870 & 42.825 & 0.377 \\
 &  &  &  & 3$-$1$-$1 & 1.896 & 1.802 & 42.228 & 0.379 \\
 &  &  &  & 4$-$1$-$1 & 1.918 & 1.824 & 41.396 & 0.376 \\
\cline{4-9}
 &  &  & \multirow[c]{2}{*}{RSD-S} & 3$\times$3 & 2.009 & 1.909 & 43.212 & 0.377 \\
 &  &  &  & 4$\times$3 & \color{red} \bfseries 2.064 & \color{red} \bfseries 1.962 & \color{red} \bfseries 44.172 & 0.378 \\
\cline{3-9} \cline{4-9}
 &  & \multirow[c]{8}{*}{4} & SD & 4 & 1.815 & 1.697 & 37.235 & 0.380 \\
\cline{4-9}
 &  &  & \multirow[c]{2}{*}{SpecTr} & 5$\times$4 & 1.870 & 1.749 & 38.695 & 0.377 \\
 &  &  &  & 7$\times$4 & 1.884 & 1.761 & 38.151 & 0.380 \\
\cline{4-9}
 &  &  & \multirow[c]{3}{*}{RSD-C} & 2$-$2$-$2$-$2 & 2.067 & 1.933 & 41.411 & 0.377 \\
 &  &  &  & 5$-$1$-$1$-$1 & 2.021 & 1.889 & 40.048 & 0.379 \\
 &  &  &  & 7$-$1$-$1$-$1 & 2.054 & 1.921 & 41.217 & 0.376 \\
\cline{4-9}
 &  &  & \multirow[c]{2}{*}{RSD-S} & 5$\times$4 & 2.217 & 2.073 & 43.154 & 0.378 \\
 &  &  &  & 7$\times$4 & \color{red} \bfseries 2.290 & \color{red} \bfseries 2.142 & \color{red} \bfseries 43.310 & 0.375 \\
\cline{3-9} \cline{4-9}
 &  & \multirow[c]{8}{*}{5} & SD & 5 & 1.861 & 1.713 & 35.968 & 0.375 \\
\cline{4-9}
 &  &  & \multirow[c]{2}{*}{SpecTr} & 6$\times$5 & 1.936 & 1.781 & 36.752 & 0.376 \\
 &  &  &  & 12$\times$5 & 1.994 & 1.835 & 37.541 & 0.371 \\
\cline{4-9}
 &  &  & \multirow[c]{3}{*}{RSD-C} & 2$-$2$-$2$-$2$-$2 & 2.142 & 1.971 & 40.408 & 0.376 \\
 &  &  &  & 6$-$1$-$1$-$1$-$1 & 2.093 & 1.926 & 38.438 & 0.376 \\
 &  &  &  & 12$-$1$-$1$-$1$-$1 & 2.155 & 1.983 & 38.945 & 0.377 \\
\cline{4-9}
 &  &  & \multirow[c]{2}{*}{RSD-S} & 6$\times$5 & 2.335 & 2.149 & 40.428 & 0.375 \\
 &  &  &  & 12$\times$5 & \color{red} \bfseries 2.488 & \color{red} \bfseries 2.289 & \color{red} \bfseries 43.359 & 0.375 \\
\cline{1-9} \cline{2-9} \cline{3-9} \cline{4-9}
\bottomrule
\end{tabular}
\end{table}
\begin{table}[h]
\centering
\caption{We summarize experiment results with Llama 2-Chat-7B target and 115M draft for the Dolly task with various draft lengths.
                               Draft Length (DL) means the (maximum) length for all draft token sequences generated by the draft model.
                               For decoders (Dec.), we consider Auto-Regressive sampling (AR), Speculative Decoding (SD), SpecTr, 
                               Recursive Speculative Decoding with Constant branching factors (RSD-C),
                               Recursive Speculative Decoding with Stochastic Beam Search (RSD-S).
                               The contents in decoder specification (Spec.) have different meanings for each decoder;
                               $K\times L$ means the number $K$ of draft paths and draft length $L$ for SpecTr;
                               it describes constant branching factors for each level of the tree (from root to leaf) for RSD-C;
                               $K\times L$ means the beamwidth $K$ and draft length $L$ for RSD-S.
                               Block efficiency (Eff.), Memory Bound Speed Up (MBSU), Token Rate (TR), and Accuracy (Acc.) are given for each algorithm.
                               For each group of rows having the same complexity, we highlight the best values for all columns except accuracy.
                            }
\label{table_7b-chat_dolly_v2}
\vskip 0.15in
\begin{tabular}{lllllrrrr}
\toprule
 &  &  &  &  & Eff. & MBSU & TR & Acc. \\
Model & Task & DL & Dec. & Spec. &  &  &  &  \\
\midrule
\multirow[c]{33}{*}{Llama 2-Chat-7B} & \multirow[c]{33}{*}{Dolly} & 0 & AR & - & 1.000 & 1.000 & 37.596 & - \\
\cline{3-9} \cline{4-9}
 &  & \multirow[c]{8}{*}{2} & SD & 2 & 2.122 & 2.051 & 51.492 & - \\
\cline{4-9}
 &  &  & \multirow[c]{2}{*}{SpecTr} & 2$\times$2 & 2.177 & 2.104 & 50.471 & - \\
 &  &  &  & 3$\times$2 & 2.215 & 2.140 & 49.350 & - \\
\cline{4-9}
 &  &  & \multirow[c]{3}{*}{RSD-C} & 2$-$1 & 2.182 & 2.109 & 50.732 & - \\
 &  &  &  & 2$-$2 & 2.253 & 2.178 & \color{red} \bfseries 52.610 & - \\
 &  &  &  & 3$-$1 & 2.201 & 2.127 & 51.639 & - \\
\cline{4-9}
 &  &  & \multirow[c]{2}{*}{RSD-S} & 2$\times$2 & 2.239 & 2.164 & 50.320 & - \\
 &  &  &  & 3$\times$2 & \color{red} \bfseries 2.278 & \color{red} \bfseries 2.202 & 51.740 & - \\
\cline{3-9} \cline{4-9}
 &  & \multirow[c]{8}{*}{3} & SD & 3 & 2.429 & 2.309 & 51.847 & - \\
\cline{4-9}
 &  &  & \multirow[c]{2}{*}{SpecTr} & 3$\times$3 & 2.549 & 2.423 & 54.051 & - \\
 &  &  &  & 4$\times$3 & 2.579 & 2.451 & 53.358 & - \\
\cline{4-9}
 &  &  & \multirow[c]{3}{*}{RSD-C} & 2$-$2$-$2 & 2.628 & 2.498 & 54.402 & - \\
 &  &  &  & 3$-$1$-$1 & 2.508 & 2.384 & 52.888 & - \\
 &  &  &  & 4$-$1$-$1 & 2.506 & 2.382 & 52.892 & - \\
\cline{4-9}
 &  &  & \multirow[c]{2}{*}{RSD-S} & 3$\times$3 & 2.660 & 2.528 & 54.360 & - \\
 &  &  &  & 4$\times$3 & \color{red} \bfseries 2.686 & \color{red} \bfseries 2.553 & \color{red} \bfseries 55.581 & - \\
\cline{3-9} \cline{4-9}
 &  & \multirow[c]{8}{*}{4} & SD & 4 & 2.642 & 2.470 & 51.204 & - \\
\cline{4-9}
 &  &  & \multirow[c]{2}{*}{SpecTr} & 5$\times$4 & 2.853 & 2.668 & 52.977 & - \\
 &  &  &  & 7$\times$4 & 2.888 & 2.700 & 54.500 & - \\
\cline{4-9}
 &  &  & \multirow[c]{3}{*}{RSD-C} & 2$-$2$-$2$-$2 & 2.914 & 2.725 & \color{red} \bfseries 55.146 & - \\
 &  &  &  & 5$-$1$-$1$-$1 & 2.716 & 2.539 & 52.369 & - \\
 &  &  &  & 7$-$1$-$1$-$1 & 2.728 & 2.551 & 53.662 & - \\
\cline{4-9}
 &  &  & \multirow[c]{2}{*}{RSD-S} & 5$\times$4 & 2.994 & 2.799 & 54.032 & - \\
 &  &  &  & 7$\times$4 & \color{red} \bfseries 3.005 & \color{red} \bfseries 2.810 & 54.242 & - \\
\cline{3-9} \cline{4-9}
 &  & \multirow[c]{8}{*}{5} & SD & 5 & 2.764 & 2.543 & 50.163 & - \\
\cline{4-9}
 &  &  & \multirow[c]{2}{*}{SpecTr} & 6$\times$5 & 3.072 & 2.826 & 53.907 & - \\
 &  &  &  & 12$\times$5 & 3.153 & 2.902 & \color{red} \bfseries 55.039 & - \\
\cline{4-9}
 &  &  & \multirow[c]{3}{*}{RSD-C} & 2$-$2$-$2$-$2$-$2 & 3.125 & 2.876 & 54.231 & - \\
 &  &  &  & 6$-$1$-$1$-$1$-$1 & 2.854 & 2.626 & 51.919 & - \\
 &  &  &  & 12$-$1$-$1$-$1$-$1 & 2.860 & 2.631 & 52.547 & - \\
\cline{4-9}
 &  &  & \multirow[c]{2}{*}{RSD-S} & 6$\times$5 & 3.221 & 2.963 & 53.055 & - \\
 &  &  &  & 12$\times$5 & \color{red} \bfseries 3.241 & \color{red} \bfseries 2.982 & 52.504 & - \\
\cline{1-9} \cline{2-9} \cline{3-9} \cline{4-9}
\bottomrule
\end{tabular}
\end{table}
\begin{table}[h]
\centering
\caption{We summarize experiment results with Llama 2-Chat-13B target and 115M draft for the XSum task with various draft lengths.
                               Draft Length (DL) means the (maximum) length for all draft token sequences generated by the draft model.
                               For decoders (Dec.), we consider Auto-Regressive sampling (AR), Speculative Decoding (SD), SpecTr, 
                               Recursive Speculative Decoding with Constant branching factors (RSD-C),
                               Recursive Speculative Decoding with Stochastic Beam Search (RSD-S).
                               The contents in decoder specification (Spec.) have different meanings for each decoder;
                               $K\times L$ means the number $K$ of draft paths and draft length $L$ for SpecTr;
                               it describes constant branching factors for each level of the tree (from root to leaf) for RSD-C;
                               $K\times L$ means the beamwidth $K$ and draft length $L$ for RSD-S.
                               Block efficiency (Eff.), Memory Bound Speed Up (MBSU), Token Rate (TR), and Accuracy (Acc.) are given for each algorithm.
                               For each group of rows having the same complexity, we highlight the best values for all columns except accuracy.
                            }
\label{table_13b-chat_xsum_v2}
\vskip 0.15in
\begin{tabular}{lllllrrrr}
\toprule
 &  &  &  &  & Eff. & MBSU & TR & Acc. \\
Model & Task & DL & Dec. & Spec. &  &  &  &  \\
\midrule
\multirow[c]{33}{*}{Llama 2-Chat-13B} & \multirow[c]{33}{*}{XSum} & 0 & AR & - & 1.000 & 1.000 & 28.727 & 0.112 \\
\cline{3-9} \cline{4-9}
 &  & \multirow[c]{8}{*}{2} & SD & 2 & 1.941 & 1.906 & 38.799 & 0.113 \\
\cline{4-9}
 &  &  & \multirow[c]{2}{*}{SpecTr} & 2$\times$2 & 1.973 & 1.938 & 38.368 & 0.110 \\
 &  &  &  & 3$\times$2 & 1.976 & 1.941 & 38.896 & 0.114 \\
\cline{4-9}
 &  &  & \multirow[c]{3}{*}{RSD-C} & 2$-$1 & 2.044 & 2.008 & 40.215 & 0.111 \\
 &  &  &  & 2$-$2 & 2.166 & 2.128 & 42.176 & 0.112 \\
 &  &  &  & 3$-$1 & 2.093 & 2.056 & 39.946 & 0.114 \\
\cline{4-9}
 &  &  & \multirow[c]{2}{*}{RSD-S} & 2$\times$2 & 2.127 & 2.089 & 40.412 & 0.111 \\
 &  &  &  & 3$\times$2 & \color{red} \bfseries 2.232 & \color{red} \bfseries 2.193 & \color{red} \bfseries 42.700 & 0.111 \\
\cline{3-9} \cline{4-9}
 &  & \multirow[c]{8}{*}{3} & SD & 3 & 2.160 & 2.104 & 39.082 & 0.111 \\
\cline{4-9}
 &  &  & \multirow[c]{2}{*}{SpecTr} & 3$\times$3 & 2.200 & 2.142 & 39.496 & 0.113 \\
 &  &  &  & 4$\times$3 & 2.211 & 2.153 & 38.870 & 0.110 \\
\cline{4-9}
 &  &  & \multirow[c]{3}{*}{RSD-C} & 2$-$2$-$2 & 2.476 & 2.411 & 42.802 & 0.112 \\
 &  &  &  & 3$-$1$-$1 & 2.308 & 2.247 & 41.188 & 0.112 \\
 &  &  &  & 4$-$1$-$1 & 2.346 & 2.284 & 41.058 & 0.109 \\
\cline{4-9}
 &  &  & \multirow[c]{2}{*}{RSD-S} & 3$\times$3 & 2.522 & 2.456 & 43.542 & 0.112 \\
 &  &  &  & 4$\times$3 & \color{red} \bfseries 2.616 & \color{red} \bfseries 2.548 & \color{red} \bfseries 45.064 & 0.113 \\
\cline{3-9} \cline{4-9}
 &  & \multirow[c]{8}{*}{4} & SD & 4 & 2.290 & 2.211 & 38.922 & 0.113 \\
\cline{4-9}
 &  &  & \multirow[c]{2}{*}{SpecTr} & 5$\times$4 & 2.376 & 2.293 & 38.781 & 0.111 \\
 &  &  &  & 7$\times$4 & 2.387 & 2.304 & 38.827 & 0.113 \\
\cline{4-9}
 &  &  & \multirow[c]{3}{*}{RSD-C} & 2$-$2$-$2$-$2 & 2.692 & 2.598 & 43.417 & 0.112 \\
 &  &  &  & 5$-$1$-$1$-$1 & 2.510 & 2.423 & 41.620 & 0.110 \\
 &  &  &  & 7$-$1$-$1$-$1 & 2.553 & 2.465 & 41.016 & 0.109 \\
\cline{4-9}
 &  &  & \multirow[c]{2}{*}{RSD-S} & 5$\times$4 & 2.924 & 2.823 & 46.864 & 0.112 \\
 &  &  &  & 7$\times$4 & \color{red} \bfseries 3.056 & \color{red} \bfseries 2.950 & \color{red} \bfseries 48.332 & 0.111 \\
\cline{3-9} \cline{4-9}
 &  & \multirow[c]{8}{*}{5} & SD & 5 & 2.371 & 2.269 & 37.278 & 0.111 \\
\cline{4-9}
 &  &  & \multirow[c]{2}{*}{SpecTr} & 6$\times$5 & 2.499 & 2.392 & 38.423 & 0.111 \\
 &  &  &  & 12$\times$5 & 2.530 & 2.421 & 38.113 & 0.113 \\
\cline{4-9}
 &  &  & \multirow[c]{3}{*}{RSD-C} & 2$-$2$-$2$-$2$-$2 & 2.853 & 2.731 & 43.183 & 0.110 \\
 &  &  &  & 6$-$1$-$1$-$1$-$1 & 2.621 & 2.508 & 40.949 & 0.112 \\
 &  &  &  & 12$-$1$-$1$-$1$-$1 & 2.684 & 2.569 & 39.693 & 0.112 \\
\cline{4-9}
 &  &  & \multirow[c]{2}{*}{RSD-S} & 6$\times$5 & 3.153 & 3.018 & 46.345 & 0.112 \\
 &  &  &  & 12$\times$5 & \color{red} \bfseries 3.390 & \color{red} \bfseries 3.244 & \color{red} \bfseries 48.010 & 0.108 \\
\cline{1-9} \cline{2-9} \cline{3-9} \cline{4-9}
\bottomrule
\end{tabular}
\end{table}
\begin{table}[h]
\centering
\caption{We summarize experiment results with Llama 2-Chat-13B target and 115M draft for the WMT task with various draft lengths.
                               Draft Length (DL) means the (maximum) length for all draft token sequences generated by the draft model.
                               For decoders (Dec.), we consider Auto-Regressive sampling (AR), Speculative Decoding (SD), SpecTr, 
                               Recursive Speculative Decoding with Constant branching factors (RSD-C),
                               Recursive Speculative Decoding with Stochastic Beam Search (RSD-S).
                               The contents in decoder specification (Spec.) have different meanings for each decoder;
                               $K\times L$ means the number $K$ of draft paths and draft length $L$ for SpecTr;
                               it describes constant branching factors for each level of the tree (from root to leaf) for RSD-C;
                               $K\times L$ means the beamwidth $K$ and draft length $L$ for RSD-S.
                               Block efficiency (Eff.), Memory Bound Speed Up (MBSU), Token Rate (TR), and Accuracy (Acc.) are given for each algorithm.
                               For each group of rows having the same complexity, we highlight the best values for all columns except accuracy.
                            }
\label{table_13b-chat_wmt_v2}
\vskip 0.15in
\begin{tabular}{lllllrrrr}
\toprule
 &  &  &  &  & Eff. & MBSU & TR & Acc. \\
Model & Task & DL & Dec. & Spec. &  &  &  &  \\
\midrule
\multirow[c]{33}{*}{Llama 2-Chat-13B} & \multirow[c]{33}{*}{WMT} & 0 & AR & - & 1.000 & 1.000 & 29.233 & 0.340 \\
\cline{3-9} \cline{4-9}
 &  & \multirow[c]{8}{*}{2} & SD & 2 & 1.729 & 1.699 & 36.679 & 0.346 \\
\cline{4-9}
 &  &  & \multirow[c]{2}{*}{SpecTr} & 2$\times$2 & 1.806 & 1.774 & 37.746 & 0.338 \\
 &  &  &  & 3$\times$2 & 1.758 & 1.727 & 36.160 & 0.343 \\
\cline{4-9}
 &  &  & \multirow[c]{3}{*}{RSD-C} & 2$-$1 & 1.758 & 1.727 & 37.315 & 0.357 \\
 &  &  &  & 2$-$2 & 1.928 & 1.894 & 39.464 & 0.345 \\
 &  &  &  & 3$-$1 & 1.891 & 1.858 & 38.600 & 0.344 \\
\cline{4-9}
 &  &  & \multirow[c]{2}{*}{RSD-S} & 2$\times$2 & 1.956 & 1.922 & 39.052 & 0.333 \\
 &  &  &  & 3$\times$2 & \color{red} \bfseries 2.023 & \color{red} \bfseries 1.987 & \color{red} \bfseries 40.458 & 0.349 \\
\cline{3-9} \cline{4-9}
 &  & \multirow[c]{8}{*}{3} & SD & 3 & 1.831 & 1.783 & 35.515 & 0.342 \\
\cline{4-9}
 &  &  & \multirow[c]{2}{*}{SpecTr} & 3$\times$3 & 1.839 & 1.791 & 35.208 & 0.347 \\
 &  &  &  & 4$\times$3 & 1.915 & 1.865 & 35.999 & 0.338 \\
\cline{4-9}
 &  &  & \multirow[c]{3}{*}{RSD-C} & 2$-$2$-$2 & 2.078 & 2.023 & 38.741 & 0.335 \\
 &  &  &  & 3$-$1$-$1 & 2.115 & 2.060 & 39.754 & 0.336 \\
 &  &  &  & 4$-$1$-$1 & 2.033 & 1.980 & 38.500 & 0.348 \\
\cline{4-9}
 &  &  & \multirow[c]{2}{*}{RSD-S} & 3$\times$3 & 2.138 & 2.082 & 39.120 & 0.348 \\
 &  &  &  & 4$\times$3 & \color{red} \bfseries 2.271 & \color{red} \bfseries 2.212 & \color{red} \bfseries 40.370 & 0.340 \\
\cline{3-9} \cline{4-9}
 &  & \multirow[c]{8}{*}{4} & SD & 4 & 1.963 & 1.895 & 35.054 & 0.346 \\
\cline{4-9}
 &  &  & \multirow[c]{2}{*}{SpecTr} & 5$\times$4 & 2.050 & 1.979 & 35.219 & 0.348 \\
 &  &  &  & 7$\times$4 & 2.012 & 1.943 & 34.089 & 0.339 \\
\cline{4-9}
 &  &  & \multirow[c]{3}{*}{RSD-C} & 2$-$2$-$2$-$2 & 2.314 & 2.233 & 39.395 & 0.342 \\
 &  &  &  & 5$-$1$-$1$-$1 & 2.098 & 2.025 & 36.062 & 0.344 \\
 &  &  &  & 7$-$1$-$1$-$1 & 2.213 & 2.136 & 37.582 & 0.335 \\
\cline{4-9}
 &  &  & \multirow[c]{2}{*}{RSD-S} & 5$\times$4 & 2.385 & 2.302 & 39.546 & 0.340 \\
 &  &  &  & 7$\times$4 & \color{red} \bfseries 2.629 & \color{red} \bfseries 2.538 & \color{red} \bfseries 43.600 & 0.329 \\
\cline{3-9} \cline{4-9}
 &  & \multirow[c]{8}{*}{5} & SD & 5 & 2.089 & 1.999 & 34.688 & 0.347 \\
\cline{4-9}
 &  &  & \multirow[c]{2}{*}{SpecTr} & 6$\times$5 & 2.077 & 1.988 & 32.472 & 0.348 \\
 &  &  &  & 12$\times$5 & 2.069 & 1.980 & 32.444 & 0.342 \\
\cline{4-9}
 &  &  & \multirow[c]{3}{*}{RSD-C} & 2$-$2$-$2$-$2$-$2 & 2.401 & 2.298 & 38.278 & 0.343 \\
 &  &  &  & 6$-$1$-$1$-$1$-$1 & 2.381 & 2.278 & 38.381 & 0.339 \\
 &  &  &  & 12$-$1$-$1$-$1$-$1 & 2.254 & 2.157 & 35.717 & 0.346 \\
\cline{4-9}
 &  &  & \multirow[c]{2}{*}{RSD-S} & 6$\times$5 & 2.532 & 2.423 & 39.268 & 0.348 \\
 &  &  &  & 12$\times$5 & \color{red} \bfseries 2.683 & \color{red} \bfseries 2.568 & \color{red} \bfseries 41.290 & 0.340 \\
\cline{1-9} \cline{2-9} \cline{3-9} \cline{4-9}
\bottomrule
\end{tabular}
\end{table}
\input{table/table_13b-chat_dolly_v2}
\input{table/table_70b-chat_xsum_v2}
\input{table/table_70b-chat_wmt_v2}
\input{table/table_70b-chat_dolly_v2}
\input{table/table_125M_13B-base_xsum_v2}
\input{table/table_125M_13B-base_wmt_v2}
\input{table/table_125M_30B-base_xsum_v2}
\input{table/table_125M_30B-base_wmt_v2}
\input{table/table_125M_66B-base_xsum_v2}
\input{table/table_125M_66B-base_wmt_v2}
\input{table/table_350M_13B-base_xsum_v2}
\input{table/table_350M_13B-base_wmt_v2}
\input{table/table_350M_30B-base_xsum_v2}
\input{table/table_350M_30B-base_wmt_v2}
\input{table/table_350M_66B-base_xsum_v2}
\input{table/table_350M_66B-base_wmt_v2}

\clearpage

\input{table/table_7b-base_xsum_v1}
\input{table/table_7b-base_wmt_v1}
\input{table/table_13b-base_xsum_v1}
\input{table/table_13b-base_wmt_v1}
\input{table/table_70b-base_xsum_v1}
\input{table/table_70b-base_wmt_v1}
\input{table/table_7b-chat_xsum_v1}
\input{table/table_7b-chat_wmt_v1}
\input{table/table_7b-chat_dolly_v1}
\input{table/table_13b-chat_xsum_v1}
\input{table/table_13b-chat_wmt_v1}
\input{table/table_13b-chat_dolly_v1}
\input{table/table_70b-chat_xsum_v1}
\input{table/table_70b-chat_wmt_v1}
\input{table/table_70b-chat_dolly_v1}
\input{table/table_125M_13B-base_xsum_v1}
\input{table/table_125M_13B-base_wmt_v1}
\input{table/table_125M_30B-base_xsum_v1}
\input{table/table_125M_30B-base_wmt_v1}
\input{table/table_125M_66B-base_xsum_v1}
\input{table/table_125M_66B-base_wmt_v1}
\input{table/table_350M_13B-base_xsum_v1}
\input{table/table_350M_13B-base_wmt_v1}
\input{table/table_350M_30B-base_xsum_v1}
\input{table/table_350M_30B-base_wmt_v1}
\input{table/table_350M_66B-base_xsum_v1}
\input{table/table_350M_66B-base_wmt_v1}

\end{document}